\documentclass{article}

\usepackage{arxiv}
\usepackage{amsthm}
\usepackage[utf8]{inputenc} % allow utf-8 input
\usepackage[T1]{fontenc}    % use 8-bit T1 fonts
\usepackage{hyperref}       % hyperlinks
\usepackage{url}            % simple URL typesetting
\usepackage{booktabs}       % professional-quality tables
\usepackage{amsfonts}       % blackboard math symbols
\usepackage{nicefrac}       % compact symbols for 1/2, etc.
\usepackage{microtype}      % microtypography
\usepackage{lipsum}		% Can be removed after putting your text content
\usepackage{graphicx}
\usepackage{natbib}
\usepackage{doi}

\usepackage{indentfirst,csquotes}

\usepackage{amssymb,amsthm,amsmath}
\usepackage{xcolor,paralist,hyperref,titlesec,fancyhdr,etoolbox}
\newtheorem{theorem}{Theorem}[]
\newtheorem{definition}{Definition}

\newtheorem{lemma}{Lemma}
\newtheorem{proposition}{Proposition}

\titlespacing*{\section}{0pt}{0ex}{0ex}

\hypersetup{ colorlinks=true, linkcolor=black, filecolor=black, urlcolor=black }
\def\proof{\noindent {\it Proof. $\, $}}
\def\endproof{\hfill $\Box$ \vskip 5 pt }

\usepackage{lipsum}
\RequirePackage{tgtermes}
\RequirePackage{newtxtext}
\RequirePackage{newtxmath}
\RequirePackage{bm}
\RequirePackage{endnotes}

\definecolor{clemson-orange}{RGB}{234,106,32}
\definecolor{chicago-maroon}{RGB}{128,0,0}
\definecolor{northwestern-purple}{RGB}{82,0,99}
\definecolor{cornell-red}{RGB}{179,27,27}
\definecolor{sauder-green}{RGB}{171,180,0}
\definecolor{harveymudd-gold}{RGB}{178,139,51}
\usepackage{color}
\usepackage[dvipsnames]{xcolor} 
\usepackage{hyperref}
\usepackage{algorithm}
\usepackage{algpseudocode}
\usepackage{tikz}
\usepackage{mathtools}
\usepackage{natbib}
 \bibpunct[, ]{[}{]}{,}{n}{}{,}%
\allowdisplaybreaks[4]

\usepackage{tabularx}
\usepackage{booktabs}
\usepackage{multirow}

\newtheorem{fact}{Fact}
\usepackage[most]{tcolorbox}
\newtcolorbox{highlightbox}{
  colback={gray!10!white}, % 背景色 (浅橙色)
  colframe=black,      % 边框色 (黑色)
  boxrule=1pt,
  arc=2mm,
  left=4pt, right=4pt, top=4pt, bottom=4pt,
  breakable,
  enhanced jigsaw,      % 改善跨页效果
  break at=-\baselineskip  % 控制分页位置
}
\newtcolorbox{goalbox}[1][]{%
  colback=white,
  colframe=black,
  fonttitle=\bfseries\color{black},
  title={Goal:{\color{gray!40!orange}\, Scale Independent Optimizer}},
  enhanced,
  attach boxed title to top left={yshift=-2mm, xshift=2mm},
  boxed title style={colframe=black, colback=gray!20!white},
  coltext=gray!40!black,
  left=1mm, right=1mm,     % 内边距
  width=\linewidth,    % 横向更长（默认比正文宽 1cm）
  enlarge left by=-5mm,            % 向左移动
  #1
}

\usepackage{tcolorbox}
\tcbuselibrary{theorems}

\usepackage[textsize=tiny]{todonotes}
\usepackage{wrapfig}
\definecolor{navyblue}{RGB}{0, 38, 102} % 深蓝，接近 #002666

\newcount\Comments
\newcommand{\yplu}[1]{{\color{orange}{ [\textbf{Yiping:}} #1]}}
\newcommand{\kibitz}[2]{\ifnum\Comments=1{\textcolor{#1}{\textsf{\footnotesize #2}}}\fi}
\newcommand{\Jiajin}[1]{\kibitz{magenta}{[Jiajin: #1]}}
\DeclareMathOperator*{\argmin}{arg\,min}
\newcommand{\1}{{\mathbbm{1}}}

\def\bD{{\mathbf{D}}}

\def\x{{\mathbf{X}}}
\def\y{{\mathbf{Y}}}

\def\b{\bm{b}}

\def\g{{\mathbf g}}

\def\k{{\mathbf k}}

\def\q{\bm{q}}

\def\v{\bm{v}}

\def\x{\bm{x}}
\def\y{\bm{y}}
\def\z{\bm{z}}

\def\1{{\mathbf 1}}

\usepackage{xcolor} 
\definecolor{amethyst}{rgb}{0.6, 0.4, 0.8}
\newcount\Comments
\usepackage{amsmath}
\allowdisplaybreaks
\usepackage{tikz}
\usepackage{graphicx}
\usepackage{placeins} % For \FloatBarrier
\usepackage{adjustbox}
\newtcbox{\alertinline}[1][red]
  {on line, arc = 0pt, outer arc = 0pt,
    colback = #1!20!white, colframe = #1!50!black,
    boxsep = 0pt, left = 1pt, right = 1pt, top = 2pt, bottom = 2pt,
    boxrule = 0pt, bottomrule = 0pt, toprule = 0pt}

\newcommand{\pmean}{{(p,\textrm{mean})}}
\newcommand{\qmean}{{(q,\textrm{mean})}}

\usepackage{makecell}
\usepackage{enumitem}
\usepackage{subcaption}  % For subfigures
\usepackage{float}       % For controlling figure placement (optional)

\tcbset{
  mybox/.style={
    sharp corners,
    colback=white,
    boxrule=0pt,
    width=\textwidth,
    boxsep=5pt,
    left=5pt,
    right=0pt,
    top=5pt,
    bottom=5pt,
    leftrule=5pt,
    toprule=0pt,
    bottomrule=0pt,
    rightrule=0pt,
    colframe=black,
    leftlower=0mm,
    leftupper=0mm,
    overlay={
      \begin{tcbclipinterior}
        \fill[gray!50] (frame.south west) rectangle (interior.north west);
      \end{tcbclipinterior}
    }
  }
}
\allowdisplaybreaks[4]
\newtheorem{assumption}{Assumption}
\newtheorem{remark}{Remark}
\usepackage[textsize=tiny]{todonotes} 

\title{{ On the Width Scaling of Neural  Optimizers Under Matrix Operator Norms I: \\{\large Row/Column Normalization and Hyperparameter Transfer}\thanks{This technical report describes ongoing research. It will be updated as the work progresses.}}}

\author{ \hspace{1mm}Ruihan Xu \\
	University of Chicago\\
	\texttt{ruihanx@uchicago.edu} \\
	\url{https://ruihanxx.github.io/}
	\And
\hspace{1mm}Jiajin Li \\
	Sauder School of Business\\University of British Columbia\\
\texttt{jiajin.li@sauder.ubc.ca} \\
\url{https://gerrili1996.github.io/}\\
	 \And
	Yiping Lu \\
    Industrial Engineering \& Management Science\\
	Northwestern University \\
	\texttt{yiping.lu@northwestern.edu} \\
    \url{{https://2prime.github.io/}}
}

\renewcommand{\shorttitle}{Width Scaling of Neural Optimizer}

\hypersetup{
pdftitle={Row Normalization},
pdfauthor={Ruihan Xu, Jiajin Li, Yiping Lu},
}

\begin{document}
\maketitle

\begin{abstract}
A central question in modern deep learning is how to design optimizers whose behavior remains stable as the network width $w$ increases.
We address this question by interpreting several widely used neural-network optimizers, including \textrm{AdamW} and \textrm{Muon}, as instances of steepest descent under matrix operator norms. 
This perspective links optimizer geometry with the Lipschitz structure of the network forward map, and enables width-independent control of both Lipschitz and smoothness constants. 
However, steepest-descent rules induced by standard $p \to q$ operator norms lack layerwise composability and therefore cannot provide width-independent bounds in deep architectures. 
We overcome this limitation by introducing a family of mean-normalized operator norms, denoted $\pmean \to \qmean$, that admit layerwise composability, yield width-independent smoothness bounds, and give rise to practical optimizers such as \emph{rescaled} \textrm{AdamW}, row normalization, and column normalization. 
The resulting learning rate width-aware scaling rules recover $\mu$P scaling~\cite{yang2021tensor} as a special case and provide a principled mechanism for cross-width learning-rate transfer across a broad class of optimizers. 
We further show that \textrm{Muon} can suffer an $\mathcal{O}(\sqrt{w})$ worst-case growth in the smoothness constant, whereas a new family of row-normalized optimizers we propose achieves width-independent smoothness guarantees. Based on the observations, we propose MOGA (Matrix Operator Geometry Aware), a width-aware optimizer based only on row/column-wise normalization that enables stable learning-rate transfer across model widths. Large-scale pre-training on GPT-2 and LLaMA shows that MOGA, especially with row normalization, is competitive with Muon while being notably faster in large-token and low-loss regimes.
\end{abstract}

% keywords can be removed

\section{Introduction}

% Recent neural scaling laws \citep{cortes1993learning,hestness2017deep,kaplan2020scaling,bahri2024explaining} offer a clear narrative: Model performance improves predictably with scale, suggesting that larger architectures should yield better results. Yet putting these laws into practice exposes a gap: They say little about how the optimal learning rate varies with width.  In practice, for mainstream optimizers such as \textrm{AdamW} \citep{yang2021tensor,yang2022tensor,yang2023spectral} and \textrm{Muon} \citep{ghosh2025understanding}, the optimal learning rate is strongly width dependent: The optimal rate tuned for a network with 512 hidden units can diverge or slow markedly when width increases to 2048. This sensitivity indicates that standard optimizers do not naturally respect architectural scaling. To realize the gains promised by scaling laws, we therefore aim for optimizers whose tuning varies only weakly with width. Naturally, we ask:
% \vspace{2mm}
% \begin{center}
%     \bf How can we decouple the optimal learning rate from network width, and do algorithms that succeed at small scale remain effective as models grow? %How can we mitigate the width scaling of the optimal learning rate? 
% \end{center}
% \vspace{2mm}

Recent neural scaling laws~\citep{cortes1993learning,hestness2017deep,kaplan2020scaling,bahri2024explaining}
offer a clear narrative: Model performance improves predictably with scale, suggesting that larger architectures should yield better results.
Yet putting these laws into practice exposes a fundamental gap: They say little about how optimization hyperparameters  should vary with width.
In practice, for mainstream optimizers such as \textrm{AdamW} \citep{yang2021tensor,yang2022tensor,yang2023spectral}
and \textrm{Muon} \citep{ghosh2025understanding}, the optimal learning rate is strongly width dependent:
A value tuned for a network with 512 hidden units can diverge or slow markedly when the width increases to 2048.
This phenomenon reflects a lack of reliable \emph{hyperparameter transfer} across model scales, where tuning calibrated at the small model fails to generalize to the large one. This sensitivity indicates that standard optimizers do not naturally respect architectural scaling. To realize the gains promised by scaling laws, we therefore seek optimization methods whose tuning is width independent.
Naturally, we ask:
\begin{quote}
\centering
\textbf{Can we design optimization methods whose optimal learning rate transfers reliably\\as network width increases?}
\end{quote}

We answer this question in the affirmative by viewing existing neural network optimizers,
including \textrm{SignSGD}~\citep{bernstein2018signsgd},
\textrm{AdamW}~\citep{kingma2014adam,loshchilov2017decoupled},
\textrm{GradPower}~\citep{wang2025gradpower}, and
\textrm{Muon}~\citep{carlson2015preconditioned,muonblog},
within a unified framework as instances of steepest descent under matrix operator norms~\citep{flynn2017duality,maddison2021dual}.
Specifically, we consider the optimization problem $\min_{\bm{\Theta}} f(\bm{\Theta})$,
where $\bm{\Theta}$ denotes the network parameters.
At iteration $t$, classical steepest descent in Euclidean geometry (Frobenius norm in the matrix space) selects the unit-norm direction that yields the largest instantaneous decrease of the first-order model:
\[
\bm{D}^t
=
\argmin_{\left\|\bm{D}\right\|_F=1}
\langle \nabla f(\bm{\Theta}^t), \bm{D}\rangle
=
-\frac{\nabla f(\bm{\Theta}^t)}{\left\|\nabla f(\bm{\Theta}^t)\right\|_F}.
\]
More generally, steepest descent can be defined with respect to an arbitrary norm $\|\cdot\|$ with dual norm $\|\cdot\|_*$.
Different choices of geometry induce different update rules that exploit problem-specific structure.
% classical steepest descent  $\mathbf{W}_{k+1} = \mathbf{W}_k - \eta_k \nabla f(\mathbf{W}_k),$ where $\eta_k > 0$ denotes the step size, updates parameters along the negative gradient of the objective, which corresponds to the direction of maximal instantaneous decrease measured in the Euclidean norm. The gradient $\nabla f(x_k)$ is the steepest descent direction under the $\ell_2$ norm because, among all unit-length directions in this norm, it maximizes the rate of decrease of $f$:
% \(\nabla f(x_k) = \arg\max_{\|d\|_2 = 1} \nabla f(x_k)^\top d,\) i.e., it is the unit-norm direction along which the directional gradient of $f$ is maximized. 
% More generally, steepest descent can be defined with respect to an arbitrary norm $\|\cdot\|$ with the dual norm $\|\cdot\|_*$. 
% At the iterate ${\bm{\Theta}}^t$, a steepest–descent direction is any
% \[
% \bm{D}^t \in \mathop{\arg\min}_{\left\|\bm{D}\right\|\le 1}\,\langle \nabla f({\bm{\Theta}}^t), \bm{D}\rangle
%      \;=\; -\,\partial\left\|\nabla f({\bm{\Theta}}^t)\right\|_*, 
% \]
% where $\partial \|\cdot\|_*$ denotes the subdifferential of the dual norm.   
% This formulation recovers classical Euclidean steepest descent when $\|\cdot\| = \|\cdot\|_F$, while alternative norms induce different update rules that exploit problem-specific geometry.
For example, when $\|\cdot\|=\|\cdot\|_p$ with $p\ge1$, the induced $\ell_{p^\ast}$ normalization, where $\tfrac{1}{p}+\tfrac{1}{p^{\ast}}=1$, yields the \textrm{GradPower} family of methods \citep{wang2025gradpower}.
Under $\ell_\infty$ geometry, the steepest-descent direction reduces to the coordinate-wise sign of the gradient, yielding the well-known SignSGD algorithm~\citep{bernstein2018signsgd}. Moreover, \textrm{AdamW} \citep{kingma2014adam,loshchilov2017decoupled} 
can in fact be viewed as a smoothed or adaptive variant of \textrm{SignSGD}, see Remark~\ref{remark:adamassigngd} for details.  
Extending steepest-descent geometry from vector spaces to matrix spaces naturally leads to scaling rules governed by matrix operator norms. 
The resulting descent directions correspond to geometry-aware or “whitened’’ updates, which can be implemented via row/column normalization or related preconditioning schemes~\citep{carlson2015preconditioned,tuddenham2022orthogonalising,bernstein2024old,muonblog,pethick2025training,moonlight}.
Viewed from this perspective, the choice of matrix operator norm acts as a versatile design knob that provides a unified geometric framework encompassing a wide range of first-order methods.

Not every operator-norm geometry, however, leads to a practical method. The chosen norm must admit an efficiently computable steepest-descent direction, ideally in closed form or via a fast approximation.
Recent work provides simple descent oracles for several families,
including Newton--Schulz iteration for $2\to2$ norms,
column normalization for $1\to q$, and row normalization for
$p\to\infty$.
In this paper, we restrict attention to these computationally tractable
cases (i.e., $p\to q$ with $p\le q$) and investigate whether their
induced geometries preserve width-insensitive optimization properties of the network loss.

% A natural next question is how to select a \emph{matrix operator norm} that faithfully characterizes the objective landscape. To this end, we follow the modular framework of \citep{bernstein2024modular,large2024scalable}, which prescribes selecting the matrix norm that renders the forward map well-conditioned with respect to weight perturbations, thereby inducing an $M$-Lipschitz constant under the chosen geometry.
% Building on this intuition, \citep{bernstein2024modular} identifies the operator norm as a particularly natural choice; we adopt it as the starting point for our analysis. 

% A natural next question is how to select a \emph{matrix operator norm} that faithfully characterizes the objective landscape.
As an initial step, following the viewpoint of~\citep{bernstein2024modular,large2024scalable}, we focus on geometries under which the Lipschitz constant of the loss with respect to weight perturbations does not deteriorate as width increases.
To formalize this discussion, we introduce the matrix operator norm induced by vector norms. 
\begin{definition}[Matrix Operator Norm]
\label{def:op_norm}
Let $\bm{D} \in \mathbb{R}^{m\times n}$,  $\|\cdot\|_{\mathrm{in}}$ and 
$\|\cdot\|_{\mathrm{out}}$ be norms on $\mathbb{R}^n$ and $\mathbb{R}^m$, respectively. 
The operator norm of $\bm{D}$ induced by these norms is defined as
\[
\left\|\bm{D}\right\|_{\mathrm{in}\to\mathrm{out}}
:= \sup_{\x\in\mathbb{R}^n,\, \x\ne 0}\frac{\left\|\bm{D} \x\right\|_{\mathrm{out}}}{\left\|\x\right\|_{\mathrm{in}}}
= \sup_{\left\|\x\right\|_{\mathrm{in}}=1}\left\|\bm{D}\x\right\|_{\mathrm{out}}.
\]
\end{definition}

%  While operator norms provide a unifying geometric framework, not every such norm leads to a practical optimizer. From a computational standpoint, the chosen geometry must admit an efficient steepest-descent oracle—that is, the descent direction should be obtainable either in closed form or via a fast iterative approximation. Otherwise, the elegance of the operator-norm formulation offers little algorithmic value. To this end, \citep{bernstein2024old} propose simple oracles for several families of operator norms: Newton–Schulz iteration for $2\to 2$ the norm, column normalization for  $1\to q$, and row normalization for 
% $p\to \infty$. In what follows, we restrict attention to these computationally tractable cases (i.e., $p\to q$ where $p\le  q$),  and examine whether their induced geometries truly preserve the $M$-Lipschitz property of the network mapping.

% From a computational perspective, selecting the norm also requires that the steepest descent direction be efficiently computable, either via a closed-form expression or an approximate algorithm. \citep{bernstein2024old} proposed the Newton–Schulz iteration for the $2\to 2$ operator norm, column normalization for $1\to q$, and row normalization for $p\to \infty$. In what follows, we restrict our discussion to these operator norms.

Our analysis reveals that, perhaps counterintuitively, classical operator norms of the form $p \to q$  with $p\leq q$ do not in general provide width-independent Lipschitz control when stacked across multiple layers of a neural network.
To see this, consider the directional derivative of the $(i+2)$-th layer output with respect to the $i$-th layer weight matrix $\bm{W}_i$.
A standard chain-rule calculation shows that its norm decomposes into three factors:
the operator norm of the perturbation $\bm{\Delta}\bm{W}_i$,
the operator norm of the next-layer weight matrix $\bm{W}_{i+1}$,
and a cross-layer mismatch coefficient $\|\bm{I}\|_{i,\mathrm{out}\to i+1,\mathrm{in}}$,
where $\|\bm{I}\|_{i,\mathrm{out}\to i+1,\mathrm{in}}$ denotes the operator norm of the identity map between the output space of layer $i$ and the input space of layer $i+1$, capturing the geometric mismatch between consecutive layers.
When this mismatch coefficient exceeds one, even small perturbations can be amplified purely due to geometric inconsistency across consecutive layers.
This observation motivates a compatibility condition between adjacent norms, i.e., $\|\cdot\|_{\mathrm{in}} \le \|\cdot\|_{\mathrm{out}}$. 
Under this condition, the Lipchitz constant bound remains width independent. However, classical operator norms $p\to q$ with $1\leq p<q$ violate this property, i.e., 
$
\|\x\|_p \le n^{1/p-1/q}\|\x\|_q,
$ which implies $\|\bm{I}\|_{ q\to p}=n^{1/p-1/q}>1$. These observations motivate introducing a new, width-aware geometry that enforces cross-layer compatibility.
To this end, we consider the mean-normalized norms $\|\cdot\|_{(p,\mathrm{mean})}$ defined by 
\begin{equation*}
% \label{eq:mean-norm}
\left\|\x\right\|_{(p,\textrm{mean})} := \left(\frac{1}{n}\sum_{i=1}^n|\x_i|^p\right)^{1/p} = n^{-1/p}\left\|\x\right\|_p. 
\end{equation*}
The factor $n^{-1/p}$ precisely cancels the dimensional scaling of the
$\ell_p$ embeddings, thereby enforcing a compatibility condition between adjacent norms: $
\|\bm{I}\|_{(q,\mathrm{mean})\to(p,\mathrm{mean})}\le 1,
$ for all $1\le p\le q\le\infty$.
As a consequence, under the $(p,\mathrm{mean}) \to (q,\mathrm{mean})$ geometry,
the network loss admits a width-independent Lipschitz bound, see Theorem~\ref{thm:width-lip}.

Beyond width-independent Lipschitz control, a second requirement we consider is whether the operator geometry also admits a width-independent smoothness constant, which governs the curvature of the loss landscape and the stability of gradient-based updates.
Theorem~\ref{theorem:Lsmooth} shows that the smoothness constant is width independent whenever $q \ge 2p$.
More generally, any residual width dependence scales as
$w^{\frac{1}{q}-\frac{1}{2p}}$, precisely characterizing how curvature grows with the layer width.
In particular, $(1,\mathrm{mean})\to (q,\mathrm{mean})$ with $q\ge 2$ and
$(p,\mathrm{mean}) \to \infty$ both yield smoothness bounds that are independent of the network width.
By contrast, the $(2,\mathrm{mean})\to (2,\mathrm{mean})$ geometry
(corresponding to \textrm{Muon}) exhibits a smoothness constant that grows on the order of $\sqrt{w}$.

Passing from the classical $p\to q$ geometry to the
$(p,\mathrm{mean}) \to (q,\mathrm{mean})$ geometry induces a principled width-aware rescaling of the learning rate.
Indeed, the mean-normalized norm $\|\cdot\|_{(p,\mathrm{mean})}$ differs from the standard $\ell_p$ norm on $\mathbb{R}^n$ only by a multiplicative factor of $n^{1/p}$.
Consequently, the two geometries induce identical update directions up to a dimension-dependent scaling, which can be absorbed into the step size.
This observation gives rise to a new family of practical optimizers with
explicit width-aware learning-rate rules, referred to as
\emph{MOGA} (Matrix Operator Geometry Aware),
including rescaled variants of \textrm{AdamW} as well as row and column
normalization methods.
Moreover, in the special cases of \textrm{Adam} and \textrm{SignSGD}, the resulting scaling coincides exactly with the $\mu$P scaling rule of~\cite{yang2022tensor}.
Importantly, however, the underlying principles behind hyperparameter transfer are conceptually different.
$\mu$P is motivated by  preserving feature-learning behavior as width grows (e.g., keeping feature-level
update magnitudes non-degenerate), whose formal theoretical analysis  relies on the spectral condition of the network parameterization~\cite{yang2023spectral}.
In contrast, our scaling arises from an optimization-geometric perspective,
through width-independent Lipschitz and smoothness control under mean-normalized operator norms.
This viewpoint applies to a broader class of optimizers—including regimes that do not satisfy the spectral assumptions required by $\mu$P—while still enabling reliable hyperparameter transfer across widths.
Our analysis therefore provides a principled and more general foundation for hyperparameter transfer across model scales.

To summarize, under the mean-normalized operator-norm geometries
$(1,\mathrm{mean}) \to (q,\mathrm{mean})$ with $q\ge2$
and $(p,\mathrm{mean}) \to \infty$,
the network loss admits width-independent Lipschitz and smoothness bounds with respect to weight perturbations.
The induced steepest-descent updates lead to practical optimization methods,
in particular a family of rescaled row-normalization algorithms that exhibit reliable learning-rate transfer across widths in our experiments. 
In contrast, for the recently proposed and widely studied \textrm{Muon} optimizer,
which corresponds to the $(2,\mathrm{mean}) \to (2,\mathrm{mean})$ geometry, our theory reveals a more nuanced picture:
Although \textrm{Muon} attains width-independent Lipschitz control,
its smoothness constant can grow on the order of $\sqrt{w}$ in the worst case.
This reveals a potential limitation of \textrm{Muon} in maintaining stable optimization
behavior as width increases, and highlights an advantage of the new
rescaled row-normalization family proposed in this paper.

Finally, we validate our theoretical findings through large-scale language model pre-training experiments on both GPT-2 and LLaMA architectures. Consistent with our analysis, the proposed  \textrm{MOGA} scaling enables reliable learning-rate transfer across model widths: Models with substantially different parameter counts achieve their best performance at nearly the same peak learning rate, eliminating the need for extensive retuning when scaling up model size. We further evaluate \textrm{MOGA} instantiated with row normalization under both standard and large token budgets, corresponding to approximately
$\sim 1\times$ and $\sim 8\times$ the Chinchilla-optimal training tokens~\cite{hoffmann2022training}, respectively, on GPT-2 Small and LLaMA-130M. Across architectures, \textrm{MOGA} instantiated with row normalization achieves competitive or superior
optimization performance compared with strong baselines including
\textrm{AdamW} and \textrm{Muon}.
In particular, under the large-token regime, \textrm{MOGA} instantiated with row normalization demonstrates a clear
advantage in the later stages of training and in the low-loss regime,
where optimization stability becomes increasingly critical.
These results highlight the practical value of width-aware optimizer design:
Beyond enabling hyperparameter transfer, the proposed \textit{width-aware
rescaled row-normalization optimizer} (i.e., MOGA instantiated with row normalization) can improve training efficiency
in regimes that are most relevant for large-scale model deployment.

{\paragraph{Notation.} Let $\x \in \mathbb{R}^n$ be a column vector, and denote its $i$-th entry by $x_i$, $i=1,\dots,n$. 
Let $\bm{D} \in \mathbb{R}^{m \times n}$ be a matrix. 
We write $\bm{D}_{r,:}$ for the $r$-th row of $\bm{D}$ and $\bm{D}_{:,c}$ for its $c$-th column, 
for $r=1,\dots,m$ and $c=1,\dots,n$. 
For two vectors $x,y \in \mathbb{R}^n$, their elementwise product is defined as 
$
(x \odot y)_i := x_i y_i, \quad i=1,\dots,n .
$ 
Let $f:\Omega \to \mathbb{R}$ be differentiable, where $\Omega\subset\mathbb{R}^{m\times n}$. 
For $\bm Z \in \Omega$ and a direction $\Delta \bm Z$, the directional derivative of $f$ at $\bm Z$ along $\Delta \bm Z$ is defined as
\[
\nabla f(\bm Z)[\Delta \bm Z]
:= \lim_{t \to 0} \frac{f(\bm Z + t\Delta \bm Z) - f(\bm Z)}{t},
\]
whenever the limit exists. Let $f:\Omega \to \mathbb{R}$ be twice differentiable. 
For $\bm Z \in \Omega$ and directions $\Delta \bm Z_1,\Delta \bm Z_2$, 
the directional Hessian of $f$ at $\bm Z$ is defined as
\[
\nabla^2 f(\bm Z)[\Delta \bm Z_1,\Delta \bm Z_2]
:= \lim_{t \to 0}
\frac{
\nabla f(\bm Z + t \Delta \bm Z_2)[\Delta \bm Z_1]
- \nabla f(\bm Z)[\Delta \bm Z_1]
}{t},
\]
whenever the limit exists. 
We write $a \lesssim b$ if there exists a universal constant $C>0$, 
independent of the width, such that
$a \le C\, b$.

% Let $X$ and $Y$ be Banach spaces, and let $f : X \to Y$ be a mapping.  
% For $x \in X$ and a direction $h \in X$, the \textit{directional gradient} of $f$ at $x$ in the direction $h$ is defined as
% $\nabla_xf(x)[h]
%     := \lim_{t \to 0} \frac{f(x + t h) - f(x)}{t}$, whenever this limit exists.  
% We say that $f$ is (G\^{a}teaux) differentiable at $x$ if the above limit exists for all directions $h \in X$.  Similarly, the directional Hessian is defined as \( \nabla^{2} f(x)[h,k]
%     := \lim_{t \to 0}
%        \frac{1}{t}
%        \Bigl(
%           \nabla_x f(x + t k)[h]
%           - \nabla_x f(x)[h]
%        \Bigr).\)
%        } We use the notation $a \lesssim b$ to mean that there exists a universal constant $C>0$ (independent of the problem parameters) such that $a \le C\, b$.
% } 

\section{{\color{navyblue}Matrix Thinking}:  Unifying Optimizers via Matrix Operator Norm}
\label{sec:main}
We consider the optimization problem of minimizing a neural-network-based loss
\[
\min_{\bm{W}_{1:\ell},\, \b_{1:\ell}} f(\bm{W}_{1:\ell},\b_{1:\ell})\coloneqq \mathcal{L}\left(\y_\ell(\x)\right),
\]
where $\y_\ell(\x)\in\mathbb{R}$ denotes the output of a $\ell$-layer feedforward neural network with parameters $(\bm{W}_{1:\ell}, \b_{1:\ell})$, evaluated at input $\x\in\mathbb{R}^d$. Later on, we write  ${\bm{\Theta}} := \{\bm{W}_{1:\ell},\, \b_{1:\ell} \}$ for simplicity. 
\begin{definition}[Feedforward Neural Network]\label{def:nn}
Let $\y_0(\x)\coloneqq \x$. For $i=1,\dots,K$, define recursively
\[
\y_i(\x)\coloneqq  \sigma\left(\bm{W}_i \y_{i-1}(\x) + \b_i\right),
\]
where $\bm{W}_1\in\mathbb{R}^{w\times d}$, $\bm{W}_i\in\mathbb{R}^{w\times w}$ for
$i=2,\dots,K-1$, and $\bm{W}_K\in\mathbb{R}^{1\times w}$.
The bias vectors satisfy $\b_i\in\mathbb{R}^w$ for $i=1,\dots,K-1$ and
$b_\ell\in\mathbb{R}$.
The hidden states satisfy $\y_i(\x)\in\mathbb{R}^w$ for $i<K$, and the network
output is $\y_\ell(\x)\in\mathbb{R}$.  Here, $\sigma$ denotes a activation function, which represents both a mapping $\sigma:\mathbb{R}\to\mathbb{R}$ and its natural extension $\sigma:\mathbb{R}^n\to\mathbb{R}^n$, where it is applied pointwise to each component, i.e.,  $\sigma(\x)_i=\sigma(\x_i)$ for each coordinate $i$. 
\end{definition}

Throughout this paper, we assume that the loss function $\mathcal{L}$ and all activation functions $\sigma_i$ are twice continuously differentiable with uniformly bounded first- and second-order derivatives.

 \begin{assumption}{Bounded derivatives of activation functions}
\label{assump:boundedderivativeactivation}
There exist constants $L_\sigma, M_\sigma > 0$ such that
     \[ |\sigma'(z)|\le  L_\sigma,\quad \text{and}  \quad |\sigma''(z)|\le  M_\sigma, \qquad \forall z \in \mathbb{R}.\]
 \end{assumption}
 \begin{assumption}{Bounded derivatives of the loss function}
\label{assump:boundedderivativeloss}
 There exist constants $L_J, M_J > 0$ such that
    \[
    \left|\mathcal{L}'(z)\right|\le  L_J ,\quad \text{and}  \quad  \left|\mathcal{L}''(z)\right|\le  M_J, \qquad \forall z \in \mathbb{R}. 
    \]
 \end{assumption}

With the network architecture and smoothness assumptions in place, we now turn to the central question of this section: how to design and compare first-order optimizers through the geometry they induce. We adopt a matrix operator-norm perspective, viewing many modern optimization methods as instances of steepest-descent updates under appropriately chosen operator norms. This viewpoint encompasses a variety of recent optimizers, including \textrm{SignSGD} \citep{bernstein2018signsgd}, \textrm{Lion} \citep{chen2023symbolic}, and \textrm{Muon} \citep{muonblog}, which have demonstrated strong empirical performance in large-scale learning. 
From this perspective, the choice of geometry plays a dual role. On the one hand, it governs the computational cost of each update through the tractability of the associated steepest-descent direction; on the other hand, it determines the convergence behavior of the resulting optimization method. A well-chosen operator norm must therefore strike a balance between per-iteration computational efficiency and favorable convergence properties. 

\paragraph{Computability and per-iteration cost.} 
The steepest-descent direction induced by a general $\|\cdot\|_{p\to q}$ operator norm does not necessarily admit a closed-form expression. Beyond \textrm{Muon} \citep{muonblog}, which employs a Newton--Schulz iteration to approximate steepest descent under the $2\to2$ operator norm, prior work \citep{bernstein2024old} shows that steepest descent under the $1\to q$ and $p\to\infty$ norms admits efficient column-wise and row-wise closed-form computations, respectively. In this paper, we therefore restrict attention to operator norms for which the induced steepest-descent directions are computationally tractable; see Section \ref{sectionmethod} for details.

\paragraph{Convergence speed.} Among computable operator norms, a natural question is which $(p,q)$ geometries render the optimization problem well posed and lead to favorable convergence guarantees. We address this question by analyzing the $M$-Lipschitz continuity (Section \ref{section:playtogher}) and $L$-smoothness (Section \ref{subsection:selectpq}) of neural networks under various $p\to q$ operator norms. 

Our analysis reveals a fundamental limitation of classical $parrow q$ operator norms: Both Lipschitz and smoothness constants deteriorate with increasing network width, indicating a structural mismatch between these geometries and the compositional nature of deep neural networks. This degradation arises because standard operator norms fail to propagate stability estimates across layers, thereby distorting the effective geometry of the network's forward map. 

To overcome this issue, we introduce \emph{mean-normalized operator norms}, which are specifically designed to preserve composability under width scaling. These norms yield width-independent first- and second-order estimates and align the optimization geometry with the network's forward-map geometry. Building on this geometric insight, we develop the \textrm{MOGA} (\textbf{M}atrix \textbf{O}perator \textbf{G}eometry \textbf{A}ware) optimizer and show that $\mu$P-Adam arises as a special case within our framework.

\subsection{From \textrm{AdamW}, \textrm{Muon} to Row and Column Normalization} 
%: Implementable $p\to q$ Optimizers}
\label{sectionmethod}
In this subsection, we show that many state-of-the-art optimizers can be interpreted as instances of steepest descent under suitable matrix $p\to q$ operator norms, with several such instances admitting implementable per-iteration updates \citep{bernstein2024old}. To this end, recall that the steepest-descent direction under a matrix $p\to q$ operator norm is defined as the solution to
\begin{equation}
\label{eq:steep_descent}
\bm{D}^\star := \mathop{\arg\min}_{\left\|\bm{D}\right\|_{p\rightarrow q} \le 1}
\,\langle \nabla f(\bm{\Theta}), \bm{D}\rangle .
\end{equation}

We begin with three representative optimizers: \textrm{SignSGD} \citep{bernstein2018signsgd}, \textrm{Lion} \citep{chen2023symbolic}, and \textrm{AdamW} \citep{kingma2014adam,loshchilov2017decoupled}. 
Although these methods were originally developed for optimization over vector-valued parameters, they admit natural interpretations when applied to matrix-valued parameters. At the vector level, both \textrm{SignSGD} and \textrm{AdamW} can be viewed as steepest-descent methods under the $\ell_\infty$ geometry \citep{balles2018dissecting}. When such vector updates are applied entrywise to matrix parameters, the induced geometry corresponds implicitly to the matrix operator norm $\ell_1 \to \ell_\infty$ \citep{bernstein2024old}. This connection is formalized by the following fact.

\begin{fact}[\citep{bernstein2024old}, Proposition~3]\label{prop:signdgasoperator}
Let $\bm{D} \in \mathbb{R}^{m \times n}$. Then the element-wise $\ell_\infty$ norm of $\bm{D}$ coincides with its $\ell_1 \to \ell_\infty$ operator norm, that is, 
\[
\left\|\bm{D}\right\|_{1 \to \infty}
:= \sup_{\x \neq 0} \frac{\left\|\bm{D}\x\right\|_\infty}{\left\|\x\right\|_1}
= \max_{i,j} |\bm{D}_{ij}|
=: \left\|\bm{D}\right\|_{\max}.
\]
Then, the steepest-descent subproblem \eqref{eq:steep_descent} admits the closed-form solution $\bm{D}^\star = -\textrm{sign}(\nabla f(\bm{W}))$, where $\textrm{sign}(\cdot)$ is applied entrywise. Consequently, the induced update direction recovers the \textrm{SignSGD} algorithm.
% In particular, performing steepest descent under the $\|\cdot\|_{\max}$ norm corresponds to updating each parameter according to the sign of its gradient, which recovers the \textrm{SignSGD} algorithm.
\end{fact}
\begin{remark}
\label{remark:adamassigngd}
    We briefly clarify the relationship between \textrm{Adam}, \textrm{RMSprop}, and sign-based methods at the level of update directions. Consider the sign descent update
$
\bm{\theta}^{t+1} = \bm{\theta}^{t} - \alpha\, \textrm{sign}(\tilde \g_t),
$
where $\tilde \g_t$ denotes a (possibly stochastic) gradient estimate and $\alpha>0$ is the step size. This update corresponds to steepest descent under the $\ell_\infty$ geometry, as discussed above. The optimizer \textrm{RMSprop} maintains an exponential moving average of squared gradients,
\[
\v^{t+1} = \beta_2 \v^t + (1-\beta_2){\tilde\g}^t\odot{\tilde{\g}}^t,
\qquad
\bm{\theta}^{t+1} = \bm{\theta}^{t} - \alpha\, \frac{{\tilde \g}^t}{\sqrt{\v^{t+1}+\epsilon}},
\]
% \Jiajin{In general, let us use the superscript to denote the iteration.}
where $\beta_2\in[0,1)$ and $\epsilon>0$ ensure numerical stability. At the level of update directions, \textrm{RMSprop} (and equivalently \textrm{Adam} with $\beta_1=0$) can be viewed as a smoothed variant of sign descent, where the discontinuous sign operator is replaced by a normalized gradient with adaptive, coordinate-wise scaling. Formally, in the limiting case $\beta_2\to 0$ and $\epsilon\to 0$, the \textrm{RMSprop} update reduces to
\[
\bm{\theta}^{t+1} = \bm{\theta}^{t} - \alpha\,  \frac{{\tilde{\g}}^t}{\sqrt{{\tilde\g}^t\odot{\tilde{\g}}^t}} = \bm{\theta}^t - \alpha\, \textrm{sign}({\tilde\g}^t),
\]
recovering the \textrm{SignSGD} direction. This observation places \textrm{RMSprop}, \textrm{Adam}  and \textrm{AdamW} within a broader family of sign-based and geometry-aware first-order methods.

\end{remark}

The second example is \textrm{Muon} \citep{muonblog}, which can be interpreted as performing steepest descent under the matrix $2\to2$ operator norm. Since the spectral norm of a matrix coincides with the $2\to2$ operator norm, the steepest-descent subproblem~\eqref{eq:steep_descent} admits the closed-form solution
$
\bm{D}^\star = -\textrm{matrixsign}(\nabla f(\bm{\Theta})),
$
where the matrix sign is defined via the singular value decomposition:
if $\nabla f(\bm{\Theta})=\bm{U}\bm{\Sigma}\bm{V}^\top$, then
$\textrm{matrixsign}(\nabla f(\bm{\Theta}))=\bm{U}\bm{V}^\top$.
This characterization recovers the update direction of \textrm{Muon}, which whitens or orthogonalizes the gradient matrix. Empirically, \textrm{Muon} has been shown to be an effective alternative to \textrm{AdamW} for large-scale language-model pretraining \citep{moonlight}.
% \Jiajin{Check: both $\bm{U}$ and $\bm{V}$ have been used later!}

% The final class of examples we consider are the operator norms 
% $\|\cdot\|_{p\to \infty}$ and $\|\cdot\|_{1\to q}$ 
% \citep{bernstein2024old,glentis2025minimalist}, 
% for which the steepest-descent directions admit closed-form solutions.
% These norms naturally yield row-wise and column-wise update rules on matrices, respectively.

The final class of examples consists of the operator norms $\|\cdot\|_{1\to q}$ and $\|\cdot\|_{p\to\infty}$ with $p,q\ge 1$, whose associated steepest-descent subproblems admit closed-form solutions \citep{bernstein2024old,glentis2025minimalist}. These geometries induce simple column-wise and row-wise update rules on matrix parameters, see Proposition \ref{proposition:rowcolnorm} for details. 

\begin{proposition}
% [{\protect \citep[Proposition 3]{bernstein2024old}}]
\label{proposition:rowcolnorm}
% Let $f: \mathbb{R}^{m \times n} \to \mathbb{R}$ be differentiable, and let 
% $\bG = \nabla f(\x) \in \mathbb{R}^{m \times n}$. Then we have 
Consider the steepest-descent subproblem~\eqref{eq:steep_descent} with gradient
$\bm{G}=\nabla f(\bm{\Theta})\in\mathbb{R}^{m\times n}$. 
For the operator norms $\|\cdot\|_{1\to q}$ and $\|\cdot\|_{p\to\infty}$ with $p,q\ge 1$, the corresponding steepest-descent directions admit the following closed-form expressions:
% \begin{enumerate}
% \item 123
% \end{enumerate}
\begin{itemize}[leftmargin=0.15in]
\setlength{\itemsep}{4pt}
\item[\textnormal{(i)}] (\textbf{Column-wise update, $\|\cdot\|_{1\to q}$}) \[
\bm{D}^\star = \textrm{colnorm}_q(\bm{G})
\quad \text{and} \quad 
\textrm{colnorm}_q(\bm{G})_{:,c}
:= \dfrac{\mathrm{sign}(\bm{G}_{:,c}) \odot |\bm{G}_{:,c}|^{q^\ast-1}}{\left\|\bm{G}_{:,c}\right\|_{q^\ast}^{q^\ast-1}},
\]
where $q^\ast = q/(q-1)$ is the dual exponent.
\item[\textnormal{(ii)}] 
(\textbf{Row-wise update, $\|\cdot\|_{p\to\infty}$}) 
\[
\bm{D}^\star = \textrm{rownorm}_p(\bm{G})
\quad \text{and} \quad 
\textrm{rownorm}_p(\bm{G})_{r,:} := \dfrac{\mathrm{sign}(\bm{G}_{r,:}) \odot |\bm{G}_{r,:}|^{p-1}}{\left\|\bm{G}_{r,:}\right\|_p^{p-1}}.
\] 
% \vspace{-0.1in}
% \setlength{\itemsep}{0pt}
% \setlength{\parsep}{0pt}
% \setlength{\parskip}{0pt}
% \item Steepest descent direction under  $\|\cdot\|_{1 \to q}$ norm is given by row normalization \(\textrm{colnorm}_q(\bm{G})\)  where
% \[
% \textrm{colnorm}_q(\bm{G})_{:,c} := \dfrac{\mathrm{sign}(\bm{G}_{:,c}) \odot |\bm{G}_{:,c}|^{q^\ast-1}}{\left\|\bm{G}_{:,c}\right\|_{q^\ast}^{q^\ast-1}},
% \]
% where $q^\ast = q/(q-1)$
% \item Steepest descent direction under the $\|\cdot\|_{p \to \infty}$ norm is given by column normalization \(\textrm{rownorm}_p(\bm{G}),\) where
% \[
%  \textrm{rownorm}_p(\bm{G})_{r,:} := \dfrac{\mathrm{sign}(\bm{G}_{r,:}) \odot |\bm{G}_{r,:}|^{p-1}}{\left\|\bm{G}_{r,:}\right\|_p^{p-1}}.
% \]
\end{itemize}
Here $\odot$ denotes elementwise multiplication.
\begin{comment}
For a matrix $\bm{D}\in\mathbb{R}^{m\times n}$, we use $\bm{D}_{:,c}\in\mathbb{R}^{m}$ to denote its $c$-th column vector (all rows in column $j$), and $\bm{D}_{r,:}\in\mathbb{R}^{n}$ to denote its $r$-th row vector (all columns in row $i$).
\end{comment}
\end{proposition}
\begin{proof}{Proof of Proposition~\ref{proposition:rowcolnorm}.}
By \citep[Proposition~8]{bernstein2024old}, the operator norms admit the representations
\[
\left\|\bm{D}\right\|_{1\to q} = \max_{c\in [n]} \left\|\bm{D}_{:,c}\right\|_q,
\quad  \text{and} \quad
\left\|\bm{D}\right\|_{p\to \infty} = \max_{r\in [m]} \left\|\bm{D}_{r,:}\right\|_{p^\ast},
\]
where $p^\ast = p/(p-1)$ denotes the dual exponent. As a consequence, the steepest-descent subproblem~\eqref{eq:steep_descent} under the $\|\cdot\|_{1\to q}$ and $\|\cdot\|_{p\to\infty}$ operator norms decouples across columns and rows, respectively.

We focus on the case $\|\cdot\|_{1\to q}$ for illustration; the proof for $\|\cdot\|_{p\to\infty}$ follows by an analogous argument and is therefore omitted. For the $\|\cdot\|_{1\to q}$ case, the steepest-descent subproblem~\eqref{eq:steep_descent} can be written as
\[
\min_{\bm{D}\in\mathbb{R}^{m\times n}}
\ \langle \bm{G}, \bm{D} \rangle
\quad \text{s.t.} \quad
\left\|\bm{D}_{:,c}\right\|_q \le 1,\ \forall c\in[n]. 
\]
 Since the objective decomposes as
$
\langle \bm{G}, \bm{D} \rangle
= \sum_{c=1}^n \langle \bm{G}_{:,c}, \bm{D}_{:,c} \rangle, 
$
the problem decouples across columns. For each column $c\in[n]$, we obtain the subproblem
$
\min_{\|\bm{d}\|_q \le 1} \ \langle \bm{G}_{:,c}, \bm{d} \rangle .
$ 
By Hölder's inequality, an optimal solution is given by
\[
\bm{d}^\star
= -\,\dfrac{\mathrm{sign}(\bm{G}_{:,c}) \odot |\bm{G}_{:,c}|^{q^\ast-1}}{\left\|\bm{G}_{:,c}\right\|_{q^\ast}^{q^\ast-1}},
\]
where $q^\ast = q/(q-1)$ is the dual exponent. Collecting the solutions for all columns yields
$
\bm{D}^\star = \textrm{colnorm}_q(\bm{G}),
$ 
which completes the proof. 
 
\end{proof}

\begin{highlightbox}
\textbf{Message:} Although \textrm{SignSGD/AdamW} is typically formulated under the vector $\ell_\infty$ norm, it can also be interpreted through the lens of a matrix $\ell_1\to\ell_\infty$ norm.  In a similar vein, \textrm{Muon} \citep{muonblog,carlson2015preconditioned} performs steepest descent under the $\ell_2\to\ell_2$ operator norm. More generally, our framework considers the full families of operator norms $\ell_1\to\ell_q$ and $\ell_p\to\ell_\infty$ for arbitrary $p,q\ge 1$, which induce column-wise and row-wise normalization updates, respectively. Several existing methods studied in the literature can be recovered as special cases corresponding to particular choices of $(p,q)$; for instance, \citep{glentis2025minimalist,scetbon2025gradient} focus on the $\ell_1\to\ell_2$ case, while \citep{bernstein2024old,ma2024swan} consider the $\ell_2\to\ell_\infty$ geometry.
This perspective highlights the underlying \textbf{matrix thinking} behind a broad class of first-order optimizers. Notably, the update rules induced by the $\ell_1\to\ell_q$ and $\ell_p\to\ell_\infty$ families admit efficient, vectorized implementations.

    \begin{center}
            \begin{tabular}{c|c|c|c}
         \textrm{SignSGD/AdamW}& \textrm{Column Normalization} & \textrm{Row Normalization}& \textrm{Muon} \\
         \hline
         \hline
         $\|\cdot\|_{1\to\infty }$& $\|\cdot\|_{1\to q}$ & $\|\cdot\|_{p\to\infty }$  & $\|\cdot\|_{2\to2 }$
    \end{tabular}
    \end{center}
%Since \textrm{SignSGD/AdamW} operates in the $\|\cdot\|_{\textrm{MAV} \to \ell_\infty}$ geometry, Adam effectively constrains the features in the $\textrm{MAV}$ norm, thereby enforcing feature sparsity. In contrast, \textrm{TINA} operates under the $\|\cdot\|_{\textrm{RMS} \to \ell_\infty}$ norm, which avoids the feature sparsity induced by $\ell_1$ constraints and aligns with the design  philosophy of \textrm{Muon}. Nevertheless, \textrm{TINA} still acts in the vector space---it remains vector acting.
\end{highlightbox}

% We find that the matrix $p\to q$ operator norms with implementable pre-iteration forms all satisfy $p \le q$. Therefore, we restrict our focus to {the case $p\to q$ with $p \le q$}.

In summary, all matrix $p\to q$ operator norms that admit implementable per-iteration updates fall within the regime $p \le q$. Accordingly, we focus on the case $p\to q$ with $p \le q$ in the remainder of the paper.

% \subsection{\color{blue}{Why We Need the $\pmean\to\qmean$ Norm: Making Operators Play Nice Together}}
\subsection{Width-independent Lipschitz Bound under Mean-Normalized Geometry}
\label{section:playtogher}

Section~\ref{sectionmethod} shows that a broad class of first-order optimizers can be interpreted as steepest-descent methods under matrix $p\to q$ operator norms ($p \le q$) with computable update rules. A natural next question is whether these classical operator-norm geometries are well aligned with the stability properties of deep neural networks. In particular, the chosen matrix norm should reflect the model’s sensitivity to weight perturbations so that the loss remains width-independent Lipschitz continuous \citep{bernstein2020distance,bernstein2024modular,large2024scalable}.  However, despite their computational tractability and clean geometric form, we show that classical $p \to q$ operator norms with $p \le q$ generally fail to yield width-independent Lipschitz bounds for neural networks. As a result, the induced Lipschitz bounds deteriorate with network width under these norms.
This structural limitation motivates a refined geometry. In the remainder of this subsection, we introduce mean-normalized operator norms and show that they do yield width-independent Lipschitz bounds, see Theorem \ref{thm:width-lip} for details.
 
\begin{comment}
\citep{bernstein2020distance,bernstein2024modular,large2024scalable} emphasize that the selected matrix norm should match the model’s sensitivity to weight perturbations so that the loss remains $M$-Lipschitz. Under this view, \citep{bernstein2024modular} choose the matrix operator norms, yet its selection lacks a rigorous justification. However, and somewhat surprisingly, we show in this section that \textbf{standard operator norms fail to yield width-independent $M$-Lipschitz bounds for neural networks}. As a result, wider networks experience increasingly distorted loss landscapes under these norms.
\end{comment}

% To analyze the $M$-Lipschitz property of a neural network under a given matrix operator norm, we view the network as a composition of linear maps and nonlinearities $\sigma$:
To quantify the Lipschitz behavior induced by a given operator norm, we start to analyze the network in a layer-wise manner and track how Lipschitz bounds compose across layers.
We model the network as a composition of linear maps and nonlinear activation functions:
\[
\y_{\ell+1} = \sigma(\z_{\ell+1}), \quad \text{and} \quad 
\qquad \z_{\ell+1} = \bm{W}_{\ell+1} \y_\ell + \b_{\ell+1},
\]
where $\z_{\ell+1}$ denotes the pre-activation.

Each linear layer $\bm{W}_\ell$ maps features from one normed space to the next pre-activation space. When $\bm{W}_\ell$ is endowed with the operator norm $\|\cdot\|_{\mathrm{in}\to\mathrm{out}}$, Lipschitz stability propagates layer by layer. In particular, if $\sigma$ is $1$-Lipschitz, then the sensitivity of $\y_{\ell+1}$ to perturbations in an earlier weight matrix $\bm{W}_i$ ($i<\ell+1$) satisfies
\[
\left\|\nabla_{\bm{W}_i} \y_{l+1} \right\|_{\mathrm{out}}\le\left\|\nabla_{\bm{W}_i}\z_{l+1}\right\|_{\mathrm{out}}=\left\|\bm{W}_{l+1}\nabla_{\bm{W}_i} \y_{l}\right\|_{\mathrm{out}}\le \left\|\bm{W}_{l+1}\right\|_{\mathrm{in}\to\mathrm{out}}\left\|\nabla_{\bm{W}_i} \y_{l}\right\|_{\mathrm{in}},
\] 
where  $\nabla_{\bm{W}_i} \y_\ell$ denotes the Jacobian of $\y_\ell$ with respect to $\bm{W}_i$. Similarly, by the definition of the induced operator norm of $W_{\ell+2}$, we have
$\|\nabla_{\bm{W}_i} \y_{l+2}\|_{\mathrm{out}}
\;\le\;
\|\bm{W}_{l+2}\|_{\mathrm{in}\to\mathrm{out}}\,
\|\nabla_{\bm{W}_i} \y_{l+1}\|_{\mathrm{in}}.$  
To compose this bound with the previous layer-wise estimate, the output norm of layer $\ell+1$ must be compatible with the input norm of layer $\ell+2$. That is, we impose the following norm-compatibility condition between consecutive layers:
\begin{equation}\label{cond:norm-compat}
\left\|\cdot\right\|_{\mathrm{in}} \le \left\|\cdot\right\|_{\mathrm{out}}.
\end{equation}
Under the compatibility condition \eqref{cond:norm-compat}, the layer-wise bounds compose and yield
\[
\left\|\nabla_{\bm{W}_i} \y_{\ell+2}\right\|_{\mathrm{in}}
\le
\left\|\bm{W}_{\ell+2}\right\|_{\mathrm{in}\to\mathrm{out}}
\left\|\nabla_{\bm{W}_i} \y_{\ell+1}\right\|_{\mathrm{out}}
\le
\left(\prod_{k=i+1}^{\ell+2} \left\|\bm{W}_k\right\|_{\mathrm{in}\to\mathrm{out}}\right)
\left\|\nabla_{\bm{W}_i} \y_i\right\|_{\mathrm{in}}.
\] 

From the layer-wise bounds derived above, a width-independent cross-layer stability estimate based on operator norms is valid only if the input and output norms across consecutive layers satisfy the compatibility condition \eqref{cond:norm-compat}. This ensures that the sensitivity bounds can be composed across layers without distortion and therefore imposes a structural constraint on the admissible operator-norm geometries. However, classical matrix-induced norm families—such as the $1\to q$ and $p\to\infty$ cases that admit closed-form steepest-descent directions—do not meet this compatibility requirement in a dimension-independent manner. The difficulty is that the underlying feature norms are not uniformly comparable across layers. Indeed, $\ell_p$ norms are not uniformly equivalent: for $\x=(1,\dots,1)\in\mathbb{R}^n$ and any $p> 1$,
$
\|\x\|_\infty = 1
\;<\;
\|\x\|_{p} = n^{1/p}
\;<\;
\|\x\|_{1} = n,
$
so the norm ratios grow with dimension. 

To remove this dimension dependence, we introduce a mean-normalized variant of the $p$-power norm and work with the $\pmean\to\qmean$  operator-norm geometry instead of the standard $p \to q$ setting. Formally, for $\x\in\mathbb{R}^n$ and $1 \le p < \infty$, we define the $(p,\mathrm{mean})$ norm by 
\begin{equation}
\label{eq:mean-norm}
\left\|\x\right\|_{\pmean}
:=
\left(\frac{1}{n}\sum_{i=1}^n |x_i|^p\right)^{1/p}. 
\end{equation}
% and set $\|\x\|_{(\infty,\mathrm{mean})} := \|
% \x\|_\infty$.
This family recovers several familiar quantities as special cases: 
$\|\cdot\|_{(1,\mathrm{mean})}$ equals the mean absolute value (\textrm{MAV}), 
$\|\cdot\|_{(2,\mathrm{mean})}$ equals the root mean square (\textrm{RMS}), 
and $\|\cdot\|_{(\infty,\mathrm{mean})}$ reduces to the maximum norm, i.e., $\|\x\|_\infty = \max_i |x_i|$.
This normalization rescales the classical $\ell_p$ norm by the width factor and removes dimension-dependent growth. The key property of the $(p,\mathrm{mean})$ norms is that they remain uniformly comparable across widths.

\begin{fact}[Monotonicity of $(p,\mathrm{mean})$ norms]
\label{fact:pmean-monotone}
Let $\x \in \mathbb{R}^n$ and $1 \le p < q \le \infty$. Then
\[
\left\|\x\right\|_{(p,\mathrm{mean})}
\le
\left\|\x\right\|_{(q,\mathrm{mean})}.
\]
\end{fact}

\begin{proof}{Proof of Fact~\ref{fact:pmean-monotone}.}
The claim follows from the generalized mean inequality: For any $a_i \ge 0$ and $1 \le p \leq q$,
\[
\left(\frac{1}{n}\sum_{i=1}^n a_i^p\right)^{1/p}
\le
\left(\frac{1}{n}\sum_{i=1}^n a_i^q\right)^{1/q}.
\]
Apply this inequality with $a_i = |x_i|$ yields the result.
 
\end{proof}  

% When we consider the weights of the network matrix $W_l$ in the $(1,\textrm{mean})\to\qmean$ or $\pmean\to\ell_\infty$ geometry, the norms play well together since $\|\cdot\|_{(1,\textrm{mean})}\le\|\cdot\|_{\qmean}$ ,$\|\cdot\|_{\pmean}\le\|\cdot\|_\infty$ and lead to the following stability bound for neural network: 

We now establish a width-independent Lipschitz bound for neural networks under the general $\pmean \to \qmean$ operator-norm geometry.
\begin{figure}
    \centering
    \includegraphics[width=0.9\linewidth]{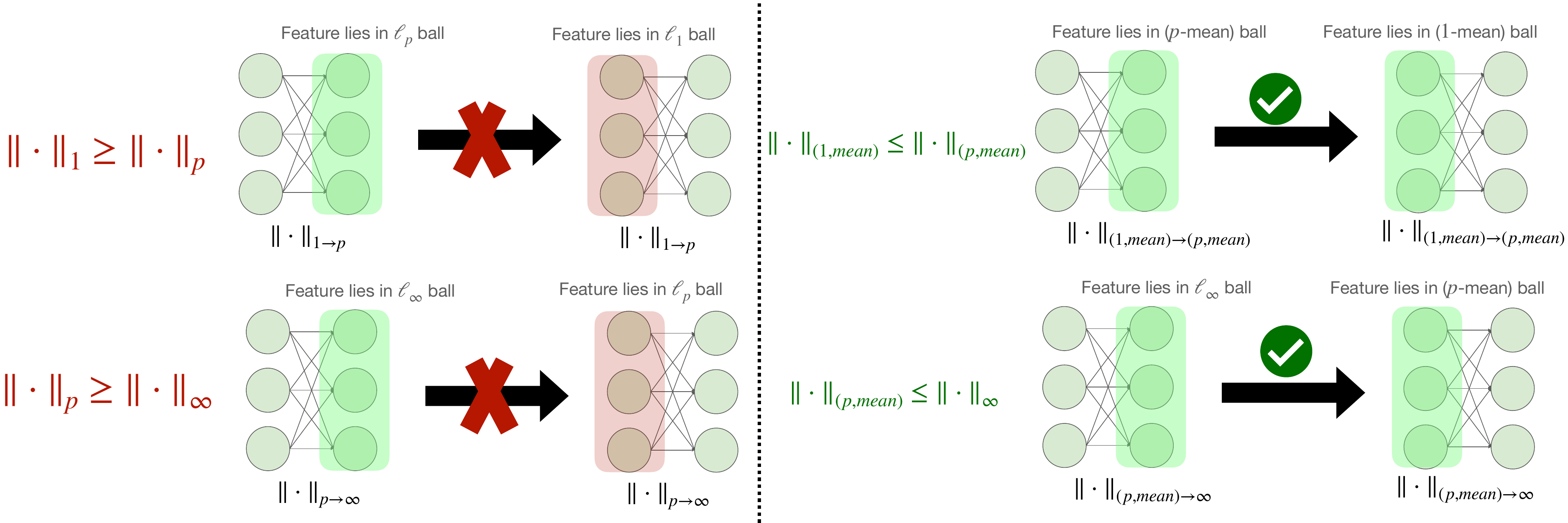}
    \caption{\textbf{Operators Should Play Nice Together.}  Chaining layer-wise stability bounds under $\|\cdot\|_{p \to q}$ requires $\|\cdot\|_p \le \|\cdot\|_q$. This fails for classical $p\to q$ norms when $p\leq q$ but holds for $(p,\textrm{mean}) \to (q,\textrm{mean})$ norms, yielding dimension-independent bounds. 
    % \Jiajin{Maybe we don't need this figure.}
    }
    \label{fig:placeholder}
\end{figure}
\begin{theorem}[Width-independent Lipschitz bound under mean-normalized geometry]
\label{thm:width-lip}
Consider a $K$-layer neural network defined in Definition \ref{def:nn} with activation function $\sigma$ and loss $\mathcal{L}$ satisfying 
Assumptions~\ref{assump:boundedderivativeactivation}--\ref{assump:boundedderivativeloss}. Let $1 \le p \le q < \infty$, and suppose the input satisfies $\|\x\|_2 \le C$ for some $C>1$.
Define the parameter set as 
$
\Omega_C := \{{\bm{\Theta}}:
\|{\bm{\Theta}}\|_{\mathrm{block}} \le C
\},$
where the block norm $\|{\bm{\Theta}}\|_{\mathrm{block}}$ is defined as  
\begin{equation}
\label{eq:block_norm}
 \max\left\{
\left\|\bm{W}_1\right\|_{1\to(q,\mathrm{mean})},
\max_{2\le i\le K-1}\left\|\bm{W}_i\right\|_{(p,\mathrm{mean})\to(q,\mathrm{mean})},
\left\|\bm{W}_K\right\|_{(p,\mathrm{mean})\to\infty},
\max_i \left\|\b_i\right\|_\infty
\right\}.
\end{equation}
Then the loss function is Lipschitz continuous on $\Omega_C$, i.e., 
\[
|f({\bm{\Theta}})-f({\bm{\Theta}}^\prime)|
\le
M_{\mathrm{net}}\,
\left\|{\bm{\Theta}}-{\bm{\Theta}}^\prime\right\|_{\mathrm{block}}
\quad
\forall\, {\bm{\Theta}}, {\bm{\Theta}}^\prime\in\Omega_C,
\]
% % Suppose there exists a constant $C>1$ such that the input, weights, and biases satisfy
% % \[
% % \max\!\left\{
% % \left\|\x\right\|_2,\;
% % \left\|\bm{W}_1\right\|_{1\to(q,\mathrm{mean})},\;
% % \max_{2\le i\le K-1}\left\|\bm{W}_i\right\|_{(p,\mathrm{mean})\to(q,\mathrm{mean})},\;
% % \left\|\bm{W}_K\right\|_{(p,\mathrm{mean})\to\infty},\;
% % \max_i \left\|\b_i\right\|_\infty
% % \right\}
% % \le C.
% % \]
% % Define the block norm as 
% % \begin{align*}
% % \left\|(\bm{W}_{1:K},\b_{1:K})\right\|_{\mathrm{block}}
% % :=
% % \max\!\left\{
% % \left\|\bm{W}_1\right\|_{1\to(q,\mathrm{mean})},\;
% % \max_{2\le i\le K-1}\left\|\bm{W}_i\right\|_{(p,\mathrm{mean})\to(q,\mathrm{mean})},\;
% % \left\|\bm{W}_K\right\|_{(p,\mathrm{mean})\to\infty},\;
% % \max_{1\le i\le K}\left\|\b_i\right\|_\infty
% % \right\}.
% % \end{align*}
% Then for any two parameter tuples $(\bm{W}_{1:K}^1,\b_{1:K}^1)$, $(\bm{W}_{1:K}^2,\b_{1:K}^2)$  in the above set, we have the width-independent Lipschitz estimate
% \[
% \big|f(\bm{W}_{1:K}^1,\b_{1:K}^1)-f(\bm{W}_{1:K}^2,\b_{1:K}^2)\big|
% \;\le\;
% M_{\mathrm{net}}\;
% \left\|(\bm{W}_{1:K}^1-\bm{W}_{1:K}^2,\ \b_{1:K}^1-\b_{1:K}^2)\right\|_{\mathrm{block}},
% \]
where $M_{\mathrm{net}}>0$ depends only on $C$, $K$, and the constants in Assumptions~\ref{assump:boundedderivativeactivation}--\ref{assump:boundedderivativeloss}, and is independent of all layer widths.
\end{theorem}

\begin{remark}[Why the bounded parameter set is natural]
The restriction ${\bm{\Theta}} \in \Omega_C$ is natural in practice, as standard training procedures implicitly control parameter norms. Under the standard decoupled weight-decay update, 
$
{\bm{\Theta}}^{t+1} = (1-\eta\lambda){\bm{\Theta}}^t + \eta \bm{D}^t,
$ 
where $\eta>0$, $0<\lambda<1$ and $\|\bm{D}^t\|\le 1$, we have $
\|{\bm{\Theta}}^{t+1}\| \le (1-\eta\lambda)\|{\bm{\Theta}}^t\| + \eta.
$ Unrolling the recursion gives $
\|{\bm{\Theta}}^t\| \le (1-\eta\lambda)^t \|{\bm{\Theta}}^0\| + 1/\lambda,
$ and hence $\sup_t \|{\bm{\Theta}}^t\| \lesssim  O(1/\lambda)$. Thus, weight decay together with bounded update directions keeps the parameters in a uniformly bounded ball. Moreover, several modern optimizers admit equivalent formulations as optimization under explicit norm constraints
\citep{chen2023lion,chen2025muon,pethick2025training}, further supporting the bounded-parameter assumption.
\end{remark}

\begin{comment}
\begin{remark}
To obtain meaningful Lipschitz estimates for a neural network, one must impose
norm constraints on the weights.  Under the decoupled weight-decay update $\x_{t+1} = (1-\eta\lambda)\x_t + \eta\bD_t$, 
where $\eta>0$ is the learning rate, $\lambda>0$ the weight-decay coefficient, and 
$\|\bD_t\|\le 1$, the parameter norm satisfies the recursion 
$\|\x_{t+1}\|\le (1-\eta\lambda)\|\x_t\|+\eta$. Unrolling this yields 
$\|\x_t\|\le (1-\eta\lambda)^t\|\x_0\| + \eta/(\eta\lambda) = (1-\eta\lambda)^t\|\x_0\| + 1/\lambda$, 
and therefore $\sup_{t\ge 0}\|\x_t\|\lesssim 1/\lambda$. Thus, combining the contraction 
$(1-\eta\lambda)\x_t$ with bounded directions $\bD_t$ keeps the parameters confined 
to a ball of radius $1/\lambda$. Thus, the combination of contraction $(1-\lambda)\x_t$ and bounded directions $\bD_t$
keeps the weights confined to a ball of radius $O(1/\lambda)$, providing a 
uniform norm bound that is crucial for Lipschitz control. This mechanism implicitly enforces a per-iteration 
norm constraint on the parameters. At the same time, recent works \citep{chen2023lion,chen2025muon,pethick2025training}
show that such an update rule is equivalent to minimizing the loss under explicit 
norm constraints. Taken together, modern language-model optimizers implicitly 
control parameter norms, which are essential both for obtaining width-robust 
Lipschitz estimates and for stabilizing neural network training.
\end{remark}
\end{comment}

We now turn to the proof of Theorem~\ref{thm:width-lip}. 

\begin{proof}{Proof of Theorem~\ref{thm:width-lip}.}
We prove the Lipschitz continuity of $f({\bm{\Theta}})=\mathcal L(y_K({\bm{\Theta}};x))$ on $\Omega_C$
with respect to $\|\cdot\|_{\mathrm{block}}$ by invoking Lemma~\ref{lem:Mlip-set}.
It suffices to show the following \emph{uniform directional-derivative bound}:
\begin{equation}\label{eq:goal-dir-bound}
\sup_{\left\|\Delta{\bm{\Theta}}\right\|_{\mathrm{block}}\le 1}\big|\nabla f({\bm{\Theta}})[\Delta{\bm{\Theta}}]\big|
\;\le\; M_{\mathrm{net}},
\qquad \forall\,{\bm{\Theta}}\in\Omega_C,
\end{equation}
for some constant $M_{\mathrm{net}}$ that does not depend on the width $w$.
% Fix any ${\bm{\Theta}}\in\Omega_C$ and any direction $\Delta{\bm{\Theta}}$ with
% $\|\Delta{\bm{\Theta}}\|_{\mathrm{block}}\le 1$, i.e.,
% \[
% \left\|\bm{W}_1\right\|_{1\to(q,\mathrm{mean})}\leq 1,
% \max_{2\le i\le K-1}\left\|\bm{W}_i\right\|_{(p,\mathrm{mean})\to(q,\mathrm{mean})}\leq 1,
% \left\|\bm{W}_K\right\|_{(p,\mathrm{mean})\to\infty}\leq 1,
% \max_i \left\|\b_i\right\|_\infty \leq 1. 
% \]
By the chain rule and Assumption~\ref{assump:boundedderivativeloss}, we have 
$
|\nabla f({\bm{\Theta}})[\Delta{\bm{\Theta}}]|
\le L_J\,|\nabla \y_K({\bm{\Theta}})[\Delta{\bm{\Theta}}]|.
$
Therefore, to prove \eqref{eq:goal-dir-bound}, it is enough to bound
$|\nabla y_K({\bm{\Theta}})[\Delta{\bm{\Theta}}]|$ uniformly over ${\bm{\Theta}}\in\Omega_C$ and
$\|\Delta{\bm{\Theta}}\|_{\mathrm{block}}\le 1$.
We decompose
\[
\nabla \y_K({\bm{\Theta}})[\Delta{\bm{\Theta}}]
=
\sum_{j=1}^K \nabla_{\bm{W}_j}\y_K({\bm{\Theta}})[\Delta \bm{W}_j]
+\sum_{j=1}^K \nabla_{\b_j}\y_K({\bm{\Theta}})[\Delta \b_j].
\]
Then, it suffices to establish width-independent bounds on each term
$\nabla_{\bm{W}_j}\y_K[\Delta \bm{W}_j]$ and $\nabla_{\b_j}\y_K[\Delta \b_j]$.

We first establish a uniform bound on the intermediate hidden states
$y_i({\bm{\Theta}})$ for $i=1,\dots,K-1$.
Since $\sigma(0)=0$ and $|\sigma'(z)|\le L_\sigma$, the mean-value theorem implies
$|\sigma(z)|\le L_\sigma |z|$ for all $z\in\mathbb{R}$. Applying this componentwise yields
\begin{equation}\label{eq:sigma-lip}
\left\|\sigma(\x)\right\|_{(q,\textrm{mean})} \le L_\sigma \left\|\x\right\|_{(q,\textrm{mean})},\forall\,\x\in\mathbb{R}^d \text{ and } q\ge 1.
\end{equation}
Fix any ${\bm{\Theta}}\in\Omega_C$ and $\|\x\|_2\le C$. Using \eqref{eq:sigma-lip} and the definition of 
$\|{\bm{\Theta}}\|_{\mathrm{block}}\le C$, we have
\[
\left\|\y_1\right\|_{(q,\mathrm{mean})}
=\left\|\sigma(\bm{W}_1\x+\b_1)\right\|_{(q,\mathrm{mean})}
\le L_\sigma\left(\left\|\bm{W}_1\right\|_{1\to(q,\mathrm{mean})}\left\|\x\right\|_1+\left\|\b_1\right\|_\infty\right)
\le 2L_\sigma C^2.
\]
For $2\le i\le K-1$, thanks to Fact \ref{fact:pmean-monotone} and $\|{\bm{\Theta}}\|_{\textrm{block}}\leq C$, we obtain
\begin{align}
\left\|\y_i\right\|_{(q,\mathrm{mean})}
&=\left\|\sigma(\bm{W}_i\y_{i-1}+\b_i)\right\|_{(q,\mathrm{mean})}
\le L_\sigma\left(\left\|\bm{W}_i\right\|_{(p,\mathrm{mean})\to(q,\mathrm{mean})}\left\|\y_{i-1}\right\|_{(p,\mathrm{mean})}+\left\|\b_i\right\|_\infty\right)\nonumber\\
&\le L_\sigma(C\left\|\y_{i-1}\right\|_{(q,\mathrm{mean})}+C)
\le 2L_\sigma C\max\{1,\left\|\y_{i-1}\right\|_{(q,\mathrm{mean})}\}\leq (2L_\sigma C)^iC,
% ,\qquad i=1,\dots,K-1. 
\label{eq:yi-rec}
\end{align}
% Iterating \eqref{eq:yi-rec} gives the crude bound
% \begin{equation}\label{eq:yi-bound}
% \left\|\y_i\right\|_{(q,\mathrm{mean})}\le (2L_\sigma C)^i\,C,\qquad i=1,\dots,K-1,
% \end{equation}
which is independent of the width $w$.

We now bound the directional derivative of $\y_K$ with respect to each parameter block. Fix any direction $\Delta{\bm{\Theta}}=(\Delta \bm{W}_{1:K},\Delta \b_{1:K})$ with
$\|\Delta{\bm{\Theta}}\|_{\mathrm{block}}\le 1$.
For each $i=1,\dots,K$, let
 $\bm{\Sigma}_i:=\mathrm{diag}(\sigma'(\z_i))$, so that
$
\|\bm{\Sigma}_i\|_{\infty}\le L_\sigma, i=1,\dots,K.
$
For $1\le j\le i\le K$, we define
\[
\bm{\Lambda}^{\bm{W}}_{i,j}:=\nabla_{\bm{W}_j}\y_i[\Delta \bm{W}_j] \quad \text{and} \quad
\bm{\Lambda}^{\b}_{i,j}:=\nabla_{\b_j}\y_i[\Delta \b_j].
\]
By the chain rule (essentially the backpropogation rule), for each $i=1,\dots,K$ and $j=1,\dots,i$, we have 
\begin{equation}
\label{eq:first_deri}
\begin{aligned}
\bm{\Lambda}^{\bm{W}}_{i,j}
&=
\begin{cases}
\bm{\Sigma}_i(\Delta \bm{W}_j \y_{i-1}), & \mathclap{j=i}\\
\bm{\Sigma}_i \bm{W}_i \bm{\Lambda}^{\bm{W}}_{i-1,j}, & \mathclap{j<i}
\end{cases}\quad \text{ and } \,\,
\bm{\Lambda}^{\b}_{i,j}
=
\begin{cases}
\bm{\Sigma}_i \Delta \b_j, & \mathclap{j=i}\\
\bm{\Sigma}_i\bm{W}_i \bm{\Lambda}^{\b}_{i-1,j}, & \mathclap{j<i}
\end{cases}
\end{aligned}.  
\end{equation}
We now bound these quantities case by case. 

% \medskip
\noindent\textbf{Case 1: $j=i$.} For $2\le i\le K-1$, using Fact~\ref{fact:pmean-monotone},
% $\|\Delta \bm{W}_i\|_{(p,\mathrm{mean})\to(q,\mathrm{mean})}\le 1$, 
and the bound \eqref{eq:yi-rec}, we obtain
\begin{align}
\left\|\bm{\Lambda}^{\bm{W}}_{i,i}\right\|_{(p,\mathrm{mean})}
&\le \left\|\bm{\Lambda}^{\bm{W}}_{i,i}\right\|_{(q,\mathrm{mean})}
= \left\|\bm{\Sigma}_i(\Delta \bm{W}_i \y_{i-1})\right\|_{(q,\mathrm{mean})} \notag\\
&\le \left\|\bm{\Sigma}_i\right\|_{\infty\to\infty}\left\|\Delta \bm{W}_i \y_{i-1}\right\|_{(q,\mathrm{mean})}
\le L_\sigma \left\|\Delta \bm{W}_i\right\|_{(p,\mathrm{mean})\to(q,\mathrm{mean})}\left\|\y_{i-1}\right\|_{(p,\mathrm{mean})}\notag\\
&\le L_\sigma \left\|\y_{i-1}\right\|_{(q,\mathrm{mean})}
\le L_\sigma (2L_\sigma C)^{i-1}C,\notag \text{ and } \\
\left \|\bm{\Lambda}^{\b}_{i,i}\right\|_{(p,\mathrm{mean})}
& \le \left\|\bm{\Lambda}^{\b}_{i,i}\right\|_\infty
= \left\|\bm{\Sigma}_i\Delta \b_i\right\|_\infty
\le \left\|\bm{\Sigma}_i\right\|_{\infty}\left\|\Delta \b_i\right\|_\infty
\le L_\sigma\notag. 
% \label{eq:case1-W-mid}
\end{align}
For $i=1$, we have $\bm{\Lambda}^{\bm{W}}_{1,1}=\bm{\Sigma}_1(\Delta \bm{W}_1 \x)$ and thus
\begin{equation*}
% \label{eq:case1-W-1}
\left\|\bm{\Lambda}^{\bm{W}}_{1,1}\right\|_{(q,\mathrm{mean})}
\le L_\sigma \left\|\Delta \bm{W}_1\right\|_{1\to(q,\mathrm{mean})}\left\|\x\right\|_1
\le L_\sigma C,
\end{equation*}
where the last inequality uses $\|\Delta \bm{W}_1\|_{1\to(q,\mathrm{mean})}\le 1$ and $\|\x\|_1\lesssim C$.
For the last layer $i=K$, since $y_K\in\mathbb{R}$, we have
\begin{align*}
% \label{eq:case1-W-K}
\left|\bm{\Lambda}^{\bm{W}}_{K,K}\right|
& =\left|\bm{\Sigma}_K(\Delta \bm{W}_K \y_{K-1})\right|\le L_\sigma \left\|\Delta \bm{W}_K\right\|_{(p,\mathrm{mean})\to\infty}\left\|\y_{K-1}\right\|_{(p,\mathrm{mean})} \\ 
& \le L_\sigma \left\|\y_{K-1}\right\|_{(q,\mathrm{mean})} \leq  2^{K-1}L_\sigma^KC^{K}. 
\end{align*}
Finally, $|\bm{\Lambda}^{\b}_{1,1}|\le L_\sigma$ and $|\bm{\Lambda}^{\b}_{K,K}|\le L_\sigma$ follow similarly.

\noindent\textbf{Case 2: $j<i$ and $i\neq K$.} 
For $j<i\le K-1$, we have 
\begin{align}
\left\|\bm{\Lambda}^{\bm{W}}_{i,j}\right\|_{(p,\mathrm{mean})}
&\le \left\|\bm{\Lambda}^{\bm{W}}_{i,j}\right\|_{(q,\mathrm{mean})}
= \left\|\bm{\Sigma}_i \bm{W}_i \bm{\Lambda}^{\bm{W}}_{i-1,j}\right\|_{(q,\mathrm{mean})}
\le \left\|\bm{\Sigma}_i\right\|_{\infty\to\infty}\left\|\bm{W}_i \bm{\Lambda}^{\bm{W}}_{i-1,j}\right\|_{(q,\mathrm{mean})}\nonumber\\
&\le L_\sigma \left\|\bm{W}_i\right\|_{(p,\mathrm{mean})\to(q,\mathrm{mean})}\left\|\bm{\Lambda}^{\bm{W}}_{i-1,j}\right\|_{(p,\mathrm{mean})}
\le L_\sigma C\left\|\bm{\Lambda}^{\bm{W}}_{i-1,j}\right\|_{(p,\mathrm{mean})} \\ 
& \le (L_\sigma C)^{i-j}\left\|\bm{\Lambda}^{\bm{W}}_{j,j}\right\|_{(p,\mathrm{mean})}. 
\label{eq:case2-prop-W}
\end{align}
Combining with the result of Case 1 yields a width-independent bound on $\|\bm{\Lambda}^{\bm{W}}_{i,j}\|$. Similarly, we have 
\begin{equation}
\label{eq:b_first_order}
\left\|\bm{\Lambda}^{\b}_{i,j}\right\|_{(p,\mathrm{mean})}
\le (L_\sigma C)^{i-j}\left\|\bm{\Lambda}^{\b}_{j,j}\right\|_{(p,\mathrm{mean})}
\le (L_\sigma C)^{i-j} L_\sigma,
\end{equation}
which is also width-independent.

\noindent\textbf{Case 3: $j<i=K$.}
In this case $y_K\in\mathbb{R}$.  We have 
\begin{align*}
\left|\bm{\Lambda}^{
\bm{W}}_{K,j}\right|
&= \left|\bm{\Sigma}_K \bm{W}_K \bm{\Lambda}^{\bm{W}}_{K-1,j}\right|
\le L_\sigma \left\|W_K\bm{\Lambda}^{W}_{K-1,j}\right\|_\infty
\le L_\sigma \left\|\bm{W}_K\right\|_{(p,\mathrm{mean})\to\infty}\left\|\bm{\Lambda}^{\bm{W}}_{K-1,j}\right\|_{(p,\mathrm{mean})}\nonumber\\
&\le L_\sigma C \left\|\bm{\Lambda}^{W}_{K-1,j}\right\|_{(p,\mathrm{mean})}, \text{ and }\label{eq:case3-W}\\
\left|\bm{\Lambda}^{\b}_{K,j}\right|
& \le L_\sigma C\left\|\bm{\Lambda}^{\b}_{K-1,j}\right\|_{(p,\mathrm{mean})}.  
\end{align*}
Combining the results from Case 1 and 2 yields the width independent bounds. 

By Cases~1--3, there exists a constant $C_\star=C_\star(K,L_\sigma,C)$ independent of the width $w$
such that for all ${\bm{\Theta}}\in\Omega_C$ and all directions $\|\Delta{\bm{\Theta}}\|_{\mathrm{block}}\le 1$,
\[
\left|\nabla_{W_j}y_K({\bm{\Theta}})[\Delta W_j]\right|\le C_\star,
\quad \text{and} \quad 
\left|\nabla_{b_j}y_K({\bm{\Theta}})[\Delta b_j]\right|\le C_\star,
\qquad j=1,\ldots,K.
\]
We conclude our proof.    
\end{proof}

\subsection{Width-independent Smootheness Bound under Mean-Normalized Geometry}
\label{subsection:selectpq}

While \citep{bernstein2024modular,large2024scalable} suggest choosing norms that reflect the output’s sensitivity to input and weight perturbations, Lipschitz continuity alone provides limited structural information and may not adequately capture the dynamics of deep learning optimization. 
In this subsection, we further consider how the gradient itself varies along the optimization trajectory, which characterizes the local oscillation of the landscape. 
This property is known as $L$-smoothness, or gradient Lipschitz continuity, in optimization. 
Specifically, it requires that the gradient changes in a controlled manner, meaning that it does not vary too rapidly between nearby points. 
This property plays a central role in optimization, as it ensures that gradient-based methods can take stable and predictable steps and allows us to bound the decrease in the objective value at each iteration.

To formalize this notion, we introduce the standard concept of $L$-smoothness. 
\begin{definition}[$L$-Smoothness]
Let $f:\Omega \to \mathbb{R}$ be differentiable on a convex set $\Omega$ 
equipped with a norm $\|\cdot\|$. 
We say that $f$ is \textit{$L$-smooth} with respect to $\|\cdot\|$ if, for all 
$\bm{Z},\bm{Z}^\prime \in \Omega$,
\[
\left\|\nabla f(\bm{Z})-\nabla f(\bm{Z}^\prime)\right\|_{*}
\le L \left\|\bm{Z}-\bm{Z}^\prime\right\|,
\]
where $\|\cdot\|_{*}$ denotes the dual norm of $\|\cdot\|$.
\end{definition} 

% \begin{definition}[$L$-Smoothness, \citep{li2025note,pethick2025training,shen2025convergence}] Let $f: \bB \to \mathbb{R}$. We say that $f$ is \textit{$L$-smooth} with respect to the norm $\|\cdot\|_{\bB}$ if, for all $\x_1, \x_2 \in \bB$,  
% \[
% \left\|\nabla f(\x_1) - \nabla f(\x_2)\right\|_{\bB^\ast} \le L \, \left\|\x_1 - \x_2\right\|_{\bB}, 
% \]  
% where $\|\cdot\|_{\bB^\ast}$ denotes the dual norm of $\|\cdot\|_{\bB}$.  
% \end{definition}
% \begin{remark} Recent works \citep{li2025note,pethick2025training,shen2025convergence,glentis2025minimalist,riabinin2025gluon,kovalev2025understanding} have shown that, under a $L$-smoothness assumption, convergence results can be dimension independent, depending only on $L$. 
% \end{remark}

Recent works 
\citep{li2025note,pethick2025training,shen2025convergence,glentis2025minimalist,riabinin2025gluon,kovalev2025understanding}
have shown that, under an $L$-smoothness assumption, convergence guarantees of first-order methods can be dimension independent, depending only on the smoothness constant $L$. 
This connection can be understood through the standard descent analysis under a general norm geometry.  In particular, the steepest descent direction under $\|\cdot\|$ is defined by 
\[
\bm{D}^t
=
\argmin_{\left\|\bm{D}\right\|\leq 1}
\langle \nabla f(\bm{Z}^t), \bm{D}\rangle
=
-\frac{\nabla f(\bm{Z}^t)}{\left\|\nabla f(\bm{Z}^t)\right\|_*}, 
\]
and the update rule becomes $\bm{Z}^{t+1}=\bm{Z}^t-\eta^t \bm{D}^t$.
Since $\langle \nabla f(\bm{Z}^t), \bm{D}^t\rangle=-\|\nabla f(\bm{Z}^t)\|_{*}$,
the $L$-smoothness condition implies 
$
f(\bm{Z}^{t+1})
\le
f(\bm{Z}^t)
-
\eta^t \|\nabla f(\bm{Z}^t)\|_{*}
+
\frac{L}{2}\eta_t^2 .
$
Choosing the optimal step size
$
\eta^t
=
\frac{\|\nabla f(\bm{Z}^t)\|_{*}}{L}
$
yields
\[
f(\bm{Z}^{t+1})
\le
f(\bm{Z}^t)
-
\frac{1}{2L}\left\|\nabla f(\bm{Z}^t)\right\|_{*}^2 .
\]
This analysis shows that the optimal learning rate is governed by the smoothness constant $L$. 
Therefore, obtaining a width-independent bound on $L$ is crucial for deriving learning-rate rules that remain stable as model width grows, which motivates our study of width-independent smoothness bounds. 

Motivated by the above discussion, we next establish width-independent
bounds on the smoothness constant under the mean-normalized geometries $(p,\mathrm{mean}) \to (q,\mathrm{mean})$ considered in this work.
Our main result is stated as follows. 

\begin{theorem}\label{theorem:Lsmooth} Consider a $K$-layer neural network defined in Definition \ref{def:nn} with activation function $\sigma$ and loss $\mathcal{L}$ satisfying Assumptions \ref{assump:boundedderivativeactivation}--\ref{assump:boundedderivativeloss}. Let $1\le p\le q\le \infty$, and suppose the input satisfies $\|\x\|_2\leq C$ for some $C>1$. Define the parameter set as 
$
\Omega_C := \{{\bm{\Theta}}:
\|{\bm{\Theta}}\|_{\mathrm{block}} \le C
\},$
where $\|\bm{\Theta}\|_{\mathrm{block}}$ is defined in \eqref{eq:block_norm}. 
Then, we have
\[
\left\|\nabla f({\bm{\Theta}})-\nabla f({\bm{\Theta}}^\prime)\right\|_{\text{block},\star}
\le
L_{\mathrm{net}}\, w^{\max\left(0,\,\tfrac{2}{q}-\tfrac{1}{p}\right)}
\left\|{\bm{\Theta}}-{\bm{\Theta}}^\prime\right\|_{\mathrm{block}}
\quad
\forall\, {\bm{\Theta}}, {\bm{\Theta}}^\prime\in\Omega_C,
\]
where $L_{\mathrm{net}}>0$ depends only on $C,K$, and the constants in Assumptions~\ref{assump:boundedderivativeactivation}--\ref{assump:boundedderivativeloss}
, and is independent of all layer widths.
\end{theorem}

\begin{remark}
As shown in Theorem~\ref{theorem:Lsmooth}, the smoothness constant exhibits
a width dependence of order $w^{\max\{0,\,2/q-1/p\}}$ under the
$(p,\mathrm{mean}) \to (q,\mathrm{mean})$ geometry.
In particular, the smoothness becomes width independent whenever
$q \ge 2p$.
Therefore, achieving width-independent smoothness naturally motivates
considering geometries with sufficiently large output norms.
In addition, for practical optimization algorithms, the steepest descent
direction should admit a closed-form expression.
Combining these two considerations, our analysis focuses on the
$(1,\mathrm{mean}) \to q$ geometries with $q \ge 2$ and the
$(p,\mathrm{mean}) \to \infty$ geometries.
\end{remark}

Before presenting the proof of Theorem~\ref{theorem:Lsmooth}, we now explain how the width-dependent factor
$w^{\max\{0,\,2/q-1/p\}}$ arises in the smoothness bound under the
$\pmean\to\qmean$ geometry.
Recall that $L$-smoothness is defined with respect to perturbations in the
parameter space.
By Lemma~\ref{lem:directional_hessian}, it suffices to bound the directional
Hessian
$\nabla^2 f(\bm{\Theta})[\Delta\bm{\Theta}^1,\Delta\bm{\Theta}^2]$
uniformly over unit perturbations
$\|\Delta\bm{\Theta}^1\|_{\mathrm{block}},\|\Delta\bm{\Theta}^2\|_{\mathrm{block}}\le 1$. To see the key mechanism, consider a single layer map
$\sigma(\bm W\x)$ with a perturbation $\bm W\mapsto \bm W+\Delta\bm W$.
Since $\bm W\x$ is linear in $\bm W$, the only nontrivial second-order term
comes from the feature-wise nonlinearity $\sigma$. In particular, the second-order variation contains the quadratic activation
perturbation
$(\Delta\bm W\x)\odot(\Delta\bm W\x)$.
Using boundedness of $\sigma''$ and the loss derivatives
(Assumptions~\ref{assump:boundedderivativeactivation}--\ref{assump:boundedderivativeloss}),
the directional Hessian can be reduced to controlling this quadratic term:
\[
\bigl|\nabla^2 f(\bm{W})[\Delta\bm{W}^1,\Delta\bm{W}^2]\bigr|
\;\lesssim\;
\left\|(\Delta\bm W^1\x)\odot(\Delta\bm W^2\x)\right\|_{(p,\mathrm{mean})}.
\]
 Under the $\pmean\to\qmean$ operator-norm geometry,
$\|\Delta\bm W^i\x\|_{(q,\mathrm{mean})}$ is directly controlled by
$\|\Delta\bm W^i\|_{(p,\mathrm{mean})\to(q,\mathrm{mean})}\leq 1$.
To bound the quadratic term in the $(p,\mathrm{mean})$ norm, we invoke
Lemma~\ref{lem:montomeannorm}, which yields
\[
\left\|(\Delta\bm W^1\x)\odot(\Delta\bm W^2\x)\right\|_{(p,\mathrm{mean})}
\;\lesssim\;
w^{\max\left(0,\,\tfrac{2}{q}-\tfrac{1}{p}\right)}
\left\|\Delta\bm W^1\x\right\|_{(q,\mathrm{mean})}\,
\left\|\Delta\bm W^2\x\right\|_{(q,\mathrm{mean})}.
\]
Combining this with the operator-norm bound
$\|\Delta\bm W^i\x\|_{(q,\mathrm{mean})}
\le
\|\Delta\bm W^i\|_{(p,\mathrm{mean})\to(q,\mathrm{mean})}\|\x\|_{(p,\mathrm{mean})}$
shows that the directional Hessian---and hence the smoothness constant---scales
as $w^{\max\{0,\,2/q-1/p\}}$.

With the above preparation, we are now ready to present the proof of
Theorem~\ref{theorem:Lsmooth}. 
\begin{proof}{Proof of Theorem~\ref{theorem:Lsmooth}.}
We prove the $L$-smoothness of
$f({\bm{\Theta}})=\mathcal L(y_K({\bm{\Theta}};x))$
on $\Omega_C$ with respect to $\|\cdot\|_{\mathrm{block}}$
by invoking Lemma~\ref{lem:directional_hessian}. It suffices to establish the following uniform directional-Hessian bound:
\begin{equation}\label{eq:goal-dir-hessian-bound}
\sup_{\left\|\Delta{\bm{\Theta}}^1\right\|_{\mathrm{block}}\le 1,\,
      \left\|\Delta{\bm{\Theta}}^2\right\|_{\mathrm{block}}\le 1}
\big|\nabla^2 f({\bm{\Theta}})[\Delta{\bm{\Theta}}^1,\Delta{\bm{\Theta}}^2]\big|
\le L_{\mathrm{net}},
\qquad
\forall\,{\bm{\Theta}}\in\Omega_C,
\end{equation}
for some constant $L_{\mathrm{net}}$. By the chain rule and
Assumptions~\ref{assump:boundedderivativeactivation}
--\ref{assump:boundedderivativeloss},
we have
\[
\bigl|\nabla^2 f({\bm{\Theta}})
    [\Delta{\bm{\Theta}}^1,\Delta{\bm{\Theta}}^2]\bigr|
\le
M_J\,
\bigl|\nabla y_K({\bm{\Theta}})
        [\Delta{\bm{\Theta}}^1]\bigr|
\,
\bigl|\nabla y_K({\bm{\Theta}})
        [\Delta{\bm{\Theta}}^2]\bigr|
+
L_J\,
\bigl|\nabla^2 y_K({\bm{\Theta}})
        [\Delta{\bm{\Theta}}^1,\Delta{\bm{\Theta}}^2]\bigr|.
\]
Therefore, to prove \eqref{eq:goal-dir-hessian-bound}, it suffices to bound
$\nabla y_K({\bm{\Theta}})[\Delta{\bm{\Theta}}]$
and
$\nabla^2 y_K({\bm{\Theta}})
[\Delta{\bm{\Theta}}^1,\Delta{\bm{\Theta}}^2]$
uniformly over ${\bm{\Theta}}\in\Omega_C$
with
$\|\Delta{\bm{\Theta}}\|_{\mathrm{block}}\le 1$ and
$\|\Delta{\bm{\Theta}}^1\|_{\mathrm{block}},
 \|\Delta{\bm{\Theta}}^2\|_{\mathrm{block}}\le 1$.
The first-order term
$\nabla y_K({\bm{\Theta}})[\Delta{\bm{\Theta}}]$
already admits a width-independent bound from
Theorem~\ref{thm:width-lip}.
Hence, it remains to bound the second-order directional derivative.

To proceed, fix any directions 
$\Delta{\bm{\Theta}^1}=(\Delta \bm{W}_{1:K}^1,\Delta \b_{1:K}^1)$ and
$\Delta{\bm{\Theta}^2}=(\Delta \bm{W}_{1:K}^2,\Delta \b_{1:K}^2)$
with
$\|\Delta{\bm{\Theta}^1}\|_{\mathrm{block}}\le 1$ and
$\|\Delta{\bm{\Theta}^2}\|_{\mathrm{block}}\le 1$.
For each $i=1,\dots,K$, let
$
\bm{\Sigma}_i:=\mathrm{diag}(\sigma'(\z_i))$ and 
$\bm{\Gamma}_{i}:=\mathrm{diag}(\sigma''(\z_i))$ so that $\|\bm{\Sigma}_{i}\|_\infty\leq L_\sigma$ and $\|\bm{\Gamma}_{i}\|_\infty\leq M_\sigma$
For $1\le j,\ell\le i\le K$, we introduce the following shorthand
notations for the first- and second-order directional derivatives: 
\begin{align*}
\bm{\Lambda}^{\bm{W}^1}_{i,j}
&:= \nabla_{\bm{W}_j}\y_i[\Delta \bm{W}_j^1],
\qquad
&
\bm{\Lambda}^{\b^1}_{i,j}
&:= \nabla_{\b_j}\y_i[\Delta \b_j^1],\\
\bm{\Lambda}^{\bm{W}^2}_{i,j}
&:= \nabla_{\bm{W}_j}\y_i[\Delta \bm{W}_j^2],
\qquad
&
\bm{\Lambda}^{\b^2}_{i,j}
&:= \nabla_{\b_j}\y_i[\Delta \b_j^2],\\[4pt]
\Xi^{\bm{W}}_{i,j,\ell}
&:= \nabla^2_{\bm W_j,\bm W_\ell}\y_i
[\Delta \bm{W}_j^1,\Delta \bm{W}_\ell^2],
\qquad
&
\Xi^{\bm{b}}_{i,j,\ell}
&:= \nabla^2_{\bm b_j,\bm b_\ell}\y_i
[\Delta \bm{b}_j^1,\Delta \bm{b}_\ell^2],\\
\Xi^{\bm{W},\bm{b}}_{i,j,\ell}
&:= \nabla^2_{\bm W_j,\bm b_\ell}\y_i
[\Delta \bm{W}_j^1,\Delta \bm{b}_\ell^2].
\end{align*}

Since the first-order derivatives have already been computed in
\eqref{eq:first_deri}, we next compute the second-order directional
derivatives by differentiating \eqref{eq:first_deri} once more and
applying the second-order chain rule. Recall that $\y_0 = \x$.
For any parameter perturbations
$\Delta{\bm{\Theta}}^1, \Delta{\bm{\Theta}}^2$, and for each
$i=1,\dots,K$, we have
\[
\nabla^2 \y_i({\bm{\Theta}})
[\Delta{\bm{\Theta}}^1,\Delta{\bm{\Theta}}^2]
=
\bm{\Gamma}_i \left(
\nabla \z_i({\bm{\Theta}})[\Delta{\bm{\Theta}}^1]
\odot
\nabla \z_i({\bm{\Theta}})[\Delta{\bm{\Theta}}^2]
\right)
+
\bm{\Sigma}_i \,
\nabla^2 \z_i({\bm{\Theta}})
[\Delta{\bm{\Theta}}^1,\Delta{\bm{\Theta}}^2],
\]
where $\z_i = \bm W_i \y_{i-1} + \b_i$, and
\[
\nabla \z_i({\bm{\Theta}})[\Delta{\bm{\Theta}}]
=
\Delta \bm W_i \, \y_{i-1}
+
\bm W_i \nabla \y_{i-1}({\bm{\Theta}})[\Delta{\bm{\Theta}}]
+
\Delta \b_i,
\]
\begin{align*}
    & \nabla^2 \z_i(\bm{\Theta})[\Delta
    \bm{\Theta}^1,\Delta\bm{\Theta}^2] \\
    = & \Delta \bm{W}_i^1\,\nabla \y_{i-1}(\bm{\Theta})[\Delta\bm{\Theta}^2] +\Delta \bm{W}_i^2\,\nabla \y_{i-1}(\bm{\Theta})[\Delta\bm{\Theta}^1] +\bm{W}_i\nabla^2 \y_{i-1}(\bm{\Theta})[\Delta\bm{\Theta}^1,\Delta\bm{\Theta}^2].
\end{align*}
We decompose the second-order directional derivative layerwise as
\begin{align*}
 \nabla^2 \y_K({\bm{\Theta}})
[\Delta{\bm{\Theta}}^1,\Delta{\bm{\Theta}}^2] 
&= 
\sum_{1\le j,\ell\le K}
\left(
\nabla^2_{\bm W_j,\bm W_\ell} \y_K
[\Delta\bm W_j^1,\Delta\bm W_\ell^2]
+
\nabla^2_{\bm W_j,\bm b_\ell} \y_K
[\Delta\bm W_j^1,\Delta\bm b_\ell^2]\right) \\
&+ \sum_{1\le j,\ell\le K}
\left(\nabla^2_{\bm b_j,\bm W_\ell} \y_K
[\Delta\bm b_j^1,\Delta\bm W_\ell^2]
+
\nabla^2_{\bm b_j,\bm b_\ell} \y_K
[\Delta\bm b_j^1,\Delta\bm b_\ell^2]
\right).
\end{align*}
The two mixed terms
$\nabla^2_{\bm W_j,\bm b_\ell}\y_K[\cdot,\cdot]$
and
$\nabla^2_{\bm b_j,\bm W_\ell}\y_K[\cdot,\cdot]$
can be treated analogously and therefore admit the same bound.
Hence, it suffices to establish bounds for the three types of terms
introduced above. We now specialize the chain-rule identity to
layerwise perturbations to derive recursion formulas for the
second-order quantities. In what follows, we only derive recursion formulas for $\Xi^{\bm W}_{i,j,\ell}$ and
$\Xi^{\bm b}_{i,j,\ell}$ with $j\le \ell$.
The remaining cases $j>\ell$ follow by symmetry of the Hessian. 
% Therefore, it suffices to establish bounds for the
% three types of terms above.  Fix any direction $\Delta{\bm{\Theta}^1}=(\Delta \bm{W}_{1:K}^1,\Delta \b_{1:K}^1)$, $\Delta{\bm{\Theta}^2}=(\Delta \bm{W}_{1:K}^2,\Delta \b_{1:K}^2)$  with
% $\|\Delta{\bm{\Theta}^1}\|_{\mathrm{block}}\le 1$ and $\|\Delta{\bm{\Theta}^2}\|_{\mathrm{block}}\le 1$.
% For each $i=1,\dots,K$, let
%  $\bm{\Sigma}_i:=\mathrm{diag}(\sigma'(\z_i))$ and $\bm{\Gamma}_{i}:=\mathrm{diag}(\sigma''(\z_i))$. 
% For $1\le j,\ell\le i\le K$, we introduce the shorthand notations
% for the first- and second-order directional derivatives defined below:

% \begin{align*}
% &\bm{\Lambda}^{\bm{W}^1}_{i,j}:=\nabla_{\bm{W}_j}\y_i[\Delta \bm{W}_j^1], \quad
% \bm{\Lambda}^{\b^1}_{i,j}:=\nabla_{\b_j}\y_i[\Delta \b_j^1], \quad \bm{\Lambda}^{\bm{W}^2}_{i,j}:=\nabla_{\bm{W}_j}\y_i[\Delta \bm{W}_j^2], \quad
% \bm{\Lambda}^{\b^2}_{i,j}:=\nabla_{\b_j}\y_i[\Delta \b_j^2], \\
% & \Xi^{\bm{W}}_{i,j,\ell}:= \nabla^2_{\bm W_j,\bm W_\ell}\y_i [\Delta \bm{W}_j^1,\Delta \bm{W}_\ell^2], \,\,  \Xi^{\bm{b}}_{i,j,\ell}:= \nabla^2_{\bm b_j,\bm b_\ell}\y_i [\Delta \bm{b}_j^1,\Delta \bm{b}_\ell^2], \,\,    \Xi^{\bm{W},\bm{b}}_{i,j,\ell}:= \nabla^2_{\bm W_j,\bm b_\ell}\y_i [\Delta \bm{W}_j^1,\Delta \bm{b}_\ell^2]. 
%  \end{align*}

\begin{align*}
&\Xi^{\bm{W}}_{i,j,\ell}
=
\begin{cases}
\bm{\Gamma}_{i}\left((\Delta\bm W_j^1 \y_{i-1}) \odot (\Delta\bm W_\ell^2 \y_{i-1})\right),
& j=\ell=i, \\
\bm{\Gamma}_{i}\left((\bm{W}_i \bm{\Lambda}^{\bm{W}^1}_{i-1,j})
\odot (\Delta\bm W_\ell^2 \y_{i-1})\right)
+ \bm{\Sigma}_i \,\Delta\bm W_\ell^2\,\bm{\Lambda}^{\bm{W}^1}_{i-1,j},
& j<\ell=i, \\
\bm{\Gamma}_{i}\left((\bm{W}_i \, \bm{\Lambda}^{\bm{W}^1}_{i-1,j})
\odot (\bm{W}_i \, \bm{\Lambda}^{\bm{W}^2}_{i-1,\ell})\right)
+ \bm{\Sigma}_i \bm{W}_i \,\Xi^{\bm{W}}_{i-1,j,\ell},
& \ell<i.
\end{cases}
\\[0.6em]
&\Xi^{\bm{b}}_{i,j,\ell}
=
\begin{cases}
\bm{\Gamma}_{i}\left(\Delta\bm b_j^1 \odot \Delta\bm b_\ell^2 \right),
& j=\ell=i, \\[0.4em]
\bm{\Gamma}_{i}\left((\bm{W}_i \, \bm{\Lambda}^{\b^1}_{i-1,j}) \odot \Delta\bm b_\ell^2 \right),
& j<\ell=i, \\[0.4em]
\bm{\Gamma}_{i}\left((\bm{W}_i \, \bm{\Lambda}^{\b^1}_{i-1,j})
\odot (\bm{W}_i \, \bm{\Lambda}^{\b^2}_{i-1,\ell})\right)+ \bm{\Sigma}_i \bm{W}_i \,\Xi^{\bm{b}}_{i-1,j,\ell},
& \ell<i.
\end{cases}
\\[0.6em]
&\Xi^{\bm{W},\bm{b}}_{i,j,\ell}
=
\begin{cases}
\bm{\Gamma}_{i}\left((\Delta\bm W_j^1 \y_{i-1}) \odot \Delta\bm b_\ell^2 \right),
& j=\ell=i, \\[0.4em]
\bm{\Gamma}_{i}\left((\bm{W}_i \, \bm{\Lambda}^{\bm{W}^1}_{i-1,j}) \odot \Delta\bm b_\ell^2 \right),
& j<\ell=i, \\[0.4em]
% \bm{\Gamma}_{i}\left((\bm{W}_i \, \bm{\Lambda}^{\bm{W}^1}_{i-1,j}) \odot \Delta\bm b_\ell^2 \right)+\bm{\Sigma}_i\Delta \bm W_j^1\bm{\Lambda}^{\b^2}_{i-1,\ell},
\bm{\Gamma}_i\left((\Delta \bm{W}_i^1\y_{i-1}) \odot (\bm{W}_i\bm{\Lambda}_{i-1,\ell}^{\b^2})\right) + \bm{\Sigma}_i\Delta\bm{W}_j^1\bm{\Lambda}_{i-1,\ell}^{\b^2}, 
& \ell<j=i, \\[0.4em]
\bm{\Gamma}_{i}\left((\bm{W}_i \, \bm{\Lambda}^{\bm{W}^1}_{i-1,j})
\odot (\bm{W}_i \, \bm{\Lambda}^{\b^2}_{i-1,\ell})\right)+ \bm{\Sigma}_i \bm{W}_i \,\Xi^{\bm{W},\bm{b}}_{i-1,j,\ell},
& \ell, j<i.
\end{cases}
\end{align*}
We now bound these quantities case by case. 

% \medskip
\noindent\textbf{Case 1: $j=\ell=i$.} For $2\le i\le K-1$, we have 
\begin{align*}
        &  \left\|\Xi^{\bm{W}}_{i,i,i}\right\|_\pmean\\
        \le & \left\|\bm{\Gamma}_{i}\right\|_\infty \left\|(\Delta \bm W_i^1 \y_{i-1}) \odot (\Delta \bm W_i^2 \y_{i-1})\right\|_\pmean\\
        \le & M_\sigma w^{\max\left(0,\,\tfrac{2}{q}-\tfrac{1}{p}\right)} \left\|\Delta \bm W_i^1 \y_{i-1}\right\|_\qmean \left\|\Delta \bm W_i^2 \y_{i-1}\right\|_\qmean 
        \\
        \le & M_\sigma w^{\max\left(0,\,\tfrac{2}{q}-\tfrac{1}{p}\right)}\left\|\Delta \bm W_i^1\right\|_{_\pmean\to_\qmean} \left\|\y_{i-1}\right\|_\pmean \left\|\Delta \bm W_i^2\right\|_{_\pmean\to_\qmean}\left\|\y_{i-1}\right\|_\pmean\\
        \le & M_\sigma w^{\max\left(0,\,\tfrac{2}{q}-\tfrac{1}{p}\right)}\left\|\y_{i-1}\right\|_\qmean^2\\
        \le & M_\sigma 2^{2i-2}L_\sigma^{2i-2}C^{2i} w^{\max\left(0,\,\tfrac{2}{q}-\tfrac{1}{p}\right)},\\
\end{align*}
where the second inequality follows from Lemma \ref{lem:montomeannorm}, the fourth one is due to Fact \ref{fact:pmean-monotone}, and the last one is from \eqref{eq:yi-rec}. Similarly, we have 
\begin{align*}
        &\left\|\Xi^{\bm b}_{i,i,i}\right\|_\pmean\le \left\|\bm{\Gamma}_{i}\right\|_\infty\left\|\Delta\bm b_i^1\odot \Delta\bm b_i^2\right\|_\infty\le M_\sigma, \text{and}
\end{align*} 
\begin{align*}
     \left\|\Xi^{\bm{W},\bm b}_{i,i,i}\right\|_\pmean & \le \left\|\bm{\Gamma}_{i}\right\|_\infty \left\|(\Delta \bm W_i^1 \y_{i-1}) \odot \Delta\bm b_i^2\right\|_\pmean\\ &
     \le  M_\sigma \left\|\Delta\bm b_i^2\right\|_\infty \left\|\Delta \bm W_i^1 \y_{i-1}\right\|_\pmean\\
        &\le M_\sigma \left\|\Delta \bm W_i^1\right\|_{_\pmean\to_\qmean} \left\|\y_{i-1}\right\|_\pmean\\
        &\le  M_\sigma \left\|\Delta \bm W_i^1\right\|_{_\pmean\to_\qmean} \left\|\y_{i-1}\right\|_\qmean
        \le M_\sigma2^{i-1}L_\sigma^{i-1}{C^{i}}.
\end{align*}
Similarly, for the boundary layers $i=1$ and $i=K$, by replacing
$\|\cdot\|_{_\pmean\to_\qmean}$ with $\|\cdot\|_{_1\to_\qmean}$ and
$\|\cdot\|_{_\pmean\to\infty}$ respectively, we obtain the following bounds,
\begin{align*}
&\left\|\Xi^{\bm W}_{1,1,1}\right\|_\pmean
\le
M_\sigma  C^2 w^{\max\left(0,\frac{2}{q}-\frac{1}{p}\right)},\\
&\left|\Xi^{\bm W}_{K,K,K}\right|
\le
M_\sigma
2^{2K-2}L_\sigma^{2K-2}C^{2K}w^{\max\left(0,\frac{2}{q}-\frac{1}{p}\right)},\\
& \left\|\Xi^{\bm b}_{1,1,1}\right\|_\pmean
\le
\left\|\Xi^{\bm b}_{1,1,1}\right\|_\infty
\le
\left\|\bm \Gamma_{1}\right\|_\infty
\left\|\Delta\bm b_1^1\odot \Delta\bm b_1^2\right\|_\infty
\le
M_\sigma,\\
&\left|\Xi^{\bm b}_{K,K,K}\right|
\le
\left\|\bm \Gamma_{K}\right\|_\infty
\left\|\Delta\bm b_K^1\odot \Delta\bm b_K^2\right\|_\infty
\le
M_\sigma,\\
    & \left\|\Xi^{\bm W,\bm b}_{1,1,1}\right\|_\pmean\le \left\|\bm \Gamma_{1}\right\|_\infty\left\|(\Delta \bm W_1^1\y_0)\odot \Delta\bm b_1^2\right\|_\infty\le M_\sigma \sqrt{d} C,\\
     &\left|\Xi^{\bm W,\bm b}_{K,K,K}\right|\le \left\|\bm \Gamma_{K}\right\|_\infty\left\|(\Delta \bm W_K^1\y_{K-1})\odot \Delta\bm b_{K}^2\right\|_\infty\le M_\sigma 2^{K-1}L_\sigma^{K-1}C^K.
\end{align*}
\noindent\textbf{Case 2: $\min\{j,\ell\}<\max\{j,\ell\}=i$.} For $2\leq i\leq K-1$, we first bound the weight--weight and bias--bias blocks for $j<\ell=i$.
The remaining case  $\ell<j=1$ follows by symmetry.
For the mixed blocks, both boundary configurations $\Xi^{\bm W,\bm b}_{i,j,i} (j<\ell=i)$ and
$\Xi^{\bm W,\bm b}_{i,i,\ell} (\ell <j=i)$ appear and will be bounded separately.
\begin{align*}
    & \left\|\Xi^{\bm{W}}_{i,j,i}\right\|_\pmean \\=&\,\left\|\bm{\Gamma}_{i}\left((\bm{W}_i \, \bm{\Lambda}^{\bm{W}^1}_{i-1,j}) \odot (\Delta \bm W_i^2 \y_{i-1})\right)+ \bm{\Sigma}_i \Delta \bm W_i^2 \,\bm{\Lambda}^{\bm{W}^1}_{i-1,j}\right\|_{\pmean} \\
    \leq & \, \left\|\bm{\Gamma}_{i}\right\|_\infty \left\|\left((\bm{W}_i \, \bm{\Lambda}^{\bm{W}^1}_{i-1,j}) \odot (\Delta \bm W_i^2 \y_{i-1})\right)\right\|_{\pmean}+ \left\|\bm{\Sigma}_i \Delta \bm W_i^2 \,\bm{\Lambda}^{\bm{W}^1}_{i-1,j}\right\|_{\pmean} \\
    \leq &\, M_\sigma w^{\max\left(0,\,\tfrac{2}{q}-\tfrac{1}{p}\right)} \left\|\bm{W}_i \, \bm{\Lambda}^{\bm{W}^1}_{i-1,j}\right\|_\qmean \left\|\Delta \bm W_i^2\y_{i-1}\right\|_\qmean\\
 &+ L_\sigma\left\|\Delta \bm W_i^2 \right\|_{_\pmean\to_\qmean}\left\|\bm{\Lambda}^{\bm{W}^1}_{i-1,j}\right\|_\pmean\\     \le &\,  M_\sigma w^{\max\left(0,\,\tfrac{2}{q}-\tfrac{1}{p}\right)} \left\|\bm{W}_i \right\|_{_\pmean\to_\qmean} \left\| \bm{\Lambda}^{\bm{W}^1}_{i-1,j}\right\|_\pmean \left\|\Delta \bm W_i^2\right\|_{_\pmean\to_\qmean}\left\| \y_{i-1}\right\|_\qmean \\
  &+ L_\sigma\left\|\Delta \bm W_i^2 \right\|_{_\pmean\to_\qmean}\left\|\bm{\Lambda}^{\bm{W}^1}_{i-1,j}\right\|_\pmean \\
  \leq &  M_\sigma w^{\max\left(0,\,\tfrac{2}{q}-\tfrac{1}{p}\right)} 
  C \left\| \bm{\Lambda}^{\bm{W}^1}_{i-1,j}\right\|_\pmean \left\| \y_{i-1}\right\|_\qmean + L_\sigma \left\|\bm{\Lambda}^{\bm{W}^1}_{i-1,j}\right\|_\pmean  \\ 
 \le & \,  M_\sigma w^{\max\left(0,\,\tfrac{2}{q}-\tfrac{1}{p}\right)} 2^{2i-3}L_\sigma^{2i-2}C^{2i} +2^{i-2}L_\sigma^iC^{i-1},
\end{align*}
where the second inequality follows from Lemma \ref{lem:montomeannorm} and the definition of the operator norm, and the last one is from \eqref{eq:yi-rec} and \eqref{eq:case2-prop-W}. 
Similarly, we have 
\begin{align*}
    \left\|\Xi^{\bm b}_{i,j,i}\right\|_\pmean \le & \,  \left\|\bm{\Gamma}_{i}\right\|_\infty \left\|(\bm{W}_i\bm{\Lambda}^{\b^1}_{i-1,j})\odot \Delta\bm b_{i}^2\right\|_\pmean\\
    \le& \, M_\sigma w^{\max\left(0,\,\tfrac{2}{q}-\tfrac{1}{p}\right)}\left\|\bm W_i\right\|_{_\pmean\to_\qmean}\left\|\bm{\Lambda}^{\b^1}_{i-1,j}\right\|_\pmean\left\|\Delta\bm b_{i}^2\right\|_{\infty}\\
    \le & M_\sigma w^{\max\left(0,\,\tfrac{2}{q}-\tfrac{1}{p}\right)}L_\sigma^{i-1} C^{i-1},
\end{align*}
where the last inequality follows from \eqref{eq:b_first_order}. Then, we continue to bound two cross terms:  
\begin{align*}
\left\|\Xi^{\bm W, \bm b}_{i,j,i}\right\|_\pmean & = \left\|\bm{\Gamma}_{i}\left((\bm{W}_i \, \bm{\Lambda}^{\bm{W}^1}_{i-1,j}) \odot \Delta \b_i^2\right)\right\|_{\pmean}\\
    &\le \left\|\bm{\Gamma}_{i}\right\|_\infty\left\|\bm{W}_i \, \bm{\Lambda}^{\bm{W}^1}_{i-1,j}\right\|_\qmean \left\| \Delta \b_i^2\right\|_\infty\\
    &\le M_\sigma \left\|\bm{W}_i \right\|_{_\pmean\to_\qmean} \left\| \bm{\Lambda}^{\bm{W}^1}_{i-1,j}\right\|_\pmean \le M_\sigma 2^{i-2}L_\sigma^{i-1}C^{i}, \text{$(j<\ell=i)$}
\end{align*}
\begin{align*}
     &\, \left\|\Xi^{\bm W, \bm b}_{i,i,\ell}\right\|_\pmean \\
     = &\,  \left\|\bm{\Gamma}_{i}\left((\Delta\bm{W}_i \, \y_{i-1}) \odot (\bm{W}_i \, \bm{\Lambda}^{\b^2}_{i-1,\ell})\right)+ \bm{\Sigma}_i\, \Delta \bm W_i^1 \,\bm{\Lambda}^{\b^2}_{i-1,\ell}\right\|_{\pmean}\\
    \le & \, \left\|\bm{\Gamma}_{i}\right\|_\infty\left\|(\Delta\bm{W}_i \, \y_{i-1}) \odot (\bm{W}_i \, \bm{\Lambda}^{\b^2}_{i-1,\ell})\right\|_{\pmean}+ \left\|\bm{\Sigma}_i\, \Delta \bm W_i^1 \,\bm{\Lambda}^{\b^2}_{i-1,\ell}\right\|_{\pmean}\\
    \le &\,  M_\sigma w^{\max\left(0,\,\tfrac{2}{q}-\tfrac{1}{p}\right)} \left\|\Delta \bm{W}_i\right\|_{_\pmean\to_\qmean}\left\|\y_{i-1}\right\|_\pmean\left\|\bm{W}_i \right\|_{_\pmean\to_\qmean} \left\| \bm{\Lambda}^{\bm{b}^2}_{i-1,\ell}\right\|_\pmean \\
    &\, + L_\sigma\left\|\Delta \bm W_i^1 \right\|_{_\pmean\to_\qmean}\left\|\bm{\Lambda}^{\b^2}_{i-1,\ell}\right\|_\pmean \\
        \leq &\, M_\sigma w^{\max\left(0,\,\tfrac{2}{q}-\tfrac{1}{p}\right)} C\left\|\y_{i-1}\right\|_\qmean \left\| \bm{\Lambda}^{\bm{b}^2}_{i-1,\ell}\right\|_\pmean +L_\sigma  \left\| \bm{\Lambda}^{\bm{b}^2}_{i-1,\ell}\right\|_\pmean \\ 
 \le &\, M_\sigma w^{\max\left(0,\,\tfrac{2}{q}-\tfrac{1}{p}\right)} 2^{i-1}L_\sigma^{2i-2}C^{2i-1} +L_\sigma^iC^{i-2}, \, \text{$(\ell<j=i)$}. 
\end{align*}
\noindent\textbf{Case 3: $\max\{j,\ell\}<i<K$.} Similar with \textbf{Case 2}, WLOG, we assume $j\leq \ell$. 
\begin{align*}
\left\|\Xi^{\bm{W}}_{i,j,\ell}\right\|_{\pmean} = &\, \left\|\bm{\Gamma}_{i}\left((\bm{W}_i \, \bm{\Lambda}^{\bm{W}^1}_{i-1,j}) \odot (\bm{W}_i \,  \bm{\Lambda}^{\bm{W}^2}_{i-1,\ell})\right)+ \bm{\Sigma}_i \bm{W}_i \,\Xi^{\bm{W}}_{i-1,j,\ell}\right\|_{\pmean}\\
\le &\,  M_\sigma w^{\max\left(0,\,\tfrac{2}{q}-\tfrac{1}{p}\right)}\left\|\bm{W}_i \,\bm{\Lambda}^{\bm{W}^1}_{i-1,j}\right\|_\qmean\left\|\bm{W}_i \, \bm{\Lambda}^{\bm{W}^2}_{i-1,\ell}\right\|_\qmean \\
&\, +L_\sigma\left\|\bm{W}_i\right\|_{_\pmean\to_\qmean}\left\|\Xi^{\bm{W}}_{i-1,j,\ell}\right\|_\pmean\\
\le &\,  M_\sigma w^{\max\left(0,\,\tfrac{2}{q}-\tfrac{1}{p}\right)}\left\|\bm{W}_i \right\|^2_{_\pmean\to_\qmean} \left\| \bm{\Lambda}^{\bm{W}^1}_{i-1,j}\right\|_\pmean \left\|\bm{\Lambda}^{\bm{W}^2}_{i-1,\ell}\right\|_\pmean \\
&\, +L_\sigma\left\|\bm{W}_i\right\|_{_\pmean\to_\qmean}\left\|\Xi^{\bm{W}}_{i-1,j,\ell}\right\|_\pmean\\
\le &\, M_\sigma w^{\max\left(0,\,\tfrac{2}{q}-\tfrac{1}{p}\right)}2^{2i-4}L_\sigma^{2i-2} C^{2i}+L_\sigma C\left\|\Xi^{\bm{W}}_{i-1,j,\ell}\right\|_\pmean. 
\end{align*}
Combining the bounds from Cases~1--2 and unrolling the above recursion,
we obtain
\[
\left\|\Xi^{\bm W}_{i,j,\ell}\right\|_{\pmean}
\lesssim
w^{\max\left(0,\frac{2}{q}-\frac{1}{p}\right)},
\]
where the hidden constant depends only on $C$, $K$, $L_\sigma$, and $M_\sigma$,
but is independent of the layer widths. Similarly, we have 
\begin{align*}
& \left\|\Xi^{\bm{b}}_{i,j,\ell}\right\|_\pmean\\
=&\, \left\|\bm{\Gamma}_{i}\left((\bm{W}_i \, \bm{\Lambda}^{\b^1}_{i-1,j}) \odot (\bm{W}_i \,  \bm{\Lambda}^{\bm b^2}_{i-1,\ell})\right) + \bm{\Sigma}_i \bm{W}_i \,\Xi^{\bm{b}}_{i-1,j,\ell}\right\|_\pmean\\
\le &\,  M_\sigma w^{\max\left(0,\,\tfrac{2}{q}-\tfrac{1}{p}\right)}\left\|\bm{W}_i \, \bm{\Lambda}^{\b^1}_{i-1,j}\right\|_\qmean \left\|\bm{W}_i \, \bm{\Lambda}^{\b^2}_{i-1,\ell}\right\|_\qmean+ L_\sigma \left\| \bm{W}_i \,\Xi^{\bm{b}}_{i-1,j,\ell}\right\|_\qmean\\
\le &\, M_\sigma w^{\max\left(0,\,\tfrac{2}{q}-\tfrac{1}{p}\right)}\left\|\bm{W}_i\right\|^2_{_\pmean\to_\qmean}\left\|\bm{\Lambda}^{\b^1}_{i-1,j}\right\|_{\pmean}\left\| \bm{\Lambda}^{\b^2}_{i-1,\ell}\right\|_{\pmean}\\
&\,+L_\sigma \left\|\bm{W}_i\right\|_{_\pmean\to_\qmean}\left\|\Xi^{\bm{b}}_{i-1,j,\ell}\right\|_\pmean\\
\le &\, M_\sigma w^{\max\left(0,\,\tfrac{2}{q}-\tfrac{1}{p}\right)} (L_\sigma C)^{2i-2}+L_\sigma C\left\|\Xi^{\bm{b}}_{i-1,j,\ell}\right\|_\pmean\lesssim
w^{\max\left(0,\frac{2}{q}-\frac{1}{p}\right)}.
\end{align*}
Finally, for $\Xi^{\bm W,\bm b}_{i,j,\ell}$, we have
\begin{align*}
\left\|\Xi^{\bm W,\bm b}_{i,j,\ell}\right\|_{\pmean}= &\, \left\|\bm{\Gamma}_{i}\left((\bm{W}_i \, \bm{\Lambda}^{\bm{W}^1}_{i-1,j}) \odot (\bm{W}_i \,  \bm{\Lambda}^{\b^2}_{i-1,\ell})\right)+ \bm{\Sigma}_i \bm{W}_i \,\Xi^{\bm W,\bm b}_{i-1,j,\ell}\right\|_{\pmean}\\
\le &\, M_\sigma w^{\max\left(0,\,\tfrac{2}{q}-\tfrac{1}{p}\right)}\left\|\bm{W}_i \,\bm{\Lambda}^{\bm{W}^1}_{i-1,j}\right\|_\qmean\left\|\bm{W}_i \, \bm{\Lambda}^{\b^2}_{i-1,\ell}\right\|_\qmean \\
&\, +L_\sigma\left\|\bm{W}_i\right\|_{_\pmean\to_\qmean}\left\|\Xi^{\bm W,\bm b}_{i-1,j,\ell}\right\|_\pmean\\
\le&\, M_\sigma w^{\max\left(0,\,\tfrac{2}{q}-\tfrac{1}{p}\right)}\left\|\bm{W}_i \right\|^2_{_\pmean\to_\qmean} \left\| \bm{\Lambda}^{\bm{W}^1}_{i-1,j}\right\|_\pmean \left\|\bm{\Lambda}^{\b^2}_{i-1,\ell}\right\|_\pmean \\
&\,+L_\sigma\left\|\bm{W}_i\right\|_{_\pmean\to_\qmean}\left\|\Xi^{\bm W,\bm b}_{i-1,j,\ell}\right\|_\pmean\\
\le &\, M_\sigma w^{\max\left(0,\,\tfrac{2}{q}-\tfrac{1}{p}\right)}2^{i-2}L_\sigma^{2i-2} C^{2i-1}+L_\sigma C\left\|\Xi^{\bm W,\bm b}_{i-1,j,\ell}\right\|_\pmean\lesssim
w^{\max\left(0,\frac{2}{q}-\frac{1}{p}\right)}.
\end{align*}
\noindent\textbf{Case 4: $\max\{j,\ell\}<i=K$.} WLOG, we assume $j\leq \ell$. With the similar computation as \textbf{Case 3}, we have
\begin{align*}
\left|\Xi^{\bm W}_{K,j,\ell}\right|= &\, \left|\bm \Gamma_{K}\left((\bm{W}_K \, \bm{\Lambda}^{\bm W^1}_{K-1,j}) \cdot (\bm{W}_K \,  \bm{\Lambda}^{\bm W^2}_{K-1,\ell})\right)
+ \bm{\Sigma}_K \bm{W}_K \,\Xi^{\bm W}_{K-1,j,\ell}\right|\\
\le &\,  M_\sigma \left|\bm{W}_K \, \bm{\Lambda}^{\bm W^1}_{K-1,j}\right|\left|\bm{W}_K \, \bm{\Lambda}^{\bm W^2}_{K-1,\ell}\right| +L_\sigma\left\|\bm{W}_K\right\|_{_\pmean\to_\infty}\left\|\Xi^{\bm W}_{K-1,j,\ell}\right\|_\pmean\\
\le &\, M_\sigma \left\|\bm{W}_i \right\|_{_\pmean\to_\infty} \left\| \bm{\Lambda}^{\bm W^1}_{K-1,j}\right\|_\pmean\left\|\bm{W}_i\right\|_{_\pmean\to_\infty} \left\| \bm{\Lambda}^{\bm W^2}_{K-1,\ell}\right\|_\pmean \\
&\,+L_\sigma C\left\|\Xi^{\bm W}_{K-1,j,\ell}\right\|_\pmean
\lesssim w^{\max\left(0,\,\tfrac{2}{q}-\tfrac{1}{p}\right)},
\end{align*}
\begin{equation*}
\begin{aligned}
       \left|\Xi^{\bm b}_{K,j,\ell}\right|
       =&\,\left|\bm \Gamma_K\left((\bm{W}_K \, \bm{\Lambda}^{\bm b^1}_{K-1,j}) \cdot (\bm{W}_K \,  \bm{\Lambda}^{\bm b^2}_{K-1,\ell})\right) + \bm{\Sigma}_K \bm{W}_K \,\Xi^{\bm b}_{K-1,j,\ell}\right|\\
        \le &\, M_\sigma \left|\bm{W}_K \, \bm{\Lambda}^{\bm b^1}_{K-1,j}\right|\left|\bm{W}_K \, \bm{\Lambda}^{\bm b^2}_{K-1,\ell}\right|+ L_\sigma \left\|\bm{W}_K\right\|_{_\pmean\to\infty}\left\|  \,\Xi^{\bm b}_{K-1,j,\ell}\right\|_\pmean\\
        \le &\, M_\sigma \left\|\bm{W}_K\right\|_{_\pmean\to_\infty}\left\|\bm{\Lambda}^{\bm b^1}_{K-1,j}\right\|_{\pmean}\left\|\bm{W}_K\right\|_{_\pmean\to_\infty}\left\| \bm{\Lambda}^{\bm b^2}_{K-1,\ell}\right\|_{\pmean}\\
        &\,+L_\sigma \left\|\bm{W}_K\right\|_{_\pmean\to_\infty}\left\|\Xi^{\bm b}_{K-1,j,\ell}\right\|_\pmean \lesssim  w^{\max\left(0,\,\tfrac{2}{q}-\tfrac{1}{p}\right)}.
\end{aligned}
\end{equation*}
Finally, for $\Xi^{\bm W,\bm b}_{K,j,\ell}$ and $\Xi^{\bm W,\bm b}_{K,\ell, j}$, we have
\begin{align*} \left|\Xi^{\bm W,\bm b}_{K,j,\ell}\right|
= & \, \left|\bm \Gamma_{K}\left((\bm{W}_K \, \bm{\Lambda}^{\bm W^1}_{K-1,j}) \cdot (\bm{W}_K \,  \bm{\Lambda}^{\bm b^2}_{K-1,\ell})\right)
+ \bm{\Sigma}_K \bm{W}_K \,\Xi^{\bm W,\bm b}_{K-1,j,\ell}\right|\\
\le &\, M_\sigma \left|\bm{W}_K \, \bm{\Lambda}^{\bm W^1}_{K-1,j}\right|\left|\bm{W}_K \, \bm{\Lambda}^{\bm b^2}_{K-1,\ell}\right| +L_\sigma\left\|\bm{W}_K\right\|_{_\pmean\to_\infty}\left\|\Xi^{\bm W,\bm b}_{K-1,j,\ell}\right\|_\pmean\\
\le &\, M_\sigma \left\|\bm{W}_i \right\|_{_\pmean\to_\infty} \left\| \bm{\Lambda}^{\bm W^1}_{K-1,j}\right\|_\pmean\left\|\bm{W}_i\right\|_{_\pmean\to_\infty} \left\| \bm{\Lambda}^{\bm b^2}_{K-1,\ell}\right\|_\pmean \\ &\, +L_\sigma C\left\|\Xi^{\bm W,\bm b}_{K-1,j,\ell}\right\|_\pmean
\lesssim w^{\max\left(0,\,\tfrac{2}{q}-\tfrac{1}{p}\right)}.
\end{align*}
Putting everything together yields the result.  
\end{proof}  
\begin{highlightbox}
    \textbf{Message:}
While certain operator-norm geometries allow the $M$-Lipschitz constant to remain well controlled, the behavior of the $L$-smoothness coefficient can vary significantly:
\begin{itemize}
  \item \textbf{\textrm{Muon}:}
  Under the $\|\cdot\|_{(2,\textrm{mean})\to(2,\textrm{mean})}$ geometry, the $L$-smoothness coefficient scales as $O(\sqrt{\text{width}})$.
  \item \textbf{Width-independent smoothness:}
  In contrast, width-independent $L$-smoothness can be achieved under alternative geometries. In particular, the gradient of a neural network admits a width-independent $L$-smoothness coefficient under the $\|\cdot\|_{(1,\textrm{mean})\to\pmean}$ or $\|\cdot\|_{\pmean\to \infty}$ geometries. As a special case, under the $\|\cdot\|_{(1,\textrm{mean})\to\infty}$ geometry associated with \textrm{Adam}/\textrm{SignSGD}, the $L$-smoothness coefficient is width-independent.
\end{itemize}
\end{highlightbox}

% \begin{highlightbox}
% \textbf{Message:} Although the operator norms may play nicely together, allowing the $M$-Lipschitz constant to be well controlled, the $L$-smoothness coefficient can still exhibit diverse behavior:
% \begin{itemize}
%   \item \textbf{\textrm{Muon}:} Under the $\|\cdot\|_{\textrm{RMS} \to \textrm{RMS}}$ geometry, the $L$-smoothness coefficient scales as $O(\sqrt{\text{width}})$.
%      \item \textbf{Width-independent Smoothness is Possible:}  The gradient of a neural network enjoys a width-independent $L$-smoothness coefficient in the $\|\cdot\|_{{1,\textrm{mean}} \to \pmean}$ or $\|\cdot\|_{\pmean \to \ell_\infty}$ geometry. Under the $\|\cdot\|_{{1,\textrm{mean}} \to \ell_\infty}$ geometry of \textrm{Adam}/\textrm{SignSGD}, the neural network enjoys a width-independent $L$-smoothness coefficient.
% \end{itemize}
% \end{highlightbox}

\subsection{\textrm{MOGA} Optimizer} 
% : Fixing the Scaling Problems in $p \to q$ Geometry}
\label{subsection:scaling}

In this subsection, we show that passing from the classical $p\to q$ geometry to the
$(p,\textrm{mean})\to(q,\textrm{mean})$ geometry amounts to a width-aware rescaling of the learning rate.
This follows from the fact that the mean-normalized norm $\|\cdot\|_{(p,\textrm{mean})}$ differs from the standard
$\ell_p$ norm on $\mathbb{R}^d$ only by a multiplicative factor of $d^{1/p}$.
Consequently, the two geometries induce equivalent update directions up to a dimension-dependent scaling,
which can be absorbed into the step size. 
Moreover, we show that the resulting width-aware rescaling coincides exactly with the
$\mu$P scaling rule of \cite{yang2022tensor} in the special cases of \textrm{Adam} and \textrm{SignSGD}.
Precisely, the following equivalence holds.

\begin{fact}\label{fact:mogascaling} 
Let $\bm{D} \in \mathbb{R}^{d_{\mathrm{out}}\times d_{\mathrm{in}}}$ and
$1 \le p,q\le \infty$.
Then the operator norms under the mean-normalized geometry satisfy
% \begin{equation}
% \left\|\bm{D}\right\|_{_{(1,\textrm{mean})}\to_\qmean}=\frac{d_{\mathrm{in}}}{d_{\mathrm{out}}^{1/q}}\left\|\bD\right\|_{1\rightarrow q},
% \quad \text{and} \quad \left\|\bm{D}\right\|_{(p,\textrm{mean}) \to \infty}
% =
% d_{\mathrm{in}}^{1/p}
% \left\|\bD\right\|_{p \to \infty}
% .
% \end{equation}
\[
\left\|\bm{D}\right\|_{{(p,\textrm{mean})}\rightarrow {(q,\textrm{mean})}}= \boxed{\frac{d_{\mathrm{in}}^{1/p}}{d^{1/q}_{\mathrm{out}}}}\cdot \left\|\bm{D}\right\|_{p\rightarrow q}.
\]
Consequently, the steepest-ascent directions induced by these norms coincide
with those in Proposition~\ref{proposition:rowcolnorm} up to a width-dependent
rescaling factor. In particular,
\begin{itemize}[leftmargin=0.15in]
\setlength{\itemsep}{4pt}
\item[\textnormal{(i)}] (\textbf{Column-wise update, $\|\cdot\|_{(1,\textrm{mean})\to \qmean}$}) 
\[
\bm{D}^\star = \boxed{\frac{d_{\mathrm{out}}^{1/q}}{d_{\mathrm{in}}}}\cdot\textrm{colnorm}_q(\bm{G}).\]
\item[\textnormal{(ii)}] 
(\textbf{Row-wise update, $\|\cdot\|_{\pmean\to_\infty}$}) 
\[
\bm{D}^\star = \boxed{\frac{1}{d^{1/p}_{\mathrm{in}}}}\cdot \textrm{rownorm}_p(\bm{G}). 
\]
\end{itemize}
\end{fact}
\begin{proof}{Proof of Fact~\ref{fact:mogascaling}.}
   By Definition \ref{def:op_norm}, we have 
\[
\left\|\bm{D}\right\|_{(p,\textrm{mean})\to (q,\textrm{mean})}
=
\sup_{\x\neq0}
\frac{\left\|\bm{D} \x\right\|_{(q,\textrm{mean})}}{\left\|\x\right\|_{(p,\textrm{mean})}} 
=
\frac{d^{1/p}_{\mathrm{in}}}{d_{\mathrm{out}}^{1/q}}
\sup_{\x\neq0}
\frac{\left\|\bm{D} \x\right\|_q}{\left\|\x\right\|_p}
=
\frac{d^{1/p}_{\mathrm{in}}}{d_{\mathrm{out}}^{1/q}}
\left\|\bm{D}\right\|_{p\rightarrow q}, 
\]

% \[
% \left\|\bm{D}\right\|_{(p,\mathrm{mean})\to \infty}
% =
% \sup_{\x\ne 0}
% \frac{\left\|\bm{D}\x\right\|_{\infty}}{\left\|\x\right\|_{(p,\mathrm{mean})}} =  {d_{\mathrm{in}}}^{1/p}\sup_{\x\neq0}\frac{\left\|\bm{D} \x\right\|_{\infty}}{\left\|\x\right\|_{p}} ={d_{\mathrm{in}}}^{1/p}\left\|\bm{D}\right\|_{p\to \infty},
% \]

which establishes the claimed identities. This completes the proof.
\end{proof}

\begin{comment}
\begin{fact}[$(p,\textrm{mean})\to(q,\textrm{mean})$ Geometry as Width-aware Scaling of $p\to q$ Geometry]
\label{fact:mogascaling} For a matrix $\bD\in\mathbb{R}^{\textrm{fan\_out}\times \textrm{fan\_in}}$ and $1\le p \le \infty$, then we have 
\begin{equation} \left\|\bD\right\|_{_{(1,\textrm{mean})}\to_\pmean}=\frac{\textrm{fan\_in}}{\textrm{fan\_out}^{1/p}}\left\|\bD\right\|_{1\rightarrow p},\quad  \left\|\bD\right\|_{_\pmean\to_\infty}={\textrm{fan\_in}}^{1/p}\left\|\bD\right\|_{p\rightarrow\infty}.
\end{equation}
Thus the steepest descent direction associated with the above two norms is steepest update in Proposition \ref{proposition:rowcolnorm} with an appropriate width-dependent rescaling:
\begin{itemize}
    \item $\arg\max_{\|X\|_{(1,\textrm{mean})\to \pmean}\le 1}\langle X,\bm{G}\rangle
= {\color{red}\frac{\textrm{fan\_out}^{1/p}}{\textrm{fan\_in}}}\,\textrm{rownorm}_p(\bm{G}),$
\item $\arg\max_{\|X\|_{_\pmean\to_\infty}\le 1}\langle X,\bm{G}\rangle
={\color{red}\frac{1}{\textrm{fan\_in}^{1/p}}}\,\textrm{colnorm}_p(\bm{G}).$
\end{itemize}
\end{fact}
\end{comment}

The $\pmean\to\qmean$ geometry induces update directions that coincide with
those of the classical $p\to q$ geometry up to width-dependent rescaling factors.
These factors, highlighted in Fact~\ref{fact:mogascaling} (boxed for emphasis),
are referred to as
MOGA (Matrix Operator Geometry Aware) scaling.
Motivated by this observation, we introduce a simple yet general family of
optimizers based on this scaling, which we term \textbf{MOGA}. Here, we treat all parameters in a block-wise manner.
Gradients and momentum variables inherit the same block structure as ${\bm{\Theta}} = \{\bm{W}_i, \b_i\}_{i=1}^\ell$.

\begin{algorithm}
\caption{\textbf{MOGA}: Matrix-Operator-Geometry-Aware Steepest Descent}
\label{alg:scaled_normalization}
\vskip6pt
\begin{algorithmic}[1]
\Require Learning rates $\{\eta^t\}_{t\ge 0}$, momentum parameters $\beta_1,\beta_2\in(0,1)$,
initial parameters ${\bm{\Theta}}^{0}$
\State $\bm{M}^0 \gets \mathbf{0}$ \Comment{$\bm{M}^t$ has the same block structure as ${\bm{\Theta}}^t$}
\For{$t=1,2,\dots$}
    \State Sample stochastic gradient: $\bm{G} \gets \nabla F({\bm{\Theta}}^{t-1};\xi)$, $\xi\sim\mathbb{P}$
    \State Exponential moving average: $\bm{M}^{t} \gets \beta_1 \bm{M}^{t-1} +(1-\beta_1)\bm{G}$
    \State Lookahead (Nesterov-type) momentum: $\tilde{ \bm{M}}^{t} \gets \beta_2 \bm{M}^{t-1} +(1-\beta_2)\bm{G}$

    \If{Descent under $(1,\mathrm{mean})\to(q,\mathrm{mean})$} \Comment{$q\ge 2$}
        \For{$i=1$ to $\ell$}

            \State $\bm{W}^t_i \gets \bm{W}^{t-1}_i - \eta^t \cdot \boxed{\frac{(d^{i}_{\mathrm{out}})^{1/q}}{d^{i}_{\mathrm{in}}}}
                \cdot \textrm{colnorm}_q(\tilde{\bm{M}}^t_{\bm{W}^{t-1}_i})$
                % \Comment{Weight matrix block $\bm{W}_i\in\mathbb{R}^{d^{i}_{\mathrm{out}}\times d^{i}_{\mathrm{in}}}$}
            % \For{$c=1$ to $d^{i}_{\mathrm{in}}$}
            %     \State $
            %     (\bm{W}^t_i)_{:,c} \gets (\bm{W}^{t-1}_i)_{:,c} 
            %     - \eta^{t}\,
            %     \boxed{\frac{(d^{i}_{\mathrm{out}})^{1/q}}{d^{i}_{\mathrm{in}}}}
            %     \cdot
            %     \frac{\mathrm{sign}(\tilde M^{t}_{W_i,:,c})\odot|\tilde M^{t}_{W_i,:,c}|^{q^{\ast}-1}}
            %     {\|\tilde M^{t}_{W_i,:,c}\|_{q^\ast}^{q^\ast-1}}$
            % \EndFor
            % \Comment{Bias block $b_i\in\mathbb{R}^{d^{(i)}_{\mathrm{out}}}$ (use steepest descent under $(q,\mathrm{mean})$ if desired)}
            \State $\b_i^{t} \gets \b_i^{t-1} - \eta^{t}\,\mathrm{sign}(\tilde{\bm{M}}^{t}_{\b_i})$
            \Comment{or steepest descent under $(q,\mathrm{mean})$ norm}
        \EndFor

    \ElsIf{Descent under $(p,\mathrm{mean})\to\infty$}
        \For{$i=1$ to $\ell$}
            \State 
            % \For{$r=1$ to $d^{(i)}_{\mathrm{out}}$}
                $\bm{W}^t_i \gets \bm{W}^{t-1}_i - \eta^t \cdot \boxed{\frac{1}{(d^{i}_{\mathrm{in}})^{1/p}}}
                \cdot \textrm{rownorm}_p(\tilde{\bm{M}}^t_{\bm{W}^{t-1}_i})$
            
            %     \State $W^{(t)}_{i,r,:} \gets W^{(t-1)}_{i,r,:}
            %     - \eta^{(t)}\,
            %     \textcolor{red}{(d^{(i)}_{\mathrm{in}})^{-1/p}}
            %     \cdot
            %     \frac{\mathrm{sign}(\tilde M^{t}_{W_i,r,:})\odot|\tilde M^{t}_{W_i,r,:}|^{p-1}}
            %     {\|\tilde M^{t}_{W_i,r,:}\|_{p}^{p-1}}$
            % \EndFor
            \State $\b_i^{t} \gets \b_i^{t-1} - \eta^{t}\,\mathrm{sign}(\tilde{\bm{M}}^{t}_{\b_i})$
        \EndFor
    \EndIf
\EndFor
\end{algorithmic}
\end{algorithm}

The MOGA scaling derived from our analysis also provides
a principled explanation for hyperparameter transfer across widths, namely, that appropriately rescaled parameters and learning rates preserve the effective optimization geometry as the network width grows, thereby maintaining width-independent optimization behavior. Consequently, hyperparameters tuned on small proxy networks can be
transferred reliably to much wider models.
This property is particularly important in modern large-scale pretraining, where hyperparameter tuning is computationally expensive and reliable transfer across model sizes is crucial. 

Moreover, MOGA scaling exactly recovers the Maximal Update Parametrization ($\mu$P scaling) \citep{yang2021tensor} in the special case of \textrm{Adam}. Importantly, the underlying principles are conceptually distinct:
$\mu$P is motivated by preserving feature-learning behavior as width grows (e.g., keeping feature-level update magnitudes non-degenerate),
whereas MOGA scaling arises from an optimization-geometry viewpoint, via width-independent Lipschitz and smoothness control under mean-normalized operator norms.
Nevertheless, both lead to the hyperparameter transfer phenomenon observed in practice.  The original $\mu$P work primarily provided a parametrization rule motivated by
feature-learning considerations, without a formal theoretical analysis. Later, the work of \citep{yang2023spectral} proposes a spectral condition requiring the update
direction to have spectral norm $\sqrt{\tfrac{d_{\mathrm{in}}}{d_{\mathrm{out}}}}$, which
recovers the $\mu$P scaling for SGD.
From the perspective of our operator-geometry framework,
this spectral condition is equivalent to requiring that the update direction has a width-independent
$\|\cdot\|_{(2,\textrm{mean})arrow(2,\textrm{mean})}$ norm. We first verify that MOGA scaling for \textrm{Adam}/\textrm{SignSGD}
satisfies the spectral condition of \citep{yang2023spectral}.
Fact~\ref{fact:pmean-monotone} implies that 
$
\|\cdot\|_{(1,\textrm{mean})arrow \infty}
\le
\|\cdot\|_{(2,\textrm{mean})arrow(2,\textrm{mean})}.
$ 
Therefore, any update direction normalized to unit
$\|\cdot\|_{(1,\textrm{mean})arrow\infty}$ norm
automatically has a width-independent
$\|\cdot\|_{(2,\textrm{mean})arrow(2,\textrm{mean})}$ norm, and hence satisfies the spectral condition.

Nevertheless, MOGA scaling also admits regimes where updates remain stable even when the spectral condition is violated. To illustrate this, consider the   $(3,\mathrm{mean})\to\infty$ operator norm
and the rank-one matrix $\bm D\in \mathbb{R}^{d\times d}$ defined by $\bm{D}_{i,1} = d^{-1/3}$ and $\bm{D}_{i,j}=0$ when $j\neq 1$. 
It is straightforward to verify that  $\|\bm D\|_{(3,\mathrm{mean})\to\infty}=1$. 
In contrast, its spectral norm satisfies
$
\|\bm D \|_{(2,\mathrm{mean})\to(2,\mathrm{mean})}
=d^{1/6},
$
which grows with $d$ and therefore violates the spectral condition.
Empirically, as demonstrated in Figure~\ref{fig:lr_transfer_3} with $p=3$,
the optimal learning rate of \textrm{MOGA} with row normalization remains invariant
from GPT-2 Small to GPT-XL, indicating that the optimization dynamics
remain stable across widths.
This example shows that hyperparameter transfer in practice can arise even outside
$\mu$P scaling developed through the spectral condition, highlighting the effectiveness of the proposed
MOGA scaling.

\begin{highlightbox}
\textbf{Message}: 
Under the $(p,\mathrm{mean}) \to (q,\mathrm{mean})$ geometry,
we derive a width-aware \emph{MOGA} scaling rule that applies to a broad
family of optimizers.
In particular, several commonly used methods arise as special cases:
\vspace{2mm}
\begin{center}
% \caption{MOGA scaling under different operator geometries.
% % For a layer with input dimension $d_{\mathrm{in}}$ and output dimension
% % $d_{\mathrm{out}}$, the learning-rate scaling follows from
% % Fact~\ref{fact:mogascaling}.
% }
\setlength{\tabcolsep}{2pt}   % 默认 6pt，缩小列间距

\begin{tabularx}{\linewidth}{l c c X}
\toprule
Optimizer & Geometry & Lr scaling & Remark \\
\midrule

Adam / SignSGD 
& $(1,\mathrm{mean})\to\infty$
& $\displaystyle 1/{d_{\mathrm{in}}}$
& Recovers $\mu$P scaling \citep{yang2022tensor} \\

Column Normalization 
& $(1,\mathrm{mean})\to(q,\mathrm{mean})$
& $\displaystyle d_{\mathrm{out}}^{1/q}/d_{\mathrm{in}}$
& $q\ge2$  \\

Row Normalization 
& $(p,\mathrm{mean})\to\infty$
& $\displaystyle {d_{\mathrm{in}}^{-1/p}}$
& \\

\multirow{2}{*}{Muon}
& \multirow{2}{*}{$(2,\mathrm{mean})\to(2,\mathrm{mean})$}
& $\displaystyle \sqrt{{d_{\mathrm{in}}}/{d_{\mathrm{out}}}}$
& \multirow{2}{=}{Two controversial learning-rate rules} \\

& & $\displaystyle \sqrt{\max\{d_{\mathrm{out}},d_{\mathrm{in}}\}}$ & \\
\bottomrule
\end{tabularx}
\end{center}
% \vspace{2mm}
% \begin{itemize}
% \item The scaling $\sqrt{d_{\mathrm{in}}/d_{\mathrm{out}}}$ 
% follows directly from the width-independent Lipschitz bound established in 
% Theorem~\ref{thm:width-lip}, and corresponds to the width-invariant MOGA scaling rule.

% \item On the other hand, Theorem~\ref{theorem:Lsmooth} shows that under the 
% $(2,\mathrm{mean}) \to (2,\mathrm{mean})$ geometry the $L$-smoothness
% constant grows with width as $\mathcal{O}(\sqrt{w})$, which instead induces a
% learning-rate scaling $\sqrt{w}$. 
% This matches the empirical rule used in \cite{moonlight}.
% \end{itemize}
% \begin{itemize}
%     % \item If we use $\textrm{MAV}\to\ell_\infty$ norm do gradient descent, one need to scale the \textrm{Adam/SignSgd} updates by $\frac{1}{\textrm{fan\_in}}$ in each layer. Interestingly, our derived worst-case scaling coincides with the parameter choice in the $\mu$P parameterization \citep[Table 3]{yang2022tensor}, which scales the update by $\tfrac{1}{\mathrm{fan_in}}$ to achieve the same optimal convergence rate across different network widths, as predicted by an average-case analysis near initialization.
%     \item If we do column-wise $\ell_2$ normalization \citep{glentis2025minimalist}, the learning rate should scale $\frac{\textrm{fan\_out}^{1/2}}{\textrm{fan\_in}}$ in each layer. If we do row-wise $\ell_2$ normalization, the learning rate should scale $\sqrt{\frac{1}{\textrm{fan\_in}}}$ in each layer.
% \end{itemize}
\end{highlightbox}

\begin{remark}[Learning Rate Scaling of Muon] The learning-rate scaling of Muon remains unsettled in the literature,
with different prescriptions arising from distinct theoretical viewpoints. 
(i) 
Several works \citep{yang2023spectral,pethick2025training}
interpret Muon as a steepest-descent method under the
$(2,\mathrm{mean}) \to (2,\mathrm{mean})$ operator geometry.
Under this perspective, the scaling
$\sqrt{d_{\mathrm{in}}/d_{\mathrm{out}}}$
follows directly from the width-independent Lipschitz bound
established in Theorem~\ref{thm:width-lip},
and  the  width-invariant MOGA scaling rule derived here.
% which motivates inserting a factor $\sqrt{d_{\mathrm{in}}/d_{\mathrm{out}}}$ into the normalized update. 
(ii) In contrast, large-scale industrial implementations
\cite{moonlight} adopt the heuristic prescription
$
\eta_{\mathrm{Muon}}
=
\sqrt{\max\{d_{\mathrm{out}},d_{\mathrm{in}}\}}\;
\eta_{\mathrm{AdamW}},
$
where $\eta_{\mathrm{AdamW}}$ denotes the tuned AdamW learning rate.
Tensor-program analyses \citep{yang2022tensor}
suggest that $\eta_{\mathrm{AdamW}} = {{\Theta}}(1/w)$ as the network width
$w$ varies, implying that the above rule corresponds to
$\eta_{\mathrm{Muon}} = {{\Theta}}(1/\sqrt{w})$.
This scaling matches the worst-case $L$-smoothness growth predicted
by Theorem~\ref{theorem:Lsmooth} under the
$(2,\mathrm{mean}) \to (2,\mathrm{mean})$ geometry. (iii) 
We hypothesize that these two scalings correspond to different training regimes.
Near random initialization, the factor $\sqrt{d_{\mathrm{in}}/d_{\mathrm{out}}}$
helps maintain width-independent Lipschitz control of the loss,
whereas the worst-case $L$-smooth behavior becomes more relevant
during later stages of training.
This stage-dependent mismatch in step-size scaling may constitute a potential
limitation of Muon.
In contrast, the MOGA optimizer proposed in this paper derives its scaling
directly from the underlying operator geometry, and therefore does not
exhibit such regime-dependent inconsistencies.
\end{remark}

\section{Generalization to Transformer}
\label{section:adaptation}

We extend the scaling framework developed for feedforward networks in 
Section~\ref{sec:main} to Transformer architectures. In particular, we consider 
a standard Transformer block and specify, for each parameter matrix, the 
normalization and parametrization induced by 
Algorithm~\ref{alg:scaled_normalization}. 

Throughout this section, we adopt the $(p,\mathrm{mean}) \to (q,\mathrm{mean})$ 
operator-norm geometry introduced in Section~\ref{sec:main}. For most weight 
matrices in the Transformer block, the scaling rules follow directly from the 
same geometric principle. Vector parameters and a few special weight matrices  require slightly different treatment, 
which we discuss in detail below.
\begin{comment}
We extend the scaling framework developed for feedforward
networks in Section~\ref{sec:main} to Transformer architectures
In particular, we consider a standard Transformer block and specify,
for each parameter matrix, the normalization and parametrization induced by
Algorithm~\ref{alg:scaled_normalization}. In this section, we use the $(p,\mathrm{mean}) \to (q,\mathrm{mean})$ geometry throughout, except for vector parameters and certain special weight matrices, which will be discussed in detail below.
\end{comment}
\paragraph{Input Word Embedding and Positional Embeddings}
Input word embeddings map each token to a continuous representation,
while positional embeddings encode the position of each token in the sequence.
The token embedding matrix has dimension
$d_{\textrm{model}}\times \textrm{vocabsize}$,
and the positional embedding matrix has dimension
$d_{\textrm{model}}\times \textrm{contextsize}$.
In both cases, the ${\textrm{d}_{\textrm{in}}}$  dimension corresponds to
$\textrm{vocabsize}$ or $\textrm{contextsize}$,
while the ${\textrm{d}_{\textrm{out}}}$ dimension equals $d_{\textrm{model}}$.
Importantly, the ${\textrm{d}_{\textrm{in}}}$ dimensions of these embedding layers
are determined by the vocabulary size or context length and therefore do not
scale with the model width $d_{\textrm{model}}$.
As observed by \cite{yang2022tensor},
this implies that for Adam-type optimizers the magnitude of parameter
updates remains at an $\mathcal{O}(1)$ scale throughout training.

From an operator perspective, both token and positional embeddings act on one-hot basis vectors. Since one-hot vectors lie in the $\ell_1$ unit ball, it is therefore natural to model
the input space using the $\ell_1$ geometry. In our framework, hidden
representations are modeled in the $(q,\textrm{mean})$ geometry with
$q \ge 1$, so the embedding layer can be viewed as a linear operator
from $\ell_1$ to the hidden-feature space, i.e., a
$1 \rightarrow (q,\textrm{mean})$ operator. Under this operator norm,
\[
\|\bm{W}\|_{1\rightarrow(q,\mathrm{mean})}
=
\max_c \|\bm{W}_{:,c}\|_{(q,\mathrm{mean})},
\]
so the geometry directly controls the scale of individual embedding
vectors (the columns of the embedding matrix). This interpretation
is also consistent with the structure of embedding lookups, since
$\bm{W}\x = \bm{W}_{:,c}$ for one-hot inputs.

\begin{itemize}[leftmargin=0.15in]
\setlength{\itemsep}{4pt}
\item[\textnormal{(i)}] \textbf{SignSGD}. 
Our first choice models the embedding layer under the 
$1 \rightarrow \infty$ operator geometry.
Under this geometry, the steepest--descent update reduces to applying 
\textrm{SignSGD} to the embedding parameters
(Proposition~\ref{prop:signdgasoperator}).
\item[\textnormal{(ii)}]   \textbf{Scaled Column Normalization}. 
Our second choice models the embedding layer under the 
$1 \rightarrow (q,\textrm{mean})$ operator geometry.
The corresponding steepest--descent update takes the form $
{\textrm{d}_{\textrm{out}}}^{1/q}\,\textrm{colnorm}_{q}(\cdot)$. The factor ${\textrm{d}_{\textrm{out}}}^{1/q}$ arises because the
$(q,\textrm{mean})$ geometry corresponds to a rescaled
$\ell_q$-norm,
$
\|\x\|_{(q,\textrm{mean})} = n^{-1/q}\|\x\|_{q},
$
for vectors in $\mathbb{R}^n$.
Here $n={\textrm{d}_{\textrm{out}}}$, since each embedding vector lies in a
${\textrm{d}_{\textrm{out}}}$-dimensional output space.
\end{itemize}

% \begin{itemize}
% % \vspace{-0.1in}
%     \item \textbf{SignSGD}. Our first choice is to model the embedding layer under the \(1 \!\rightarrow\! \infty\) operator geometry, which naturally reduces to applying \textsc{SignSGD} (Proposition \ref{prop:signdgasoperator})  on the embedding parameters.
%     \item \textbf{Scaled Column Normalization} Our second choice is to model the embedding layer under the \(1 \!\rightarrow\! (p,\textrm{mean})\) operator geometry, where the corresponding steepest--descent update becomes \(\textrm{fan\_out}^{1/p}\,\textrm{colnorm}_{p}(\cdot)\). The \(\textrm{fan\_out}^{1/p}\) scaling arises because the \((p,\textrm{mean})\) geometry corresponds to a rescaled \(\ell_p\)-norm, defined by \(\|x\|_{(p,\textrm{mean})} = d^{-1/p}\|x\|_{p}\) for vectors in \(\mathbb{R}^d\), where \(d=\textrm{fan\_out}\) for the embedding layer since each embedding vector lives in a \(\textrm{fan\_out}\)-dimensional output space.
% \end{itemize}

\begin{figure}
    \centering
    \includegraphics[width=0.5\linewidth]{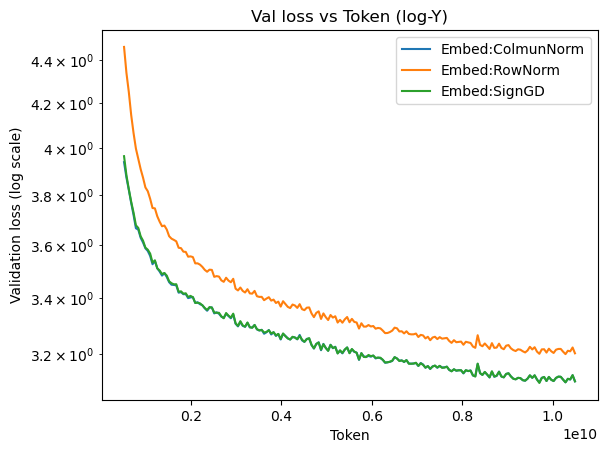}
    \caption{\textbf{Embedding layer geometry.} One-hot inputs place embedding training in the $1 \to \qmean$ geometry, where both column normalization and \textrm{SignGD} are effective.}
    \label{fig:embeeding}
\end{figure}

Although one-hot vectors are unit vectors in $\ell_p$ for any $p\ge1$,
modeling the embedding layer under the $p\rightarrow\infty$ geometry
(corresponding to a $\textrm{rownorm}_p(\cdot)$ update) performs poorly
in our experiments; see Figure~\ref{fig:embeeding}.
This observation further supports the $\ell_1$-based modeling choice
for embedding inputs.

\paragraph{Biases and LayerNorm Parameters}
Bias parameters directly perturb the hidden feature representation,
which lives in the $(q,\textrm{mean})$ geometry. To ensure that these updates remain safely within the feature constraint set, we adopt the $\ell_\infty$ geometry for all vector parameters. Since the $\ell_\infty$ ball is contained in the $(q,\textrm{mean})$ ball
when $q\ge1$, restricting updates to the $\ell_\infty$ geometry ensures that the resulting perturbation remains within the $(q,\textrm{mean})$
constraint set. Under the $\ell_\infty$ geometry, the steepest--descent direction is
given by the sign of the gradient, leading to a \textrm{SignSGD}-type update without additional scaling.

LayerNorm weights $\bm{w}^{\textrm{LN}}$ and biases $\bm{b}^{\textrm{LN}}$ are
vectors in $\mathbb{R}^{d_{\textrm{model}}}$. The LayerNorm weight
corresponds to a diagonal linear operator acting on the hidden feature space, whose operator norm under the
$(q,\textrm{mean}) \rightarrow (q,\textrm{mean})$ geometry reduces to the maximum absolute value of the diagonal entries. Consequently, the steepest--descent direction again collapses to a sign-based update.
In practice, we therefore apply pure \textrm{SignSGD} updates to all vector parameters, including LayerNorm weights, LayerNorm biases, and MLP biases. This scaling is consistent with the parametrization recommended by the
$\mu$P theory \citep{yang2022tensor}.

\paragraph{Self-Attention}
In a Transformer self-attention block, the input
$\bm{X} \in \mathbb{R}^{\text{contextsize} \times d_{\textrm{model}}}$
represents the hidden features of a sequence of tokens, where
$\text{contextsize}$ denotes the sequence length and
$d_{\textrm{model}}$ is the model width.
In the standard multi-head attention with $n_{\textrm{head}}$ heads~\citep{vaswani2017attention},
each head has value dimension $d_v$.
The input is projected into queries, keys, and values via
\[
\bm{Q} = \bm{X} \bm{W}_q, \qquad \bm{K} = \bm{X}\bm{W}_k, \qquad \bm{V} = \bm{X} \bm{W}_v,
\]
where
$
\bm{W}_q, \bm{W}^k, \bm{W}^v
\in
\mathbb{R}^{d_{\textrm{model}} \times (d_v n_{\textrm{head}})}.
$ 
The outputs from different heads are concatenated and projected back
to the model dimension through the output matrix
$
\bm{W}_o
\in
\mathbb{R}^{(d_v n_{\textrm{head}}) \times d_{\textrm{model}}}.
$
Following \cite{vaswani2017attention}, the matrices
$\bm{Q},\bm{K},\bm{V}$ are reshaped into $n_{\textrm{head}}$ heads,
i.e., tensors of shape
$(n_{\textrm{head}}, \text{contextsize}, d_v)$.
Each head then computes
\[
\mathrm{head}_i
=
\mathrm{softmax}\!\left(
\frac{\bm{Q}_i \bm{K}_i^\top}{\sqrt{d_v}}
\right)\bm{V}_i .
\]
The outputs of all heads are concatenated and projected back to the
model dimension through
\[
\mathrm{MultiHead}(
\bm{X}
)
= \mathrm{concat}(\mathrm{head}_1,\ldots,\mathrm{head}_{n_{\textrm{head}}})\, \bm{W}_o,
\qquad
\bm{W}_o \in \mathbb{R}^{d_{\textrm{model}}\times (d_v n_{\textrm{head}})}.
\]
However, in the $\mu$P parametrization \citep{yang2022tensor}, the authors
use a normalization factor $1/{d_v}$ rather than $1/\sqrt{d_v}$, i.e.,
\[
\mathrm{head}_i
=
\mathrm{softmax}\!\left(
\frac{\bm{Q}_i \bm{K}_i^\top}{{d_v}}
\right)\bm{V}_i .
\]
We now provide an alternative explanation from the perspective of
$M$-Lipschitzness and $L$-smoothness.
To ensure that the $\mathrm{softmax}$ operator remains $L$-smooth, the
magnitude of the attention logits must be controlled so that it does not
grow with the network width.
% Lemma~\ref{lem:mean-holder} implies that if the query and key vectors
% are bounded in the $(p,\mathrm{mean})$ norm, then their inner product
% satisfies
% \[
% q^\top k = \mathcal{O}\!\left(d^{\,\max\{1,\;2/p\}}\right).
% \]
Under row normalization
(i.e., the geometry $(p,\textrm{mean})  \to \ell_\infty$),
both query and key vectors are bounded in the $\ell_\infty$ norm.
Consequently,
$
|\q^\top \k|
\le
d_v\,\|\q\|_\infty \|\k\|_\infty
=
\mathcal{O}(d_v),
$
so we apply a normalization factor $1/d$ to keep the attention logits
at a constant scale.
In contrast, under column normalization 
(i.e., the geometry $(1,\textrm{mean}) \to  (q,\textrm{mean})$),
the query and key vectors remain bounded in the $(q,\mathrm{mean})$
norm. By Lemma~\ref{lem:mean-holder}, their inner product scales as
\[
\q^\top \k
=
\mathcal{O}\!\left(d_v^{\,\max\{1,\tfrac{2}{p}\}}\right),
\]
which leads to the normalization factor
$
d_v^{-\max\{1,2/p\}},
$
ensuring that the attention logits remain width-independent.

\begin{highlightbox}
\textbf{Attention scaling.}
\begin{itemize}[leftmargin=*]
\item $(p,\textrm{mean}) \!\to\! \ell_\infty$ (row normalization): scale by $1/d_v$.
\item $(1,\textrm{mean}) \!\to\! (q,\textrm{mean})$ (column normalization): scale by $d_v^{-\max\{1,2/q\}}$.
\item For \textrm{Adam/SignSgd} case, we recover the normalization in $\mu$P parameterization \citep{yang2022tensor}..
\end{itemize}
\end{highlightbox}

% \begin{highlightbox} \textbf{Message}: \begin{itemize}[leftmargin=*] \item For row normalization (i.e., the geometry $(p,\textrm{mean}) \!\to\! \infty$), we use $1/d$ normalization, \textit{i.e.} \(\mathrm{head}_i = \mathrm{softmax}\!\left(\frac{Q_i K_i^\top}{d_v}\right)V_i\). \item For column normalization (i.e., the geometry $(1,\textrm{mean}) \!\to\! (p,\textrm{mean})$), we use $1/d^{\,\max\{\,1,\;2/p\,\}}$ normalization, \textit{i.e.} \(\mathrm{head}_i = \mathrm{softmax}\!\left(\frac{Q_i K_i^\top}{d_v^{\,\max\{\,1,\;2/p\,\}}}\right)V_i\). \item For \textrm{Adam/SignSgd} case, we recover the normalization in $\mu$P parameterization \citep{yang2022tensor}. \end{itemize} \end{highlightbox}

\paragraph{MLP}
A Transformer MLP block contains two weight matrices
$
\bm{W}_1 \in \mathbb{R}^{d_{\textrm{ffn}}\times d_{\textrm{model}}}
$
and
$
\bm{W}_2 \in \mathbb{R}^{d_{\textrm{model}}\times d_{\textrm{ffn}}},
$
where the feedforward dimension is typically
$d_{\textrm{ffn}} = 4 d_{\textrm{model}}$. 
Accordingly, the $\textrm{d}_{\textrm{in}}$ and $\textrm{d}_{\textrm{out}}$t are
$(d_{\textrm{model}},\, d_{\textrm{ffn}})$ for $\bm{W}_1$
and $(d_{\textrm{ffn}},\, d_{\textrm{model}})$ for $\bm{W}_2$.
Following the treatment of other bias parameters, the MLP bias is
updated under the $\ell_\infty$ geometry, resulting in a
\textrm{SignSGD}-type update without additional scaling.

% \paragraph{MLP} In a Transformer MLP block, there are two weight matrices \(W^1 \in \mathbb{R}^{d_{\textrm{ffn}}\times d_{\textrm{model}}}, W^2 \in \mathbb{R}^{d_{\textrm{model}}\times d_{\textrm{ffn}}},\) where \( d_{\textrm{ffn}} \) is typically set to \(4d_{\textrm{model}}\).
% For \(W^1\), we set {$\textrm{fan\_in}$} to \(d_{\textrm{model}}\) and {$\textrm{fan\_out}$} to \(d_{\textrm{ffn}}\).
% For \(W^2\), we set {$\textrm{fan\_in}$} to \(d_{\textrm{ffn}}\) and {$\textrm{fan\_out}$} to \(d_{\textrm{model}}\).  In line with the treatment of other bias parameters, we update the MLP bias in the $\ell_\infty$ geometry. This results in a \textrm{signSGD}-type update with no scaling applied to the MLP bias.

%We treat the MLP bias in the same way as the LayerNorm bias. Although the bias could, in principle, be updated under the $\qmean$ geometry, we adopt the $\ell_\infty$--norm ball for simplicity. Since the $\ell_\infty$ ball is strictly smaller (and therefore more conservative) than the $\qmean$ ball, the update collapses to taking the sign of the gradient, i.e., \textrm{signSGD} without any scaling. As a result, 

\paragraph{Word Unembedding}
Following the $\mu$P scaling rule~\citep{yang2022tensor}, the word
\emph{unembedding} matrix is tied to the word \emph{embedding} matrix.
Consequently, the unembedding parameters follow the same scaling as
the embedding parameters, i.e., we do not scale down the learning rate
for the unembedding layer.

\section{Experiments}
We evaluate the proposed \textrm{MOGA} optimizer on large-scale language model
pre-training tasks and compare it against strong baselines including
\textrm{AdamW} and \textrm{Muon}.
Our experiments focus on three aspects:
(i) learning-rate transfer across model widths,
(ii) training efficiency under standard token budgets, and
(iii) performance under large token budgets.

We conduct experiments on both \textbf{GPT-2} and \textbf{LLaMA} architectures.
For GPT-2 models, we use the standard tokenizer with a vocabulary size 50,257
and context length 1,024, and train on the \emph{OpenWebText} dataset~\citep{Gokaslan2019OpenWeb}.
For LLaMA models, we use vocabulary size 32,000 and context length 256,
and pretrain from scratch on the \emph{C4} dataset~\citep{raffel2020exploring}.  All experiments are conducted on four NVIDIA H100 GPUs. 

We evaluate \textrm{MOGA} instantiated with row normalization,
and compare against \textrm{AdamW} and \textrm{Muon}.
The implementation code is publicly available at
\url{https://github.com/ruihanxx/rownorm}.

\subsection{Learning Rate Transfer}
\label{sec:lr_transfer} 
 In this subsection, we study whether the {optimal learning rate} remains invariant across model widths. Details are as follows:
\begin{itemize}
    \item \textit{Models. }We experiment on GPT-2 models ranging from Small to XL, with hidden dimensions from 768 to 1600 and parameter counts from 124M to 1.5B.
To evaluate width invariance, we instantiate \textrm{MOGA} with row normalization
using norm parameters $p\in\{1.5,2,3\}$.

\item \textit{LR schedulers and training configuration.} Models are pretrained from scratch with a 3B-token budget.
Due to a limited token budget, we adopt the LR schedule consisting of a linear warm-up over
10\% of the training steps followed by cosine decay.
The learning rate increases linearly from zero to the peak value $ \textrm{lr\_max}$, after which it follows a cosine annealing schedule down to the terminal learning rate $\textrm{lr\_min} = \textrm{lr\_max} /10$.  

\item \textit{Training configurations. } The total batch size corresponds to $1024\times512=524{,}288$ tokens per step.
We sweep the peak learning rate $\textrm{lr\_max}$ over
$\{1\times10^{-3},2\times10^{-3},4\times10^{-3},1\times10^{-2},2\times10^{-2},4\times10^{-2}\}$.
\end{itemize}
\begin{figure}[t]
\centering
\begin{minipage}[t]{0.32\linewidth}
  \centering
  \includegraphics[width=\linewidth]{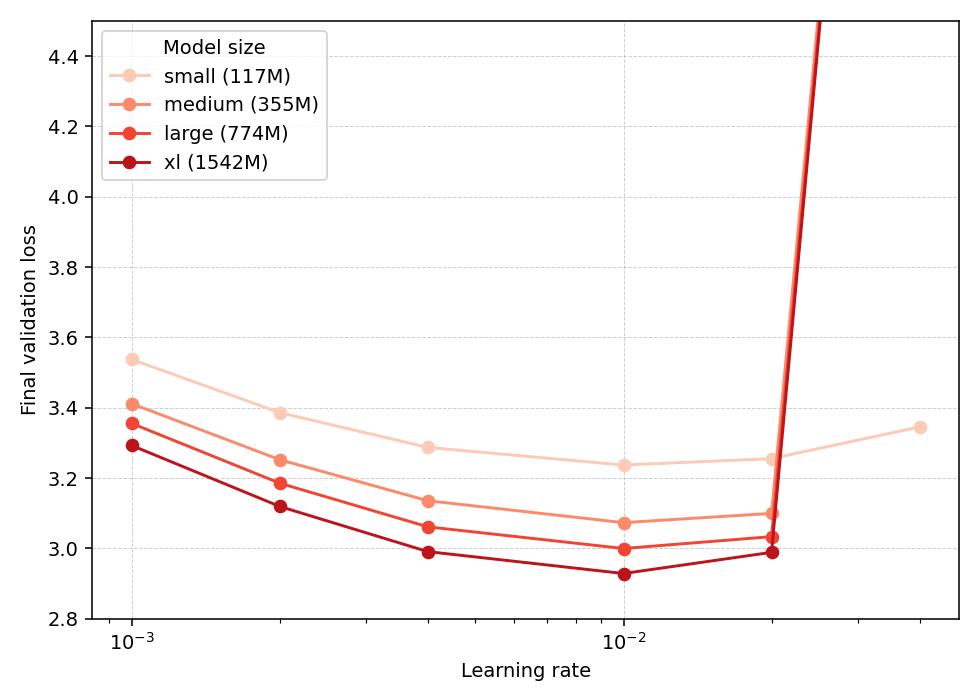}
  \subcaption{\textrm{MOGA} (p=1.5)}
\end{minipage}\hfill
\begin{minipage}[t]{0.32\linewidth}
  \centering
  \includegraphics[width=\linewidth]{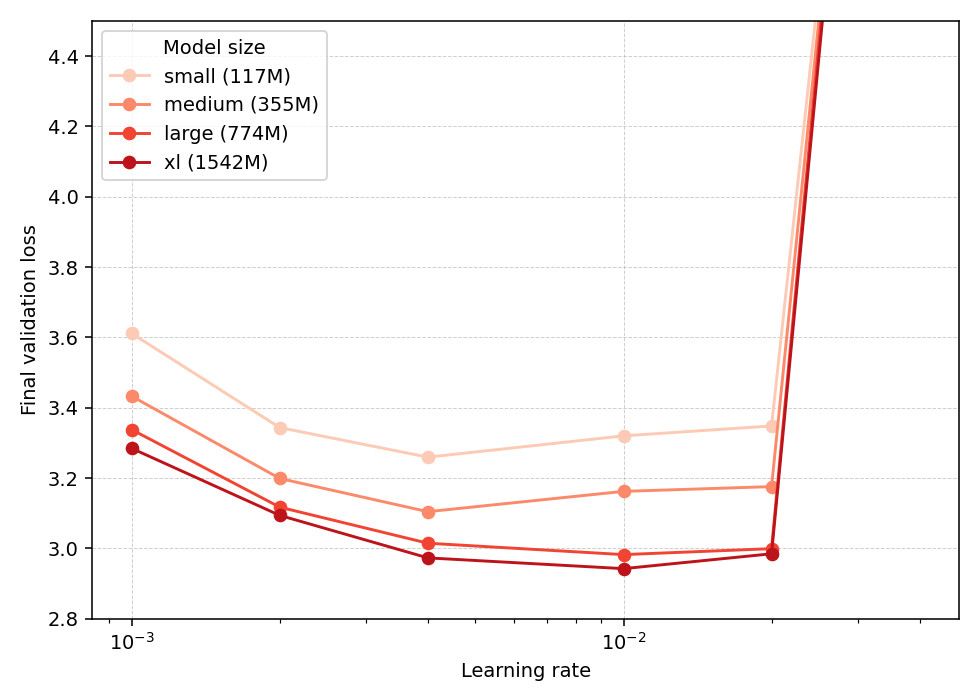}
  \subcaption{\textrm{MOGA} (p=2)}
\end{minipage}\hfill
\begin{minipage}[t]{0.32\linewidth}
  \centering
  \includegraphics[width=\linewidth]{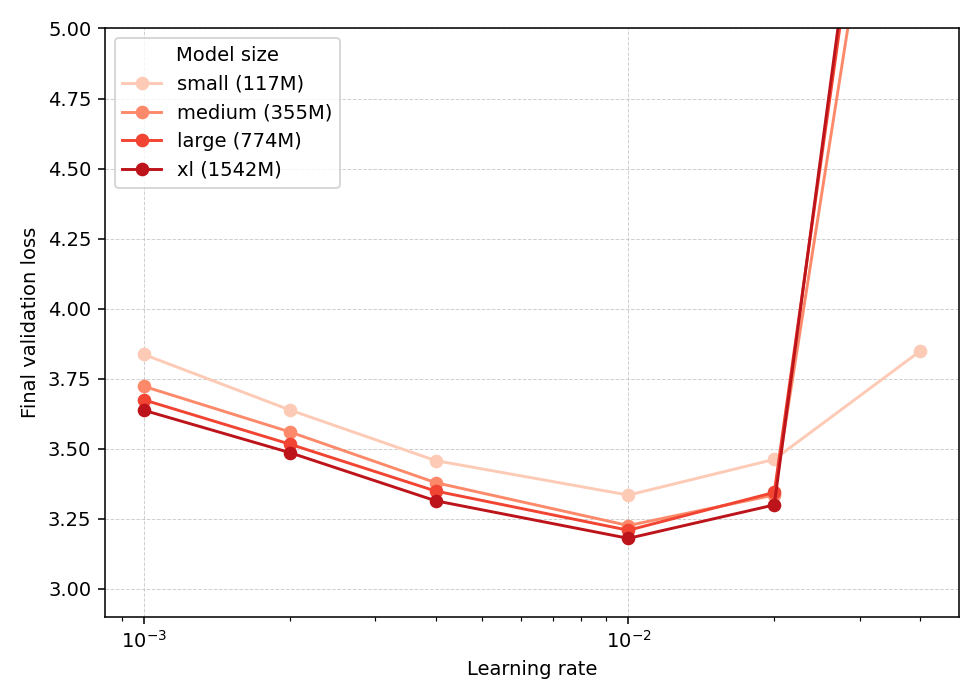}
  \subcaption{\textrm{MOGA} (p=3)}
\end{minipage}

\caption{\textbf{Learning Rate Transfer.} Performance of \textrm{MOGA} from GPT-2 Small to GPT-XL. The optimal learning rate of \textrm{MOGA} is invariant to width.}
\label{fig:lr_transfer_3}
\end{figure}
% We extend the $\mu$P scaling rule to a broader family of optimizers, enabling consistent learning-rate transfer across model widths. Models with radically different parameter counts converge fastest when using nearly the same $\textrm{lr\_max}$. This supports the idea that, when the parameterization follows a width-consistent scaling rule, the optimization geometry of the loss landscape remains stable as width grows. As a consequence, learning rates tuned on a small model can be directly transferred to large models {without additional search}. This property significantly reduces the cost of hyperparameter tuning and enables efficient scaling experiments.

% As shown in Figure~\ref{fig:lr_transfer_3}, the optimal learning rate of \textrm{MOGA}, with row normalization and multiple different norm parameters $p$, remains invariant
% across expanding model size, which supports our theory of width-independent smoothness in our optimization problem.

As shown in Figure~\ref{fig:lr_transfer_3}, models with substantially different
parameter counts achieve their best training performance at nearly the same
peak learning rate $\textrm{lr\_max}$.
This indicates that the optimal learning rate is effectively invariant to
model width under the proposed optimizer.
As a consequence, learning rates tuned on a small model can be directly
transferred to larger models without additional hyperparameter search,
significantly reducing the cost of scaling experiments.

This invariance holds consistently across multiple norm choices $p$ in the
row-normalization family, supporting our theoretical prediction of
width-independent smoothness under the mean-normalized operator geometry.
Notably, the case $p=3$ does not satisfy the spectral assumptions typically
required by $\mu$P scaling theory, yet still exhibits reliable learning-rate
transfer in practice.
This provides empirical evidence that our framework applies to a broader
class of optimizers beyond those covered by existing $\mu$P analyses.

\subsection{Standard Token Budget Experiment}
\label{sec:exp_stand}
For standard token budget pretraining, we experiment on \textbf{GPT-2 small} model and \textbf{LLaMA-130M} model, using approximately $\sim 1\times$ the Chinchilla-optimal number of training tokens~\cite{hoffmann2022training}. Details are as follows:
\begin{itemize}
    \item \textit{LR schedulers.} We adopt the same learning-rate schedule as in
Section~\ref{sec:lr_transfer}, namely a linear warm-up over 10\% of the total training steps
followed by cosine decay.
    \item \textit{Training configurations. } 
    For GPT-2, we use a total batch size of 1024 with a maximum sequence length 512, corresponding to $1024\times512=524{,}288$ tokens per step, which is consistent with the widely adopted configuration \citep{radford2019language} for GPT-2 small training. For LLaMA-130M, we follow \citep{wang2025conda} and use a total batch size of 512 with a maximum sequence length 256, yielding $512\times256=131{,}072$ tokens per step.
    \item \textit{Hyperparameters.}
    For the GPT-2 models, \textrm{AdamW} uses $\beta_{1}=0.9$, $\beta_{2}=0.95$, and weight decay $\lambda=0.1$, while \textrm{Muon} follows the configuration of~\citep{liu2025mars} with $\beta_1=0.95, \beta_2 = 0$ and weight decay $\lambda=0.1$.  For the LLaMA models, \textrm{AdamW} uses $\beta_{1}=0.9$, $\beta_{2}=0.999$, and weight decay $\lambda=0$, and \textrm{Muon} uses $\beta_1=0.95, \beta_2=0$ with weight decay $\lambda=0$.  For each optimizer, the peak learning rate $\textrm{lr\_max}$ is tuned via a logarithmic sweep during the first training epoch. Specifically, for GPT-2 we search over  $\textrm{lr\_max}\in[8\times10^{-5},\,8\times10^{-4}]$ for \textrm{AdamW}, and $\textrm{lr\_max}\in[4\times10^{-4},\,4\times10^{-3}]$ for \textrm{Muon}, and $\textrm{lr\_max}\in[4\times10^{-3},\,4\times10^{-2}]$ for \textrm{MOGA} instantiated with row normalization using $p=2$. For LLaMA we search over $\textrm{lr\_max}\in[4\times10^{-5},\,4\times10^{-4}]$ for \textrm{AdamW}, $\textrm{lr\_max}\in[4\times10^{-4},\,4\times10^{-3}]$ for \textrm{Muon}, and $\textrm{lr\_max}\in[4\times10^{-3},\,4\times10^{-2}]$ for \textrm{MOGA} instantiated with row normalization using $p=2$.  All models are subsequently trained using the optimizer-specific $\textrm{lr\_max}$ that achieves the best validation performance.
    % \item \textbf{Baselines. } 
    %     We evaluate our method against \textrm{AdamW} and \textrm{Muon}, which serve as strong baseline optimizers in all experiments. The details of configurations are as follows.
    %     \begin{itemize}
    %         \item \textbf{GPT-2 model.} For \textrm{AdamW}, we use the standard configuration with $\beta_{1}=0.9$, $\beta_{2}=0.95$, and weight decay $\lambda=0.1$. For \textrm{Muon}, we adopt settings from ~\citep{liu2025mars}, with momentum parameter $\beta=0.95$ and weight decay $\lambda=0.1$. For each optimizer, we tune the peak learning rate $\textrm{lr\_max}$ via a logarithmic sweep during the first training epoch.Specifically, we search over $\textrm{lr\_max}\in[8\times10^{-5},\,8\times10^{-4}]$ for \textrm{AdamW} and $\textrm{lr\_max}\in[4\times10^{-4},\,4\times10^{-3}]$ for \textrm{Muon}.
    %         \item \textbf{LLaMA model.} For \textrm{AdamW}, we use $\beta_{1}=0.9$, $\beta_{2}=0.99$, and weight decay $\lambda=0$. For \textrm{Muon}, we use momentum parameter $\beta=0.95$ and weight decay $\lambda=0$. For each optimizer, we tune the peak learning rate $\textrm{lr\_max}$ via a logarithmic sweep during the first training epoch.Specifically, we search over $\textrm{lr\_max}\in[4\times10^{-5},\,4\times10^{-4}]$ for \textrm{AdamW},  $\textrm{lr\_max}\in[4\times10^{-4},\,4\times10^{-3}]$ for \textrm{Muon}, and $\textrm{lr\_max}\in[4\times10^{-2},\,4\times10^{-3}]$ for \textrm{MOGA}
    %     \end{itemize}
        
    %     All baseline models are subsequently trained using the optimizer-specific $\textrm{lr\_max}$ that achieves the best validation performance.
\end{itemize}
\begin{figure}[t]
\centering
\begin{minipage}[t]{0.5\linewidth}
  \centering
  \includegraphics[width=\linewidth]{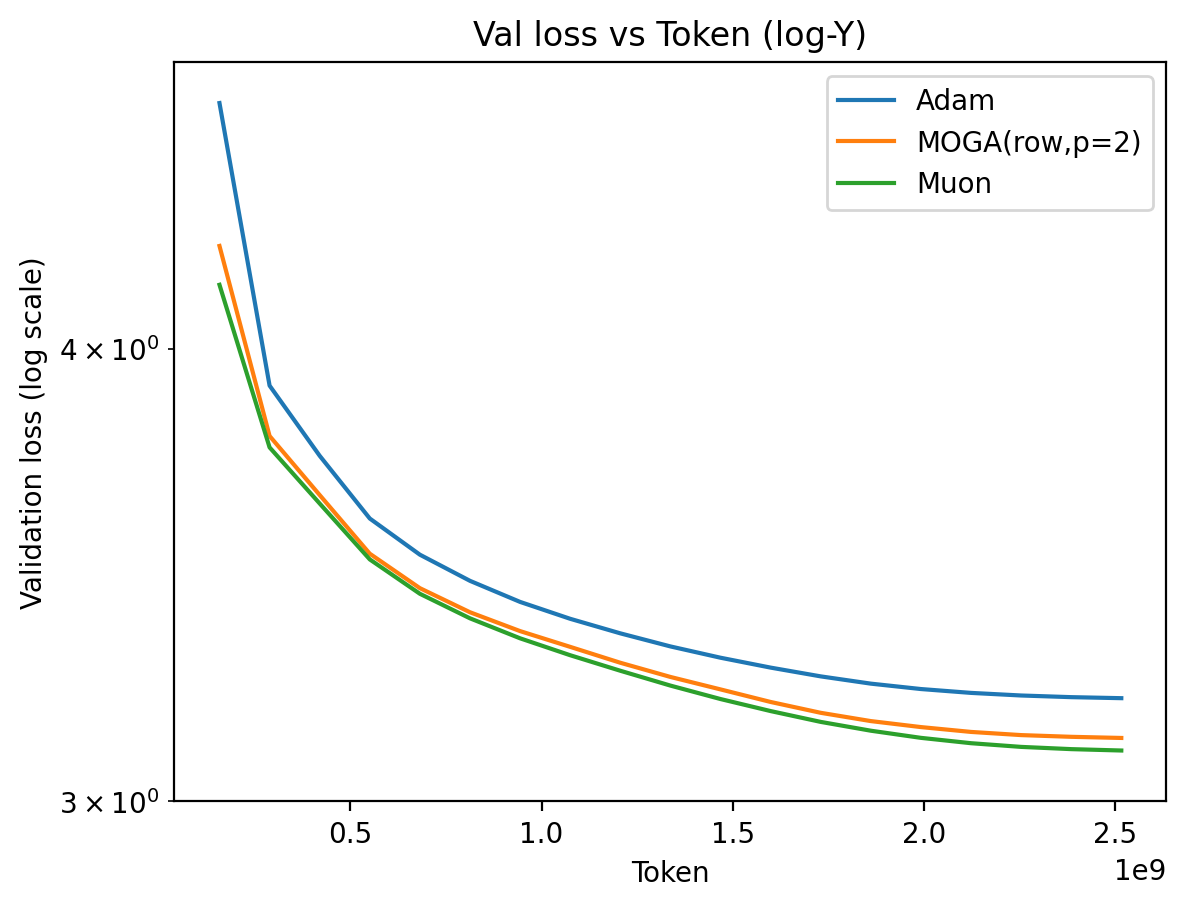}
  \subcaption{LLaMA-130M}
\end{minipage}\hfill
\begin{minipage}[t]{0.5\linewidth}
  \centering
  \includegraphics[width=\linewidth]{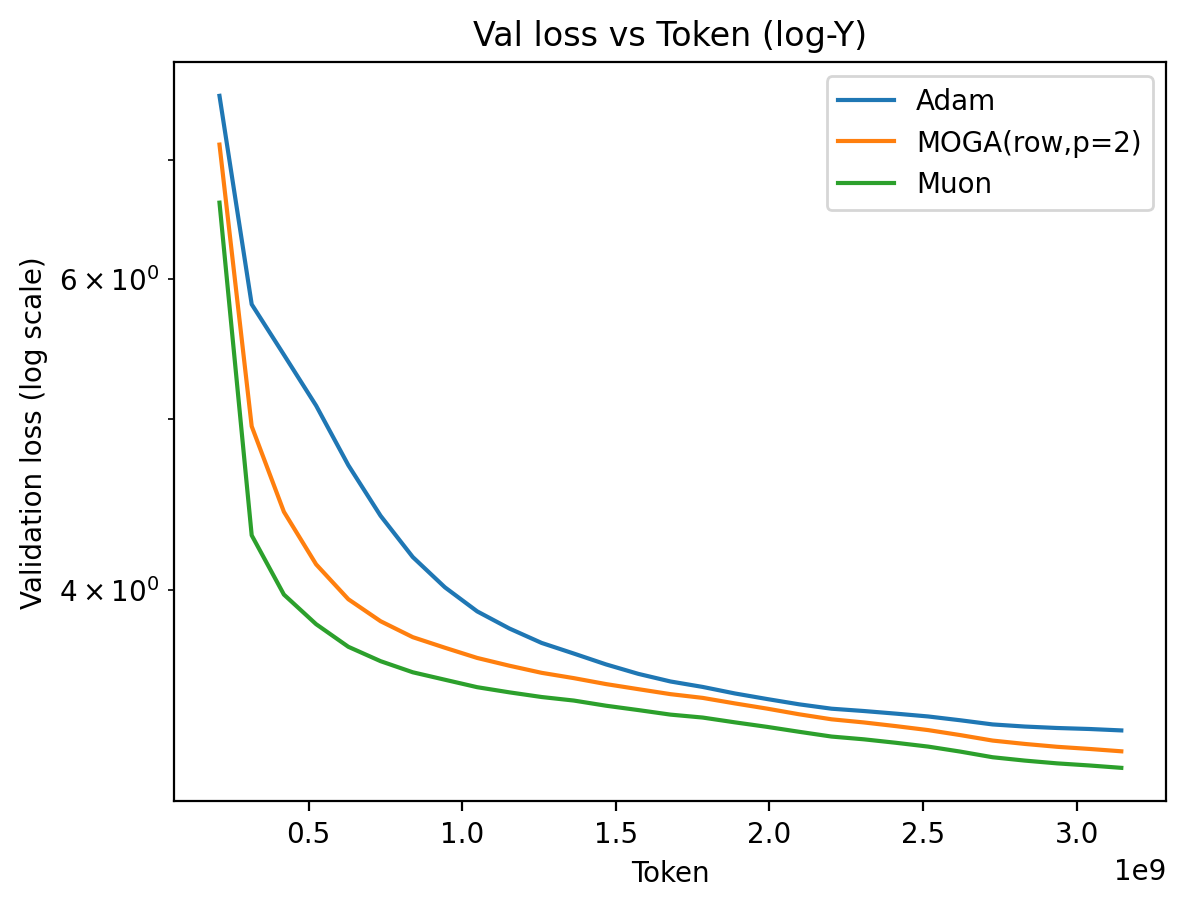}
  \subcaption{GPT-2 small}
\end{minipage}\hfill
\caption{\textbf{Standard Token Budget Experiment.} Performance of \textrm{MOGA} on GPT-2 small model and LLaMA-130M model compared to \textrm{Muon} and \textrm{AdamW} with $1\times$chinchilla optimal token budget.}
\label{fig:smalltoken}
\end{figure}

As shown in Figure~\ref{fig:smalltoken}, on the LLaMA-130M model,
\textrm{MOGA} achieves performance comparable to \textrm{Muon},
with both optimizers converging substantially faster than \textrm{AdamW}.
On the GPT-2 small model, the convergence behavior of \textrm{MOGA}
lies between that of \textrm{Muon} and \textrm{AdamW},
demonstrating competitive optimization efficiency across architectures.
% \Jiajin{We are currently waiting for the final results.
% After they are available, please add a concise take-home message here
% summarizing the main empirical findings.}

\subsection{Large Token Budget Experiment}
We evaluate optimization performance under a large-token regime,
using approximately $\sim 8\times$ the Chinchilla-optimal number of training tokens.
Experiments are conducted on \textbf{GPT-2 Small} and \textbf{LLaMA-130M} models. Details are follows:
\begin{itemize}
    \item \textit{LR schedulers.} We follow the same schedule as in Section~\ref{sec:lr_transfer}, but with different warm-up ratios.
For both GPT-2 Small and LLaMA-130M, we use a linear warm-up over 2.5\% of the total steps
followed by cosine decay.
\item   \textit{Training configurations. }  We use the same training setup as in Section~\ref{sec:exp_stand}.
\item \textit{Hyperparameters.}The hyperparameter settings are the same as those in Section~\ref{sec:exp_stand}, except for the learning-rate search ranges specified below. For GPT-2 we search over  $\textrm{lr\_max}\in[8\times10^{-5},\,8\times10^{-4}]$ for \textrm{AdamW} and $\textrm{lr\_max}\in[4\times10^{-4},\,4\times10^{-3}]$ for \textrm{Muon}. For LLaMA we search over $\textrm{lr\_max}\in[4\times10^{-5},\,4\times10^{-4}]$ for \textrm{AdamW}, $\textrm{lr\_max}\in[4\times10^{-4},\,4\times10^{-3}]$ for \textrm{Muon}, and $\textrm{lr\_max}\in[4\times10^{-3},\,4\times10^{-2}]$ for \textrm{MOGA} instantiated with row normalization using $p=2$.
\end{itemize}

\begin{figure}[t]
\centering
\begin{minipage}[t]{0.5\linewidth}
  \centering
  \includegraphics[width=\linewidth]{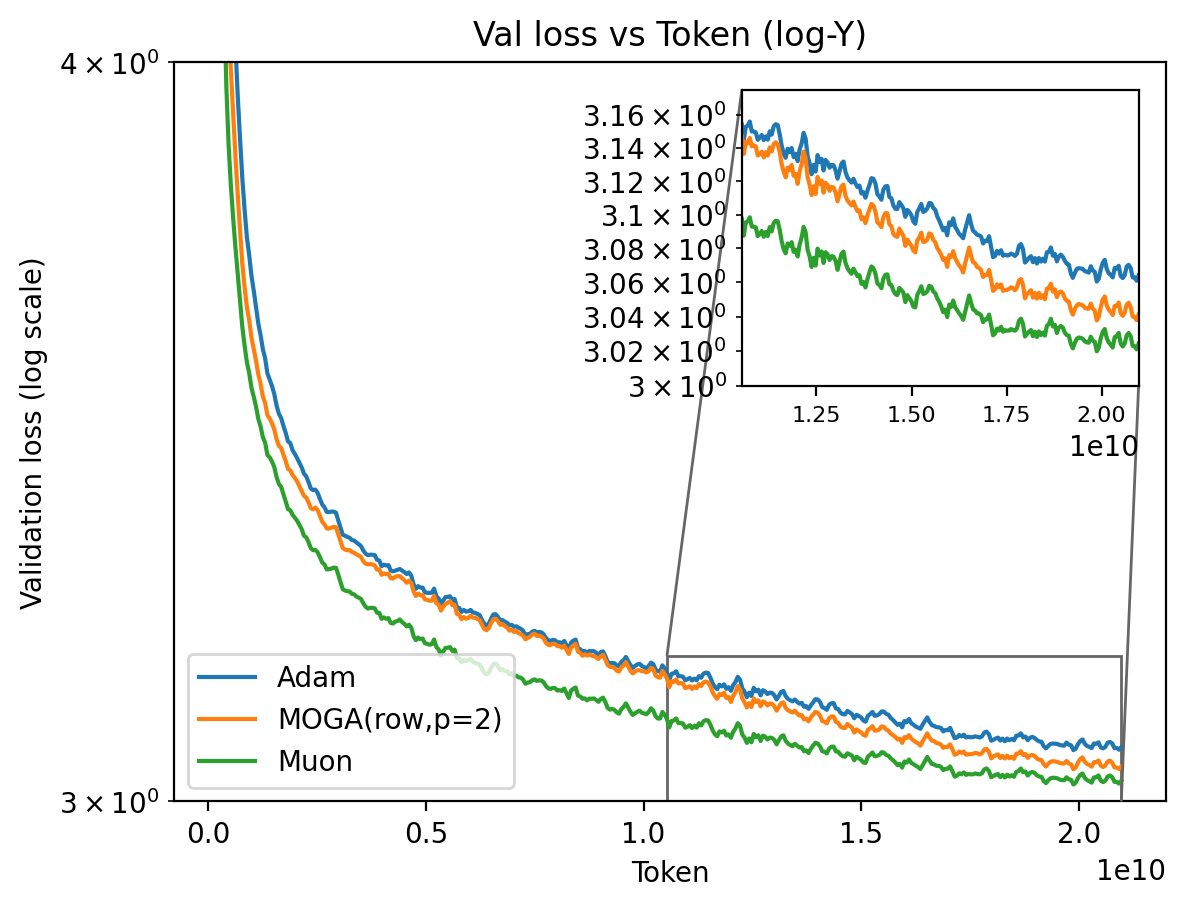}
  \subcaption{GPT-2 Small}
\end{minipage}\hfill
\begin{minipage}[t]{0.5\linewidth}
  \centering
  \includegraphics[width=\linewidth]{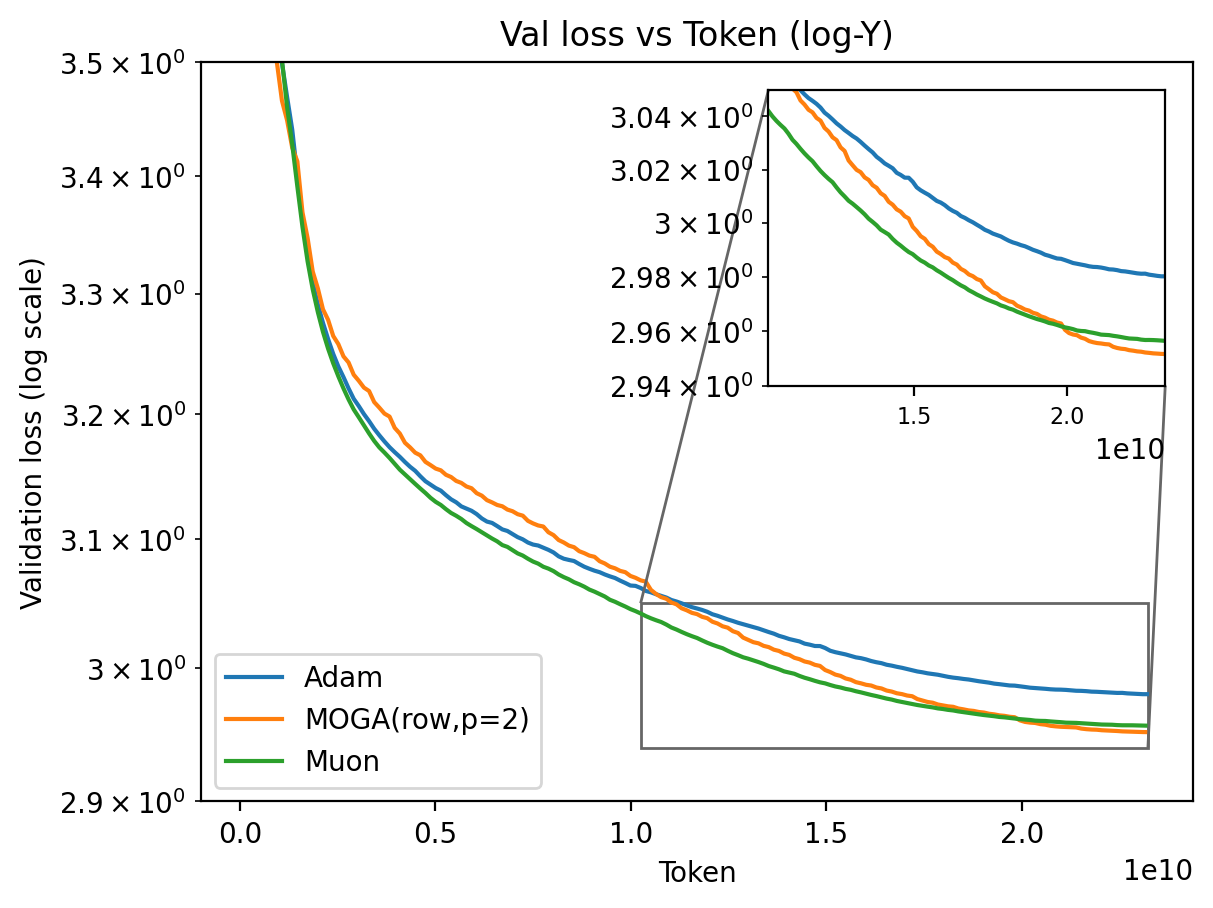}
  \subcaption{LLaMA-130M}
\end{minipage}\hfill
\caption{\textbf{Large Token Budget Experiment.} Performance of \textrm{MOGA} on GPT-2 small model and LLaMA-130M model compared to \textrm{Muon} and \textrm{AdamW}. \textrm{MOGA} exhibits a faster convergence speed in the low-loss stage.}
\label{fig:largetoken}
\end{figure}

% As shown in Figure~\ref{fig:largetoken}, \textrm{MOGA} optimizer exhibits a clear speed advantage in the later stages of training, particularly during the low-loss phase. In the GPT-2 Small experiment, \textrm{MOGA} performs close to the \textrm{Muon} optimizer and outperforms \text{AdamW}. In the LLaMA-130 experiments, \textrm{MOGA} achieves the best overall performance. This indicates that \textrm{MOGA} has an advantage in longer training runs and in the low-loss regime, which is highly valuable in practice.
As shown in Figure~\ref{fig:largetoken}, the \textrm{MOGA} optimizer exhibits a clear speed advantage in the later stages of training, particularly in the low-loss regime. 
In the GPT-2 Small experiment, \textrm{MOGA} performs comparably to \textrm{Muon} and consistently outperforms \textrm{AdamW}. 
Notably, the loss curve of \textrm{MOGA} shows a steeper downward trend toward the end of training, suggesting that with a longer training horizon it would likely surpass \textrm{Muon}. 
In the LLaMA-130 experiments, \textrm{MOGA} already achieves the best overall performance among all optimizers. 
These results indicate that \textrm{MOGA} is particularly advantageous for longer training runs and in the low-loss regime, which is highly valuable in practical large-scale training.

\section{Discussion}

\subsection{Related Work} 

\paragraph{Optimizers for Modern Language Models.}
Modern large-scale neural networks and language models commonly rely on
variants of adaptive gradient methods, beginning with AdaGrad
\citep{duchi2011adaptive,ward2020adagrad}. Most practical optimizers employ diagonal preconditioning,
as in Adam
\citep{kingma2014adam,reddi2019convergence,loshchilov2017decoupled},
or block-wise approximations such as \textrm{Adafactor}
\citep{shazeer2018adafactor},
\textrm{LAMB} \citep{you2019large},
and \textrm{Adam-Mini} \citep{zhang2024adam}.
More expressive preconditioning strategies have also been explored.
For example, \textrm{Shampoo}~\citep{gupta2018shampoo}
and \textrm{SOAP}~\citep{vyas2024soap}
use structured tensor-based preconditioners to approximate curvature
information.
Another line of work views second-order methods as the gold standard
for neural network optimization.
In particular, natural-gradient approaches approximate the Fisher
information matrix
\citep{martens2015optimizing,ren2021tensor,pooladzandi2024curvature},
while recent works such as
\citep{lin2024can,vyas2024soap}
develop scalable approximations inspired by \textrm{SOAP}-style
preconditioning.
More recently,
\citep{xie2025structured,an2025asgo}
proposed \textrm{ASGO} (One-Sided \textrm{Shampoo}),
which applies a one-sided variant of the \textrm{Shampoo} preconditioner.

% Modern large-scale neural networks and language models commonly rely on variants and approximations of the AdaGrad algorithm \citep{duchi2011adaptive,ward2020adagrad}. These methods typically employ diagonal preconditioning, as exemplified by Adam \citep{kingma2014adam,reddi2019convergence,loshchilov2017decoupled}, or block-wise approximations such as \textrm{Adafactor} \citep{shazeer2018adafactor}, \textrm{LAMB} \citep{you2019large}, and \textrm{Adam-Mini} \citep{zhang2024adam}. More sophisticated approaches use structured tensor-based preconditioning, including \textrm{Shampoo} \citep{gupta2018shampoo} and \textrm{SOAP} \citep{vyas2024soap}.  Another line of work treats the Gauss-Newton method as the gold standard for optimization. For instance, \citep{martens2015optimizing,ren2021tensor,pooladzandi2024curvature} approximate the Fisher information matrix to implement natural gradient methods, while \citep{lin2024can,vyas2024soap} employ \textrm{SOAP}-like methodologies for similar approximations. More recently, \citep{xie2025structured,an2025asgo} introduced \textrm{ASGO} (One-Sided \textrm{Shampoo}), which only applies a one-sided preconditioner within the \textrm{Shampoo}. 

%\citep{zhao2024galore}  was proposed to exploit the low-rank structure of gradients for memory efficiency.

\paragraph{Gradient Normalization.}
Starting from \textrm{SignSGD}~\citep{bernstein2018signsgd}, a line of work
has connected gradient normalization with steepest descent under
alternative norms.
For example, \textrm{Muon}~\citep{muonblog} formulates neural network
optimization as steepest descent under the spectral norm, which leads
to whitening the gradient matrix.
Similarly, \cite{glentis2025minimalist} study steepest descent under the
$\ell_1 \to \ell_2$ operator norm, which results in column normalization
of the gradient.
Several recent works employ row-wise or doubly normalized updates.
In particular,
\cite{ma2024swan} combine row normalization with the matrix sign
function to normalize the gradient, while
\cite{scetbon2025gradient} use the Sinkhorn algorithm to normalize
the gradient simultaneously across rows and columns.
In contrast, our framework considers the
$(2,\textrm{mean}) \to \ell_\infty$ geometry, which requires an
additional scaling factor of $1/\sqrt{d_{\textrm{in}}}$ in the update.
Normalization-based methods have recently attracted attention because
they exhibit favorable theoretical properties in challenging regimes.
For instance, they can achieve improved guarantees under heavy-tailed
noise~\citep{wang2025muon,fang2026normalizationpreferredworstcasecomplexity}
and for ill-conditioned optimization problems~\citep{ma2026preconditioning}.
Their convergence properties under various $L$-smoothness assumptions
have also been studied in
\citep{li2025note,pethick2025training,shen2025convergence,glentis2025minimalist,
riabinin2025gluon,kovalev2025understanding,crawshaw2025exploration,
veprikov2025preconditioned}. Our work takes a different perspective.
Rather than analyzing specific normalization schemes, we study when a
neural network objective satisfies an $L$-smoothness condition and use
this criterion to determine which normalization geometry is appropriate.
The closest related work is \cite{davis2025spectral}, which studies when
spectral descent outperforms Euclidean gradient descent by
characterizing the nuclear rank of the gradient and estimating this
quantity in a random feature model.
Another related line of work is
\cite{su2025isotropiccurvaturemodelunderstanding}, which proposes an
isotropic curvature model to analyze spectral normalization and derives
optimal updates using second-order information.
It would be interesting to combine such curvature models with
alternative geometries to design new normalization-based algorithms;
we leave this direction for future work.

\paragraph{Hyperparameter Transfer.}
Hyperparameter transfer, particularly learning-rate transfer, aims to
reuse a learning rate (or schedule) tuned on a small network when
training a larger one, thereby avoiding per-run hyperparameter search.
The goal is to identify simple scaling rules that preserve stability
and performance under changes in width, depth, batch size, and optimizer
details. A prominent line of work is based on Tensor Programs.
\cite{yang2021tensor,yang2022tensor} analyze the infinite-width limit of
forward and backward dynamics and derive parameterizations that keep
activation and gradient scales consistent across widths.
This framework yields concrete rules for transferring learning rates
from small networks to larger ones with minimal retuning.
However, the analysis is anchored in infinite-width initialization
scaling and early training dynamics, and therefore does not fully
capture later training effects such as finite-width corrections,
evolving feature correlations, or distribution shifts. A complementary approach is proposed by \cite{ma2026mup}, who upscale a
trained small network to a wider one by deterministically duplicating
neurons (equivalently, replicating parameter tensors with structured
rescaling) so that the widened model initially computes exactly the
same function. With a corresponding rescaling of optimizer
hyperparameters, the widened model follows an equivalent training
trajectory to the small model.
% A subsequent $\mu$P-scaled symmetry
% breaking perturbation can then be introduced to exploit the additional
% width while preserving stable hyperparameter transfer.
Both approaches justify learning-rate transfer only for specific
structured regimes (e.g., random initialization or deterministically
upscaled networks). In contrast, our worst-case analysis establishes a
width-independent $M$-Lipschitz and $L$-smoothness bounds, providing a principled guarantee for learning-rate transfer beyond these specialized constructions.

\subsection{How to select $p$ and $q$?}
In Theorem~\ref{theorem:Lsmooth}, both the $M$-Lipschitz and
$L$-smoothness bounds are derived under the assumption that the network
parameters lie in a width-independent bounded set under the chosen
geometry. In particular, the analysis is carried out on a compact set
$\Omega_c$. Such geometric constraints inevitably restrict the function
class represented by the network. In this subsection, we examine these
constraints from the perspective of approximation capacity and use
them to guide the choice of geometry.

To understand the implication of this assumption, we view the compact
set $\Omega_c$ as inducing a hypothesis class through the network
output map. For each parameter tuple
$(\bm{W}_{1:\ell},\b_{1:\ell}) \in \Omega_c$, the network produces an
output $\y_\ell(\x)$, and hence defines a predictor. Therefore, the
constraint $\Omega_c$ implicitly determines a function class
$\mathcal F_{\Omega_c}$, and different choices of geometry lead to
different hypothesis classes.

% In Theorem~\ref{theorem:Lsmooth}, both the $M$-Lipschitz and $L$-smoothness estimates rely on the assumption that the neural network admits a width-independent norm bound under the selected geometry. In this section, we demonstrate that this assumption imposes constraints on the network’s approximation capacity. 

We recall the standard decomposition of the expected risk in supervised learning.
Let $h^\ast$ denote the target function and $\hat h_{\mathcal{F}}$ the
model learned by an algorithm. The excess risk is defined as
\[
\mathcal{R}
:=
\mathbb{E}_{\mathbb{P}}[\mathcal{L}(\hat h_{\mathcal{F}}(\x))]
-
\mathbb{E}_{\mathbb{P}}[\mathcal{L}(h^\ast(\x))].
\]
It can be decomposed as
\begin{equation}
\begin{aligned}
\mathcal{R}
&\le
\underbrace{
\inf_{h\in \mathcal{F}}
\mathbb{E}_{\mathbb{P}}[\mathcal{L}(h(\x))]
-
\mathbb{E}_{\mathbb{P}}[\mathcal{L}(h^\ast(\x))]
}_{\text{approximation error}}
+
\underbrace{
\mathbb{E}_{\mathbb{P}_n}[\mathcal{L}(\hat h_\mathcal{F}(\x))]
-
\mathbb{E}_{\mathbb{P}}[\mathcal{L}(h_\mathcal{F}^\ast(\x))]
}_{\text{optimization error}}
\\
&\quad+
2
\underbrace{
\sup_{h\in\mathcal{F}}
\left|
\mathbb{E}_{\mathbb{P}_n}[\mathcal{L}(h(\x))]
-
\mathbb{E}_{\mathbb{P}}[\mathcal{L}(h(\x))]
\right|
}_{\text{generalization error}} .
\end{aligned}
\end{equation}
Here $\mathbb{E}_{\mathbb{P}}[\cdot]$ denotes expectation with respect to
the true (unknown) data distribution $\mathbb{P}$, while
$\mathbb{E}_{\mathbb{P}_n}[\cdot]$ denotes the empirical expectation
over $n$ samples.
The function class $\mathcal{F}$ represents the hypothesis space,
$
h^\ast_{\mathcal{F}}
=
\arg\min_{h\in\mathcal{F}}
\mathbb{E}_{\mathbb{P}}[\mathcal{L}(h(\x))]
$
denotes the best approximation of $h^\ast$ inside $\mathcal{F}$,
and
$
\hat h_{\mathcal{F}}
=
\arg\min_{h\in\mathcal{F}}
\mathbb{E}_{\mathbb{P}_n}[\mathcal{L}(h(\x))]
$
denotes the empirical risk minimizer. 
In this decomposition, the approximation error
quantifies the expressiveness of the hypothesis class $\mathcal{F}$,
the optimization error measures the suboptimality
of the learning algorithm, and the generalization error
captures the gap between training performance and unseen data.

The choice of norm influences all three components of the risk
decomposition. It determines the geometry of the hypothesis space and
thus its approximation power, affects gradients, conditioning, and
convergence in optimization, and regulates the effective capacity of
the model in generalization. We now formalize the notion of a stronger
norm.
\begin{definition}[Stronger norm]
Let $\|\cdot\|_{\textsf{A}}$ and $\|\cdot\|_{\textsf{B}}$ be two norms.
We say that $\|\cdot\|_{\textsf{A}}$ is \emph{stronger} than
$\|\cdot\|_{\textsf{B}}$ if
\[
\{\bm{X}:\|\bm{X}\|_{\textsf{A}}\le 1\}
\subseteq
\{\bm{X}:\|\bm{X}\|_{\textsf{B}}\le 1\}.
\]
\end{definition}
In other words, a stronger norm induces a smaller unit ball in the
parameter space. This geometric relationship underlies the trade-offs
between optimization, approximation, and generalization discussed
below. 

\begin{itemize}[leftmargin=0.15in]
\setlength{\itemsep}{4pt}
\item[\textnormal{(i)}]  \textbf{\textit{Optimization}}: Since stronger norms induce smaller unit balls in the parameter space,
the $M$-Lipschitz and $L$-smoothness constants of the objective
typically decrease when measured under the stronger geometry.
Consequently, gradient-based methods can adopt larger learning rates
while maintaining stability, which leads to faster convergence.
\item[\textnormal{(ii)}] \textbf{\textit{Approximation}}:
Norm constraints restrict the feasible parameter set and therefore
determine the hypothesis class that the network can represent.
If $\|\cdot\|_{\textsf{A}}$ is stronger than $\|\cdot\|_{\textsf{B}}$,
then the feasible set under $\|\cdot\|_{\textsf{A}}$ is  smaller,
which induces a more restricted hypothesis class.
As a result, stronger norms may reduce the approximation capacity of
the network.
\item[\textnormal{(iii)}] \textbf{\textit{Generalization}}:
Norm constraints also regulate the effective capacity of the model.
By limiting the magnitude of weights, stronger norms prevent the
network from fitting overly complex or noisy patterns in the training
data, which often improves generalization performance.
\end{itemize}

Overall, norm selection introduces an intrinsic trade-off.
Stronger norms improve optimization and generalization through tighter
control of the parameter space, but may reduce approximation capacity,
whereas weaker norms enlarge the hypothesis class at the potential
cost of worse smoothness and generalization behavior. 
We now examine the behavior of several normalization schemes under this
approximation--optimization trade-off.
While the unit balls of different operator norms are not directly
comparable due to dimension-dependent norm equivalence, the scaling
behavior still provides a useful proxy for how restrictive each
geometry becomes as the network width grows:

\begin{itemize}
    \item \textbf{\textrm{Muon} under $(2,\textrm{mean})\to (2,\textrm{mean})$ Geometry.} Under this geometry, the $L$-smoothness constant scales as
$\mathcal{O}(\sqrt{w})$, indicating that the optimization landscape
becomes increasingly rough as the network widens.
However, the admissible scale of the weight matrix remains
$\mathcal{O}(1)$ because the RMS normalization rescales the weights by
the matrix dimension.
As a result, the feasible parameter set does not shrink with width,
and the representational capacity of the network remains stable.
    % Although \textrm{Muon} exhibits larger $L$-smoothness that scales as $O(\sqrt{\text{Width}})$, meaning its optimization landscape becomes increasingly rough as the network widens. However, its unit ball size remains $O(1)$ and does not shrink as the network becomes wider, maintaining a consistent representational capacity regardless of width.
\item \textbf{Column normalization under ${(1,\textrm{mean})}\to\qmean$ Geometry.} This geometry achieves width-independent smoothness when $q\ge2$,
leading to a well-conditioned optimization landscape.
However, the constraint induced by the operator norm becomes
increasingly restrictive as the width grows.
To see this, note that under the $(1,\textrm{mean})$ geometry the input
norm scales as $\mathcal{O}(w^{-1})$, while the $(q,\textrm{mean})$
output norm scales as $\mathcal{O}(w^{-1/q})$.
Since the operator norm measures the maximum amplification ratio
between output and input norms, maintaining a bounded operator norm
requires the weight magnitudes to shrink by a factor
$\mathcal{O}(w^{-(q-1)/q})$.
Consequently, the admissible scale of weight columns decreases as
$\mathcal{O}(w^{-(q-1)/q})$, which shrinks rapidly with width.
Thus, although optimization becomes easier, the representational
capacity of the network is significantly constrained.

     % ${(1,\textrm{mean})}\to\pmean$ geometry achieves $O(1)$ smoothness when $p \geq 2$.  For the $\|\cdot\|_{(1,\textrm{mean})\to\ell_p\text{-mean}}$ operator norm, the input $\ell_1$-mean scales as $O(\text{Width}^{-1})$, while the output $\ell_p$-mean scales as $O(\text{Width}^{-1/p})$. Since the operator norm measures the maximum amplification ratio between output and input norms, maintaining the operator norm requires scaling by the ratio $O(\text{Width}^{-1/p})/O(\text{Width}^{-1}) = O(\text{Width}^{(p-1)/p})$. Consequently, the unit ball size scales as $O(\text{Width}^{-(p-1)/p})$. This unit ball size $O(\text{Width}^{-(p-1)/p})$ shrinks more aggressively with width, particularly for larger $p$ values, leading to more severe constraints on representational capacity.
    \item \textbf{Row normalization under $\pmean\to \infty$ Geometry.} 
Under this geometry, the $(p,\textrm{mean})$ input norm scales as
$\mathcal{O}(w^{-1/p})$, while the $\ell_\infty$ output norm remains
$\mathcal{O}(1)$.
Maintaining a bounded operator norm therefore requires the weight
magnitudes to shrink at a rate $\mathcal{O}(w^{-1/p})$.
Compared with column normalization, whose admissible scale shrinks as
$\mathcal{O}(w^{-(q-1)/q})$, row normalization contracts the feasible
parameter set much more slowly.
In particular, when $p,q\ge2$ we have
\[
\frac{1}{p}\le\frac{1}{2}<\frac{q-1}{q},
\]
so the constraint imposed by row normalization becomes less restrictive
with width compared to column normalization.
Consequently, this scaling behavior suggests that row normalization may
retain a richer hypothesis class while still benefiting from improved
smoothness.

% The $\|\cdot\|_{\ell_p\text{-mean}\to\ell_\infty}$ operator norm has unit ball size $O(\text{Width}^{-1/p})$ because the $\ell_p$-mean input norm scales as $O(\text{Width}^{-1/p})$ while the $\ell_\infty$ output norm is $O(1)$. To maintain unit operator norm, the weight matrix must scale as $O(1)/O(\text{Width}^{-1/p}) = O(\text{Width}^{1/p})$. This implies that the unit ball diameter scales as $O(\text{Width}^{-1/p})$, which shrinks more slowly than column normalization's $O(\text{Width}^{-(p-1)/p})$ as width increases. This is crucial because when $p \geq 2$, we have $\frac{1}{p} \le  \frac{1}{2} < \frac{q-1}{q}$ for all $q \geq 2$, and moreover $\frac{q-1}{q} \to 1$ as $p \to \infty$ (which leads to \textrm{SignSgd/Adam}). Therefore, row normalization surprisingly preserves a larger unit ball---and thus greater approximation capacity---compared to column normalization for the same network width.
\end{itemize}

% The row normalization family enjoys a better optimization-approximation trade-off because it simultaneously achieves $O(1)$ $L$-smoothness (ensuring efficient optimization) while maintaining a larger unit ball size that shrinks more slowly with width (preserving better approximation capacity). This favorable balance makes row normalization particularly attractive for wide networks. Therefore, we recommend steepest descent in the $\|\cdot\|_{\ell_p\text{-mean} \to \ell_\infty}$ norm for practical applications.
\begin{highlightbox}
    \textbf{Message:}  Stronger norms impose tighter constraints on the
parameter space, which improves smoothness of the optimization
landscape but restricts the expressive capacity of the model.
Norm selection therefore introduces an optimization--approximation
trade-off.
\begin{center}
\begin{tabular}{c|c|c|c}
 & \textrm{Row Normalization} & \textrm{Column Normalization} & \textrm{Muon} \\
\hline
Norm 
& $\|\cdot\|_{(p,\mathrm{mean})\to\ell_\infty}$ 
& $\|\cdot\|_{(1,\mathrm{mean})\to(q,\mathrm{mean})}$ 
& $\|\cdot\|_{(2,\mathrm{mean})\to(2,\mathrm{mean})}$ \\
\hline
$L$-smoothness 
& $\mathcal{O}(1)$ 
& $\mathcal{O}(1)$ \;($q\ge2$) 
& $\mathcal{O}(\sqrt{w})$ \\
\hline
Weight scaling 
& $\mathcal{O}(w^{-1/p})$ 
& $\mathcal{O}(w^{-(q-1)/q})$ 
& $\mathcal{O}(1)$
\end{tabular}
\end{center}
For $p,q\ge2$, both row and column normalization achieve
width-independent smoothness.
Since $\tfrac{1}{p} < \tfrac{q-1}{q}$, the admissible weight scale under row
normalization shrinks more slowly with width than under column
normalization, suggesting a better balance between smooth optimization
and representational flexibility.
\end{highlightbox}

\section{Conclusion} 
In this paper, we study the width-scaling behavior of neural optimizers
through the lens of matrix-operator-norm steepest descent, aiming to
understand how optimization geometry determines scalable learning-rate
rules. We show that many popular optimizers, including \textrm{AdamW},
\textrm{Muon}, \textrm{Lion}, and row/column normalization methods, can
be unified as instances of steepest descent under different matrix
operator norms. However, neural networks are not dimension-independently
Lipschitz under arbitrary matrix operator norms. In particular,
classical $p\to q$ operator norms fail to propagate stability across
layers: Their Lipschitz bounds do not compose cleanly, leading to
width-dependent distortions in the overall Lipschitz constant and
unstable learning-rate scaling.
% Do Transformers at different widths share the same optimal learning rate for a given optimizer? And does an algorithm that performs well on small models remain effective as the model scales up? In this paper, we study the width-scaling behavior of neural optimizers through the lens of matrix–operator–norm steepest descent, aiming to understand how optimization geometry determines scalable learning-rate rules. We understand many popular optimizer, including \textrm{AdamW}, \textrm{Muon}, \textrm{Lion} and
% row/column normalization, can be unified as instances of steepest descent
% under different matrix operator norms. However, a neural network is not dimension-independently Lipschitz under arbitrary matrix operator norms. In particular, classical $p\to q$ operator norms fail to propagate stability across layers: their Lipschitz estimates do not compose cleanly, causing width-dependent distortions in the overall Lipschitz constant and leading to unstable learning-rate scaling.

To address this issue, we introduce the \textit{mean normalized operator norm}
$(p,\textrm{mean}) \to (q,\textrm{mean})$ with $p\leq q$, which rescales the geometry to match network width. Under this geometry, operator norms compose
stably across layers, making the network forward map $M$-Lipschitz with
no width dependence. We further analyze gradient sensitivity and prove
that the objective becomes $L$-smooth with width-independent optimal
learning rates. Our theory identifies two width-stable geometries:
\[
\boxed{
(1,\textrm{mean}) \to (q,\textrm{mean}) \quad (q\ge2),
\qquad
(p,\textrm{mean}) \to \infty,}
\]
and shows that the latter preserves a less restrictive parameter scaling,
yielding a more favorable optimization--approximation trade-off. This
observation motivates a new optimizer, \textit{ rescaled row-normalized steepest
descent}, which applies a power transform followed by per-row
normalization.
% To resolve this issue, we introduce the \textbf{mean operator norm}
% $(p,\textrm{mean}) \!\to\! (q,\textrm{mean})$, which rescales the geometry
% to match neural‐network width. This makes operator norms \textit{play nicely
% together}: the network forward map becomes $M$-Lipschitz with no width
% dependence, enabling stable layerwise composition.  
% We further analyze gradient sensitivity and prove that under this geometry,
% the objective becomes $L$-smooth with \textit{width-independent optimal learning
% rates}.  Our theory identifies two width-independent geometries:
% \[
% (1,\textrm{mean}) \!\to\! (p,\textrm{mean}) \quad (p \ge 2),
% \quad \text{ and }
% \quad 
% (p,\textrm{mean}) \!\to\! \infty,
% \]
% and shows that the latter enlarges the feasible approximation set, yielding
% better optimization--approximation trade-offs.  This motivates a new optimizer:
% \textbf{row-normalized steepest descent}, which applies
% a power transform followed by per-row normalization.

%{\color{blue}Empirically, the proposed MOGA optimizer exhibits better width-scaling behavior than \textrm{AdamW}, validating our theoretical predictions without requiring any learning-rate retuning across widths in GPT pretraining. Moreover, \textrm{MOGA} attains the same token-wise convergence rate as MUON while avoiding any spectral computations. 

Based on our theory, we propose MOGA (Matrix Operator Geometry Aware), a width-aware optimizer that requires only row/column-wise normalization and enables reliable learning-rate transfer across model widths, allowing models of substantially different sizes to share nearly the same peak learning rate without extensive retuning. Large-scale pre-training experiments on GPT-2 and LLaMA further show that MOGA is competitive with Muon and can be substantially faster in large-token and low-loss regimes, where optimization stability is most critical. Overall, our results yield a $L$-smoothness based geometry-aware principle for constructing neural optimizers that remain stable under width scaling and preserve consistent learning dynamics as model size grows.

\paragraph{Acknowledgement} {We thank Jose Blanchet,  Fanghui Liu, Jianhao Ma, Jorge Nocedal, Weijie Su, Kaiyue Wen, Lei Wu, Tengyuan Liang, Yinuo Ren, Yinyu Ye,  Lexing Ying and Jiacheng You for their valuable feedback.}

% References here (outcomment the appropriate case)

% CASE 1: BiBTeX used to constantly update the references
%   (while the paper is being written).
%\bibliographystyle{informs2014} % outcomment this and next line in Case 1
%\bibliography{<your bib file(s)>} % if more than one, comma separated

\bibliographystyle{abbrv} % outcomment this and next line in Case 1
\bibliography{references} % if more than one, comma separated

\appendix

\section{Some Useful Lemmas}\label{APP:A}
In this appendix, we provide the missing proofs in the main context. 

We begin with a standard lemma relating Lipschitz continuity to bounded directional derivatives, which was used in Theorem \ref{thm:width-lip}.
\begin{lemma}
\label{lem:Mlip-set}
Let $\Omega \subset \mathbb{R}^{m\times n}$ be a convex norm ball, and let
$f:\Omega \to \mathbb{R}$ be differentiable on $\Omega$.
Fix any norm $\|\cdot\|$ on $\mathbb{R}^{m\times n}$.
Suppose there exists a constant $M>0$ such that
\[
\sup_{\left\|\Delta \bm Z\right\|\le 1}
\big|\nabla f(\bm Z)[\Delta \bm Z]\big|
\;\le\; M,
\qquad
\forall\, \bm Z \in \Omega.
\]
Then $f$ is $M$-Lipschitz continuous on $\Omega$ with respect to $\|\cdot\|$, i.e.,
\[
|f(\bm Z)-f(\bm Z')|
\;\le\;
M \left\|\bm Z-\bm Z'\right\|,
\qquad
\forall\, \bm Z,\bm Z' \in \Omega.
\]
\end{lemma}
\begin{proof}{Proof of Lemma~\ref{lem:Mlip-set}.}
Fix $\bm Z,\bm Z'\in\Omega$ and define the line segment
$
\bm Z(t):=(1-t)\bm Z'+t\bm Z=\bm Z' + t(\bm Z-\bm Z'), t\in[0,1].
$
Since $\Omega$ is a norm ball, it is convex, and hence $\bm Z(t)\in\Omega$ for all $t\in[0,1]$.
By the fundamental theorem of calculus applied along the line segment, we have 
\[
f(\bm Z)-f(\bm Z')
=\int_0^1 \nabla f(\bm Z(t))[\bm Z-\bm Z']\,dt.
\]
If $\bm Z=\bm Z'$, the claim is trivial. Otherwise let
$
\Delta \bm Z:=\frac{\bm Z-\bm Z'}{\|\bm Z-\bm Z'\|},
$ so that $\|\Delta \bm Z\|=1$. 
Then, for all $t\in[0,1]$, the assumed bound implies
\[
\big|\nabla f(\bm Z(t))[\bm Z-\bm Z']\big|
=\left\|\bm Z-\bm Z'\right\|\,\big|\nabla f(\bm Z(t))[\Delta\bm Z]\big|
\le M\left\|\bm Z-\bm Z'\right\|.
\]
Therefore,
\[
|f(\bm Z)-f(\bm Z')|
\le \int_0^1 M\left\|\bm Z-\bm Z'\right\|\,dt
= M\left\|\bm Z-\bm Z'\right\|.
\]
This proves that $f$ is $M$-Lipschitz continuous on $\Omega$ with respect to $\|\cdot\|$.  
\end{proof} 
\begin{lemma}
\label{lem:directional_hessian}
Let $\Omega \subset \mathbb{R}^{m\times n}$ be a convex set, and let
$f:\Omega \to \mathbb{R}$ be twice differentiable on $\Omega$.
Fix any norm $\|\cdot\|$ on $\mathbb{R}^{m\times n}$, and let $\|\cdot\|_{*}$
denote its dual norm.
Suppose there exists $L>0$ such that
\[
\sup_{\left\|\Delta \bm{Z}^1\right\|\le 1,\ \left\|\Delta \bm{Z}^2\right\|\le 1}
\bigl|\nabla^2 f(\bm{Z})[\Delta\bm{Z}^1,\Delta\bm{Z}^2]\bigr|
\le L,
\qquad \forall\, \bm{Z}\in \Omega.
\]
Then $f$ is $L$-smooth on $\Omega$ with respect to $\|\cdot\|$, i.e.,
\[
\left\|\nabla f(\bm{Z}) - \nabla f(\bm{Z}^\prime)\right\|_{*}
\le
L \, \left\|\bm{Z} - \bm{Z}^\prime\right\|,
\qquad
\forall\, \bm{Z},\bm{Z}^\prime\in \Omega.
\]
\end{lemma}
\begin{proof}{Proof of Lemma~\ref{lem:directional_hessian}.}
Fix $\bm{Z},\bm{Z}^\prime\in\Omega$ and define the line segment
$\bm{Z}(t):=\bm{Z}^\prime+t(\bm{Z}-\bm{Z}^\prime)$ for $t\in[0,1]$.
Since $\Omega$ is convex, $\bm{Z}(t)\in\Omega$ for all $t\in[0,1]$.
For any direction $\Delta \bm{Z}^1$, the fundamental theorem of calculus applied
along the segment yields
\begin{align*}
\nabla f(\bm{Z})[\Delta\bm{Z}^1] - \nabla f(\bm{Z}^\prime)[\Delta\bm{Z}^1]
 =
\int_0^1
\nabla^2 f(\bm{Z}(t))[\Delta\bm{Z}^1,\,\bm{Z}-\bm{Z}^\prime]\,\textrm{d}t .
\end{align*}
Taking the supremum over $\|\Delta\bm{Z}^1\|\le 1$ and using the definition of the dual norm,
\begin{align*}
\left\|\nabla f(\bm{Z})-\nabla f(\bm{Z}^\prime)\right\|_{*}
& =
\sup_{\left\|\Delta\bm{Z}^1\right\|\le 1}
\left|\nabla f(\bm{Z})[\Delta\bm{Z}^1] - \nabla f(\bm{Z}^\prime)[\Delta\bm{Z}^1]\right|\\
& \le
\int_0^1
\sup_{\left\|\Delta\bm{Z}^1\right\|\le 1}
\left|\nabla^2 f(\bm{Z}(t))[\Delta\bm{Z}^1,\,\bm{Z}-\bm{Z}^\prime]\right|\,\textrm{d}t .
\end{align*}
If $\bm{Z}=\bm{Z}^\prime$, the claim is trivial. Otherwise let
$
\Delta \bm Z^2:=\frac{\bm Z-\bm Z'}{\|\bm Z-\bm Z'\|},
$ so that $\|\Delta \bm Z^2\|=1$.  Then for all $t\in[0,1]$,
\begin{align*}
\sup_{\left\|\Delta\bm{Z}^1\right\|\le 1}
\left|\nabla^2 f(\bm{Z}(t))[\Delta\bm{Z}^1,\,\bm{Z}-\bm{Z}^\prime]\right|
& =
\left\|\bm{Z}-\bm{Z}^\prime\right\|
\sup_{\left\|\Delta\bm{Z}^1\right\|\le 1}
\left|\nabla^2 f(\bm{Z}(t))[\Delta\bm{Z}^1,\,\Delta\bm{Z}^2]\right|\\
& \le
L\,\left\|\bm{Z}-\bm{Z}^\prime\right\|.
\end{align*}
Therefore,
\[
\left\|\nabla f(\bm{Z})-\nabla f(\bm{Z}^\prime)\right\|_{*}
\le
\int_0^1 L\,\left\|\bm{Z}-\bm{Z}^\prime\right\|\,\textrm{d}t
=
L\,\left\|\bm{Z}-\bm{Z}^\prime\right\|.
\]
     
\end{proof}

\begin{lemma}
\label{lemma:stnorm}
Let $\x \in \mathbb{R}^n$. Then for any $p,q \ge 1$,
\[
\left\|\x\right\|_{(p,\mathrm{mean})}
\le
n^{\max\!\left(0,\,\tfrac{1}{q}-\tfrac{1}{p}\right)}
\left\|\x\right\|_{(q,\mathrm{mean})}.
\]
\end{lemma} 

\begin{proof}{Proof of Lemma~\ref{lemma:stnorm}.}
Recall that $\|\x\|_{(p,\mathrm{mean})}=n^{-\tfrac{1}{p}}\|\x\|_p$.
We consider two cases.

\medskip
\noindent
\textbf{Case 1: $p \le q$.}
It is well known that $\|\x\|_p \le n^{\,\tfrac{1}{p}-\tfrac{1}{q}}\|\x\|_q$.
Hence,
\[
\left\|\x\right\|_{(p,\mathrm{mean})}
=
n^{-\tfrac{1}{p}}\left\|\x\right\|_p
\le
n^{-\tfrac{1}{p}}\,n^{\,\tfrac{1}{p}-\tfrac{1}{q}}\left\|\x\right\|_q
=
n^{-1/q}\left\|\x\right\|_q
=
\left\|\x\right\|_{(q,\mathrm{mean})}.
\]

\medskip
\noindent
\textbf{Case 2: $p > q$.}
In this case, $\|\x\|_p \le \|\x\|_q$, and therefore
\[
\left\|\x\right\|_{(p,\mathrm{mean})}
=
n^{-\tfrac{1}{p}}\left\|\x\right\|_p
\le
n^{-\tfrac{1}{p}}\left\|\x\right\|_q
=
n^{\,\tfrac{1}{q}-\tfrac{1}{p}}\left\|\x\right\|_{(q,\mathrm{mean})}.
\]
Combining the two cases yields the desired result.
 
\end{proof}  

\begin{lemma}
\label{lem:montomeannorm}
Let $\x,\y \in \mathbb{R}^n$. For any $p,q \ge 1$, we have
\[
\left\|\x \odot \y\right\|_{(p,\mathrm{mean})}
\le
n^{\,\max \left(0,\,\tfrac{2}{q}-\tfrac{1}{p}\right)}\,
\left\|\x\right\|_{(q,\mathrm{mean})}\,
\left\|\y\right\|_{(q,\mathrm{mean})}.
\]
\end{lemma}
\begin{proof}{Proof of Lemma~\ref{lem:montomeannorm}.}
By definition,
\[
\left\|\x \odot \y\right\|_{(p,\mathrm{mean})}
=
\Bigl(\frac{1}{n}\sum_{i=1}^n |x_i y_i|^p\Bigr)^{1/p}.
\]
Applying the Cauchy--Schwarz inequality yields
\[
\frac{1}{n}\sum_{i=1}^n |x_i y_i|^p
=
\frac{1}{n}\sum_{i=1}^n |x_i|^p |y_i|^p
\le
\Bigl(\frac{1}{n}\sum_{i=1}^n |x_i|^{2p}\Bigr)^{1/2}
\Bigl(\frac{1}{n}\sum_{i=1}^n |y_i|^{2p}\Bigr)^{1/2}.
\]
Taking the $1/p$ power on both sides, we obtain
\begin{equation}
\label{eq:mean_norm}
\left\|\x \odot \y\right\|_{(p,\mathrm{mean})}
\le
\left\|\x\right\|_{(2p,\mathrm{mean})}\,\left\|\y\right\|_{(2p,\mathrm{mean})}.
\end{equation}
By Lemma~\ref{lemma:stnorm} with $(p,q)$ replaced by $(2p,q)$, we have
\[
\left\|\x\right\|_{(2p,\mathrm{mean})}
\le
n^{\max\left(0,\,\frac{1}{q}-\frac{1}{2p}\right)}
\left\|\x\right\|_{(q,\mathrm{mean})},
\quad \text{and} \quad 
\left\|\y\right\|_{(2p,\mathrm{mean})}
\le
n^{\max\!\left(0,\,\frac{1}{q}-\frac{1}{2p}\right)}
\left\|\y\right\|_{(q,\mathrm{mean})}.
\]
Combining with \eqref{eq:mean_norm} gives
\[
\left\|\x \odot \y\right\|_{(p,\mathrm{mean})}
\le
n^{\,2\max\!\left(0,\,\frac{1}{q}-\frac{1}{2p}\right)}
\left\|\x\right\|_{(q,\mathrm{mean})}\,\left\|\y\right\|_{(q,\mathrm{mean})},
\]
which proves the claim.
  
\end{proof}

\begin{lemma}
\label{lem:mean-holder}
Let $\x,\y\in \mathbb{R}^n$. Then for any $p,q \ge 1$,
\begin{equation}
\label{eq:mean-ip-bound}
|\langle \x,\y\rangle|
\;\le\;
n^{\max\left\{1,\tfrac{1}{p}+\tfrac{1}{q}\right \}}\;
\|\x\|_{(p,\textrm{mean})}\;
\|\y\|_{(q,\textrm{mean})}.
\end{equation}
\end{lemma}
\begin{proof}{Proof of Lemma~\ref{lem:mean-holder}.}
Let $p^*$ be the conjugate exponent of $p$, i.e., $\tfrac{1}{p}+\tfrac{1}{p^*}=1$. By H\"older's inequality,
\[
|\langle \x,\y\rangle|
\le
\|\x\|_{p}\,\|\y\|_{p'}.
\]
Moreover, by the definition of $\pmean$ norm, i.e., $\|\x\|_{p}=n^{1/p}\|\x\|_{(p,\mathrm{mean})}$, for any $\x\in\mathbb{R}^n$ and $p\ge 1$, we obtain
\[
|\langle \x,\y\rangle|
\le
n\,\|\x\|_{(p,\mathrm{mean})}\,\|\y\|_{(p^*,\mathrm{mean})}.
\]
Thanks to Lemma~\ref{lemma:stnorm}, we have 
\[
\|\y\|_{(p^*,\mathrm{mean})}
\le
n^{\left(\frac1q-\frac1{p^*}\right)_+}\|\y\|_{(q,\mathrm{mean})}.
\]
Putting them together yields 
\[
|\langle \x,\y\rangle|
\le
n^{\,1+\left(\frac1q-\frac1{p^*}\right)_+}
\|\x\|_{(p,\mathrm{mean})}\,
\|\y\|_{(q,\mathrm{mean})}.
\]
Finally, since $\tfrac{1}{p^*}=1-\tfrac{1}{p}$, we have
$
1+\left(\frac1q-\frac1{p'}\right)_+
=
\max\left\{1,\frac1p+\frac1q\right\}$, which yields the desired result. 
   
\end{proof}

% CASE 2: BiBTeX used to generate mypaper.bbl (to be further fine tuned)
%\input{mypaper.bbl} % outcomment this line in Case 2

%If you don't use BiBTex, you can manually itemize references as shown below.

% \bibliographystyle{nonumber}

% \begin{thebibliography}{3}
% \providecommand{\natexlab}[1]{#1}
% \providecommand{\url}[1]{\textrm{#1}}
% \providecommand{\urlprefix}{URL }

% \bibitem[{Smith(2005)}]{smith2005}
% Smith J (2005) Optimal resource allocation in humanitarian logistics.
%   \emph{Journal of Operations Research} 30(2):123--135.
  
% \bibitem[{Jones(2010)}]{jones2010}
% Jones S (2010) Stochastic programming models for humanitarian logistics.
%   \emph{INFORMS Mathematics of Operations Research} 35(4):567--580.

% \bibitem[{Brown(2015)}]{brown2015}
% Brown D (2015) \emph{Introduction to Stochastic Programming} (Springer).

% \end{thebibliography}

%%%%%%%%%%%%%%%%%

\end{document}